\documentclass{article}
\usepackage{titletoc}

\usepackage{microtype}
\usepackage{graphicx}
\usepackage{subfigure}
\usepackage{booktabs} 

\usepackage{hyperref}

\usepackage[accepted]{icml2024}

\usepackage{amsmath}
\usepackage{amssymb}
\usepackage{mathtools}
\usepackage{amsthm}
\usepackage{enumitem}

\usepackage{smile}
\usepackage{bbm}
\usepackage[capitalize,noabbrev]{cleveref}

\theoremstyle{plain}
\newtheorem{theorem}{Theorem}[section]

\newtheorem{lemma}[theorem]{Lemma}
\newtheorem{corollary}[theorem]{Corollary}
\theoremstyle{definition}

\newtheorem{assumption}[theorem]{Assumption}
\theoremstyle{remark}
\newtheorem{remark}[theorem]{Remark}

\allowdisplaybreaks

\def \avg {\mathrm{avg}}
\def \med {\mathrm{median}}

\usepackage{color}
\usepackage{xcolor}
\usepackage{soul}
\definecolor{carnelian}{rgb}{0.7, 0.11, 0.11}
\definecolor{lemonchiffon}{rgb}{1.0, 0.98, 0.8}
\definecolor{darkblue}{rgb}{0.0, 0.0, 0.55}
\definecolor{darkgreen}{RGB}{0, 120, 0}

\newcommand{\ifcomments}{\iftrue}

\begin{document}

\twocolumn[
\icmltitle{FADAS: Towards Federated Adaptive Asynchronous Optimization }



\icmlsetsymbol{equal}{*}

\begin{icmlauthorlist}
\icmlauthor{Yujia Wang}{psu}
\icmlauthor{Shiqiang Wang}{ibm}
\icmlauthor{Songtao Lu}{ibm}
\icmlauthor{Jinghui Chen}{psu}
\end{icmlauthorlist}

\icmlaffiliation{psu}{College of Information Sciences and Technology, Pennsylvania State University, State College, PA, USA}
\icmlaffiliation{ibm}{IBM T. J. Watson Research Center,
Yorktown Heights, NY, USA}

\icmlcorrespondingauthor{Yujia Wang}{yjw5427@psu.edu}
\icmlcorrespondingauthor{Shiqiang Wang}{\mbox{wangshiq@us.ibm.com}}
\icmlcorrespondingauthor{Songtao Lu}{songtao@ibm.com}
\icmlcorrespondingauthor{Jinghui Chen}{jzc5917@psu.edu}

\icmlkeywords{Machine Learning, ICML}

\vskip 0.3in
]



\printAffiliationsAndNotice{}  

\begin{abstract}
Federated learning (FL) has emerged as a widely adopted training paradigm for privacy-preserving machine learning. While the SGD-based FL algorithms have demonstrated considerable success in the past, there is a growing trend towards adopting adaptive federated optimization methods, particularly for training large-scale models. However, the conventional synchronous aggregation design poses a significant challenge to the practical deployment of those adaptive federated optimization methods, particularly in the presence of straggler clients. To fill this research gap, this paper introduces federated adaptive asynchronous optimization, named FADAS, a novel method that incorporates asynchronous updates into adaptive federated optimization with provable guarantees. To further enhance the efficiency and resilience of our proposed method in scenarios with significant asynchronous delays, we also extend FADAS with a delay-adaptive learning adjustment strategy. We rigorously establish the convergence rate of the proposed algorithms and empirical results demonstrate the superior performance of FADAS over other asynchronous FL baselines. 
\end{abstract}

\section{Introduction}

In recent years, federated learning (FL) \cite{mcmahan2017communication} has drawn increasing attention as an efficient privacy-preserving distributed machine learning paradigm. An FL framework consists of a central server and numerous clients, where clients collaboratively train a global model without sharing their private data. FL entails each client conducting multiple local iterations, while the central server periodically aggregates these local updates into the global model. Following the original design of the FedAvg algorithm \cite{mcmahan2017communication}, a large number of stochastic gradient descent (SGD)-based FL methods have emerged, aiming to improve the performance or efficiency of FedAvg \cite{karimireddy2020scaffold, acar2021federated, wang2020tackling}.

In addition to the successes of  SGD-based algorithms in enhancing the efficiency of FL, the adoption of adaptive optimization techniques is becoming increasingly prevalent in FL. 
Adaptive optimization techniques such as Adam \cite{kingma2014adam} and AdamW \cite{loshchilov2017decoupled} have proven their advantages over SGD in effectively training or fine-tuning large-scale models like BERT \cite{devlin2018bert}, ViT \cite{dosovitskiy2021an}, and Llama \cite{touvron2023llama}.
This progress has encouraged the incorporation of adaptive optimization into the FL settings, taking advantage of their ability to navigate update directions and dynamically adjust learning rates. 
For example, FedAdam \cite{reddi2021adaptive} and FedAMS \cite{wang2022communication} employ global adaptive optimization after the server aggregates local model updates. Moreover, strategies such as FedLALR \cite{sun2023fedlalr}, FedLADA \cite{sun2023efficient}, and FAFED \cite{wu2023faster} replace SGD with the Adam optimizer for the local training phase, exemplifying the utility of local adaptive optimizations in FL.

However, existing methods in adaptive FL still rely on traditional synchronous aggregation approaches, where the server must wait for all participating clients to complete their local training before global updates. This reliance presents a significant challenge to the practical implementation of adaptive FL methodologies, as the server is required to wait until slower clients, which may have limited computation or communication capabilities. While asynchronous FL strategies such as FedBuff \cite{nguyen2022federated} and FedAsync \cite{xie2019asynchronous} have been investigated to improve the scalability and to study the impact of client delays on the convergence of SGD-based FL algorithms, the specific implications of asynchronous delays on nonlinear adaptive gradient operations are not completely understood.
This motivates us to explore the following question: 

\textit{Can we develop an asynchronous method for adaptive federated optimization (with provable guarantees) that enhances training efficiency and is resilient to asynchronous delays?}

In this paper, we propose FADAS, \textbf{F}ederated \textbf{AD}aptive \textbf{AS}ynchronous optimization, to address this challenge. FADAS introduces asynchronous updates within the adaptive federated optimization framework and integrates a delay-adaptive mechanism for adjusting the learning rate adaptively in response to burst delays.
We summarize our contributions as follows:
\vspace{-5pt}
\begin{itemize}[leftmargin=*]
\vspace{-2pt}
    \item We propose FADAS, a novel adaptive federated optimization method that extends traditional adpative federated optimization \textit{support asynchronous client updates}. We prove that FADAS achieves a convergence rate of $\cO\big(\frac{1}{\sqrt{TM}} + \frac{\tau_{\max}\tau_{\avg}}{T}\big)$ w.r.t. the number of global communication rounds $T$ and the number of accumulated updates $M$, with bounded worst-case delay, denoted by $\tau_{\max}$, and the average of the maximum delay over all the rounds, denoted by $\tau_{\avg}$. 
\vspace{-2pt}
    \item To further reduce the dependency on the worst-case delay term $\tau_{\max}$ in the convergence rate, we extend FADAS with a \textit{delay-adaptive learning rate adjustment strategy}. Our theoretical results demonstrate that the inclusion of a delay-adaptive learning rate effectively diminishes the dependency on $\tau_{\max}$ in the convergence rate.
\vspace{-2pt}
    \item We conduct experiments across various asynchronous delay settings in both vision and language modeling tasks. Our results indicate that the proposed FADAS, whether or not including the delay-adaptive learning rate, outperforms other asynchronous FL baselines. In particular, the delay-adaptive FADAS demonstrates significant advantages in scenarios with large worst-case delays. Moreover, our experimental results on simulating the wall-clock training time underscores the efficiency of our proposed FADAS approach.
\end{itemize}

\section{Related Work}
\textbf{Federated learning.}
FL, as introduced by \citet{mcmahan2017communication}, has become a pivotal framework for collaboratively training machine learning models on edge devices while keeping local data private. Following the initial FedAvg algorithm, several works studied the theoretical analysis and empirical performance of it \cite{lin2018don, stich2018local,li2019convergence, karimireddy2020scaffold, wang2021cooperative, yang2021achieving}, and a range of works aim to improve FedAvg from different perspectives, such as reducing the impact of data heterogeneity \cite{karimireddy2020scaffold, acar2021federated, wang2020tackling}, saving the communication overhead \cite{reisizadeh2020fedpaq, jhunjhunwala2021adaptive}, and adjusting the parameter aggregation procedure \cite{tan2022adafed, wang2023lightweight}. 

\textbf{Adaptive FL optimizations and adaptive updates. }
Besides traditional SGD-based methods, there is a line of works focusing on adaptive updates in FL. A local adaptive FL method with momentum-based variance-reduced gradient was used in FAFED \cite{wu2023faster}. \citet{li2023fedda} proposed a framework for local adaptive gradient methods in FedDA. FedLALR \cite{sun2023fedlalr} uses local adaptive optimization in FL with local historical gradients and periodically synchronized learning rates. FedLADA \cite{sun2023efficient} is an efficient local adaptive FL method with a locally amended technique. \citet{jin2022accelerated} developed novel adaptive FL optimization methods from the perspective of dynamics of ordinary differential equations. Moreover, \citet{reddi2021adaptive} introduced FedAdagrad, FedAdam and FedYogi, and \citet{wang2022communication} proposed FedAMS for global adaptive FL optimizations. Several works of global adaptive learning rate \cite{jhunjhunwala2023fedexp} and adaptation in aggregation weights \cite{tan2022adafed, wang2023lightweight} are also related to adaptive learning rate adjustment.

\textbf{Asynchronous SGD and asynchronous FL. } 
There have been extensive studies over the years about asynchronous optimization techniques, including asynchronous SGD and its various adaptations. For example, Hogwild \cite{10.5555/2986459.2986537} includes an applicable lock-free, coordinate-wise asynchronous method and has been widely used in multi-thread computation. A body of works focuses on the theoretical analysis and explorations of asynchronous SGD \cite{mania2017perturbed, pmlr-v80-nguyen18c, stich2021critical, JMLRasync, glasgow2022asynchronous} and discusses the gradient delay in the convergence rate \cite{pmlr-v139-avdiukhin21a, mishchenko2022asynchronous, koloskova2022sharper, wu2022delay}. Within federated learning, innovative asynchronous aggregation algorithms like FedAsync \cite{xie2019asynchronous} allow the server to update the global model once a client finishes local training, and FedBuff \cite{nguyen2022federated} introduces a buffered aggregation approach. There are also many works focusing on algorithms based on FedBuff with theoretical and/or empirical analysis \cite{toghani2022unbounded, ortega2023asynchronous, wang2023tackling}, and other aspects of asynchronous FL \cite{chen2020vafl, yang2022anarchic,bornstein2023swift}. Although adaptive FL and asynchronous FL have achieved the success of training large machine learning models with desirable numerical performance, the exploration of asynchronous updates in the context of adaptive FL has not been well-studied yet. In this paper, we start with the asynchronous update framework in adaptive FL and further integrate delay-adaptive learning rate scheduling into it. 

\vspace{-2pt}
\section{Preliminaries}

\textbf{Federated learning. } A general FL framework considers a distributed optimization problem across $N$ clients:\vspace{-1.9em}

\begin{small}
\begin{align}\label{eq:fl-basic}
    \min_{\bx \in \RR^d} f(\bx) := \frac{1}{N}\sum_{i=1}^N F_i(\bx) = \frac{1}{N} \sum_{i=1}^N \EE_{\xi_i\sim\cD_i} [F_i (\bx; \xi_i)],
\end{align}
\end{small}
where $\bx\in\mathbb{R}^d$ is the model parameter with $d$ dimensions, $F_i(\bx)$ is the loss function corresponding to client $i$, $\cD_i$ is the local data distribution on client $i$. The objective in Eq. \eqref{eq:fl-basic} can be interpreted as setting $p_i = \frac{1}{N}$ for all clients in another commonly used objective function in FL, i.e., $f(\bx) = \sum_{i=1}^N p_i \EE_{\xi_i\sim\cD_i} [F_i (\bx; \xi_i)]$ with $p_i \geq 0 $ and $\sum_{i=1}^N p_i = 1$. FedAvg \citep{mcmahan2017communication} is a typical synchronous FL algorithm to solve Eq. \ref{eq:fl-basic}, where in the $t$-th global round, each participating client $i$ performs local SGD updates as follows:
\begin{align}
    \bx_{t,k+1}^i = \bx_{t,k}^i - \eta_l \nabla F_i(\bx_{t,k}^i; \xi) \text{ and } \bx_{t,0}^i = \bx_t
\end{align}
where $\eta_l$ is the learning rate. After several local steps (e.g., $K$ steps of local training), the server performs a global averaging step after receiving all the updates from assigned clients in $\cS_t$, i.e., $\bx_{t+1} = \frac{1}{|\cS_t|}\sum_{i \in \cS_t} \bx_{t,K}^i$. 

\textbf{Adaptive optimization and its application to FL.} 
Several adaptive optimizers have been proposed to improve the convergence of SGD, such as Adagrad \cite{duchi2011adaptive}, RMSProp \cite{tieleman2012lecture}, Adam \cite{kingma2014adam} and its variant AMSGrad \cite{j.2018on}. 
In general machine learning optimization, Adam effectively inherits the benefits of both momentum and RMSProp optimizers, leading to better empirical performance in practical applications.

\citet{reddi2021adaptive} first introduced adaptive federated optimization, which applies the adaptive optimizers during the global aggregation steps in FL. FedAMSGrad \cite{tong2020effective} and FedAMS \cite{wang2022communication} further adjust the effective global learning rate in adaptive FL. Specifically, FedAdam and FedAMS take the idea of viewing the difference of local updates $\bDelta_t^{\text{sync}} = \frac{1}{|\cS_t|} \sum_{i \in \cS_t}\bDelta_t^{i,\text{sync}} = \frac{1}{|\cS_t|} \sum_{i \in \cS_t} (\bx_{t,K}^i - \bx_t) $ as a pseudo-gradient, and applies the Adam or AMSGrad optimizer when updating global model $\bx_{t+1}$ using $\bDelta_t^{\text{sync}}$, i.e., 
\begin{align*}
    & \bbm_t = \beta_1 \bbm_{t-1} + (1-\beta_1) \bDelta_t^{\text{sync}} \notag,\\
    & \bv_t = \beta_2 \bv_{t-1} + (1-\beta_2) \bDelta_t^{\text{sync}} \odot \bDelta_t^{\text{sync}} \notag,\\
    & \bx_{t+1}=\bx_t + \eta \frac{\bbm_t}{\sqrt{\bv_t} + \epsilon} \text{ (FedAdam)},\notag\\
    & \hat\bv_t = \max(\hat\bv_{t-1}, \bv_t), \bx_{t+1}=\bx_t + \eta \frac{\bbm_t}{\sqrt{\hat\bv_t} + \epsilon} \text{ (FedAMS)},
\end{align*}
where $\odot$ denotes the element-wise product for two vectors, and for vectors $\bx,\by \in \RR^d$, $\sqrt{\bx}, \bx/\by$, $\max(\bx, \by)$ denote the element-wise square root, division, and maximum operation of the vectors.

\textbf{Asynchronous updates in FL. }
In asynchronous FL, clients train the model asynchronously and update it to the server once it finishes several steps of local training. FedBuff \cite{nguyen2022federated} has improved the global update steps with the concept of buffer based on the initial FedAsync baseline \cite{xie2019asynchronous}. In FedBuff, it requires the framework maintain a given number (referred to as the concurrency $M_c$) of clients that are actively local training. At the $t$-th global round, after the client $i$ finishes local training, it sends its local update $\bm\Delta_t^i = \bx_{t-\tau,K}^i - \bx_{t-\tau}$ to the server, where $t-\tau$ is the global round where client $i$ starts local training and $0 \leq \tau \leq t$. The server simultaneously accumulates the model update $\bm\Delta_t^i$ to the global update direction $\bm\Delta_t \leftarrow \bm\Delta_t + \bm\Delta_t^i$, and sends the latest global model to a randomly selected client who is idle. When the number of accumulated updates reaches the given buffer size of $M$, the server updates the global model with the averaging $\bm\Delta_t/M$. Meanwhile, clients who have not finished their local training will continue their training based on the previously received global model, and are not affected by the global model updates on the server. During the training, the framework always maintains a fixed number ($M_c$) of clients who are conducting local training. This is achieved by having the server randomly sample an idle client for training each time a client completes its local training and sends its update to the server.

\paragraph{Discussion about synchronous and asynchronous methods.} 
Synchronous FL typically offers consistency and stability, i.e., all client updates are based on the same global model, and this consistency may lead to a more stable and predictable learning process. However, when there exist one or a few clients that are much slower than the majority of clients, which often happens in large-scale systems, synchronous FL can be inefficient since every client needs to wait for the slowest client before progressing with the next round of training. Asynchronous FL is more efficient when clients have system heterogeneity such as diverse computational capabilities or communication bandwidth. In FL, if the delay among clients is relatively uniform, synchronous FL tends to be more stable and efficient. Overall, the choice between synchronous and asynchronous FL hinges on specific needs and system characteristics. Synchronous FL is ideal in homogeneous systems, while asynchronous FL is advantageous in heterogeneous systems with potential straggler clients.

\section{Proposed Method: FADAS}
Although adaptive FL methods achieve promising convergence and generalization performance theoretically and empirically, the existing adaptive FL methods are restricted to synchronous settings, as the server needs to wait for all the assigned clients to finish their local updates for aggregation and then update the global model. However, those synchronous adaptive FL algorithms are susceptible to the presence of stragglers, where slower clients with insufficient computation or communication speed impede the progress of the global update.

To improve the efficiency and resiliency of adaptive FL in the presence of stragglers, we introduce FADAS, a \textbf{F}ederated \textbf{AD}aptive \textbf{AS}ynchronous optimization method. Similar to FedAdam and FedAMS, the proposed FADAS algorithm takes the model update difference from clients as a pseudo-gradient and it updates the global model following an Adam-like update scheme. Algorithm \ref{alg:fadas} summarizes the details.
FADAS keeps the local asynchronous training scheme as FedBuff and maintains the concept of concurrency and buffer size for flexible control of the number of active clients and the frequency of global model update. In FADAS, after the server aggregates to obtain model update difference $\bm\Delta_t$, it finds an adaptive update direction, whose components are computed based on the AMSGrad optimizer \cite{j.2018on} as follows:
\begin{align}\label{eq:fedadam}
\begin{cases}
    & \bbm_t = \beta_1 \bbm_{t-1} + (1-\beta_1) \bDelta_t, \\
    & \bv_t = \beta_2 \bv_{t-1} + (1-\beta_2) \bDelta_t \odot \bDelta_t, \\
    & \hat\bv_t = \max(\hat\bv_{t-1}, \bv_t).
\end{cases}
\end{align}
In general, FADAS enables clients to conduct local training in their own pace, and the server aggregates the asynchronous updates for global adaptive updates. It improves the training efficiency and scalability of over synchronous adaptive FL while inheriting the advantage of adaptive optimizer of reducing oscillations and stabilizing the optimization process.

\begin{algorithm}[ht!]
\caption{FADAS (with \colorbox{teal!20}{delay adaptation})}
  \label{alg:fadas}
  \begin{flushleft}
        \textbf{Input:} local learning rate $\eta_l$, global learning rate $\eta$, adaptive optimization parameters $\beta_1, \beta_2, \epsilon$, server concurrency $M_c$, buffer size $M$, \colorbox{teal!20}{delay threshold $\tau_c$};
        \end{flushleft}
  \begin{algorithmic}[1]
        \STATE Initialize model $\bx_1$, initialize $\bm\Delta_1=\zero$, $\bbm_0=\zero, \bv_0=\zero$, $m=0$ and sample a set of $\cM_0$ with size $M_c$ active clients to run \textbf{local SGD updates}. 
      \REPEAT
        \IF{receive client update}
        \STATE Server accumulates update from client $i$: $\bm\Delta_t \leftarrow \bm\Delta_t + \bm\Delta_{t}^{i}$ and set $m \leftarrow m+1$
        \STATE Sample another client $j$ from available clients
        \STATE Send the current model $\bx_t$ to client $j$, and run \textbf{local SGD updates} on client $j$
        \ENDIF
        \IF{$m=M$}
        \STATE $\bm\Delta_t \leftarrow \frac{\bm\Delta_t}{M} $
        \STATE Update $\bbm_t$, $\bv_t, \hat\bv_t $ by \eqref{eq:fedadam}
        \IF{delay-adaptive}
        \STATE \colorbox{teal!20}{Set $\eta_t$ to be delay-adaptive based on Eq. \eqref{eq:delay_adapt}}
        \ELSE
        \STATE $\eta_t = \eta $
        \ENDIF
        \STATE Update global model $\bx_{t+1}=\bx_t + \eta_t \frac{\bbm_t}{\sqrt{\hat\bv_t} + \epsilon}$
        \STATE Set $m \leftarrow 0$,  $\bm\Delta_{t+1} \leftarrow \zero$, $t \leftarrow t+1$
        \ENDIF
       
      \UNTIL{convergence}
  \end{algorithmic}
\end{algorithm}

Although FADAS applies asynchronous local training for adaptive FL, the global adaptive optimizer adjusts the global update direction only based on local updates but without considering the impact of asynchronous delay. Intuitively, a large asynchronous delay from a client means that this model update is made based on an outdated global model. This may lead to a negative effect on the convergence, and later we also verify this intuition in the theoretical analysis. This inspires us to apply a delay-adaptive learning rate adjustment to improve the resiliency of FADAS to stragglers with large delays. Specifically, we let the server track the delay for every received model update and adopt a delay-adaptive learning rate. We highlight the delay-adaptive steps in Algorithm \ref{alg:fadas} and those steps are executed with almost no extra overhead.

\textbf{Delay tracking.} In general, the server manages the delay record for each client through straightforward time-stamping. For example, the server records the global update round $t'$ when it broadcasts the current global model $\bx_{t'}$ to client $i$, the client conducts local training with $\bx_{t'}$. When the server receives the first $\bm\Delta_t^i$ from client $i$ at round $t \geq t'$, the gradient delay for $\bm\Delta_t^i$, which is $\tau_t^i = t - t'$, is updated and recorded on the server. 

\textbf{Delay-adaptive learning rate. }Assume that for each global update round $t$, clients in the set $\cM_t$ ( $|\cM_t|=M$) send updates to the server. The received model updates at global round $t$ have a maximum delay $\tau_t^{\max}$ defined as $\tau_t^{\max}:= \max\{\tau_t^i, i \in \cM_t\}$. Suppose we set up a delay threshold $\tau_c$, we can define a delay-adaptive learning rate as:
\begin{align}\label{eq:delay_adapt}
\eta_t = 
\begin{cases}
\eta & \text{if } \tau_t^{\max} \leq \tau_c, \\
\min\left\{\eta, \frac{1}{\tau_t^{\max}}\right\} & \text{if } \tau_t^{\max} > \tau_c.
\end{cases}
\end{align}

Intuitively, this design implies that we need to turn the learning rates down for the model update $\bm\Delta_t$ with larger current-step delays. Specifically, if the current-step maximum delay $\tau_t^{\max}$ is larger than a given threshold $\tau_c$, we scale down the learning rates for this step in proportional to $1/\tau_t^{\max}$ (also capped by a constant learning rate $\eta$) to avoid that the high-latency update worsens the convergence.
\textbf{Comparison with FedAsync ~\cite{xie2019asynchronous}. } FedAsync ~\cite{xie2019asynchronous}  also studies delay-adaptive weighted averaging during global model updates. In FedAsync, after the server receives a local model $\bx_{\text{new}}$, it updates $\bx_{t}$ based on $\bx_t = (1-\alpha_t) \bx_{t-1} + \alpha_t \bx_{\text{new}}$, and FedAsync includes a hinge strategy of $\alpha_t$ which is similar to our delay-adaptive strategy in Eq. \eqref{eq:delay_adapt}. However, unlike FedAsync, where the server updates the global model immediately upon receiving a new update from a client, FADAS updates the global model less frequently. In FADAS, the server accumulates $M$ local updates before a global update. 
Moreover, the convergence analysis in FedAsync did not consider their delay adaptation procedure,
while we provide a convergence analysis incorporating the effect of delay adaptation in the next section. 

\section{Theoretical Analysis}
In this section, we delve into the theoretical analysis of our proposed FADAS algorithm. We first introduce some common assumptions required for the analysis. Subsequently, we present the analysis in two parts: one focusing on FADAS without delay adaptation, as discussed in Section~\ref{subsec:fadas}, and the other on the delay-adaptive FADAS in Section~\ref{subsec:da_fadas}.
\begin{assumption}[Smoothness]\label{as:smooth}
Each objective function on the $i$-th worker $F_i(\bx)$ is $L$-smooth, i.e., $\forall \bx,\by \in \RR^d$, 
\begin{align*}
    \|\nabla F_i(\bx) - \nabla F_i(\by) \|\leq L\|\bx-\by\|.
\end{align*}
\end{assumption}

\begin{assumption}[Bounded Variance]\label{as:bounded-v} 
    Each stochastic gradient is unbiased and has a bounded local variance, i.e., for all $\bx, i \in [N]$, we have
    $\EE \big[ \|\nabla F_i(\bx;\xi)- \nabla F_i(\bx)\|^2\big] \leq \sigma^2$,
and the loss function on each worker has a global variance bound,
$\frac{1}{N}\sum_{i=1}^N \|\nabla F_i(\bx)-\nabla f(\bx)\|^2 \leq \sigma_g^2$.
\end{assumption}
Assumption \ref{as:smooth} and \ref{as:bounded-v} are standard assumptions in federated non-convex optimization literature \cite{li2019communication,yang2021achieving,reddi2021adaptive,wang2022communication,wang2023lightweight}. The global variance upper bound of $\sigma_g^2$ in Assumption~\ref{as:bounded-v} measures the data heterogeneity across clients, and a global variance of $\sigma_g^2 = 0$ indicates a uniform data distribution across clients. 

\begin{assumption}[Bounded Gradient]\label{as:bounded-g}
Each loss function on the $i$-th worker $F_i(\bx)$ has $G$-bounded stochastic gradient on $\ell_2$ norm, i.e., for all $\xi$, we have $\|\nabla F_i(\bx;\xi)\|\leq G$.
\end{assumption}
Assumption \ref{as:bounded-g} is necessary for adaptive gradient algorithms for both general \cite{kingma2014adam, chen2020closing}, distributed \cite{wang2022cdadam} and federated adaptive optimization \cite{reddi2021adaptive, wang2022communication, sun2023efficient}. This is because the effective global learning rate for adaptive gradient methods is $\frac{\eta}{\sqrt{\hat\bv_t}+ \epsilon}$, and we need a lower bound for $\big\|\frac{\eta}{\sqrt{\hat\bv_t}+ \epsilon}\big\|$ to guarantee that the effective learning rate does not vanish to zero. 

\begin{assumption}[Bounded Delay of Gradient Computation] \label{as:bound-g-delay}
    Let $\tau_t^i$ represent the delay for global round $t$ and client $i$ which is applied in Algorithm \ref{alg:fadas}. The delay $\tau_t^i$ is the difference between the current global round $t$ and the global round at which client $i$ started to compute the gradient. We assume that the maximum gradient delay (worst-case delay) is bounded, i.e., $\tau_{\max} = \max_{t \in [T], i \in [N]} \{\tau_t^i\} < \infty$. 
\end{assumption}
Assumption~\ref{as:bound-g-delay} is common in analyzing asynchronous and anarchic FL algorithms which incorporate the gradient delays into their algorithm design \cite{koloskova2022sharper, yang2021achieving, nguyen2022federated, toghani2022unbounded, wang2023tackling}. 

\begin{assumption}[Uniform Arrivals of Gradient Computation]\label{as:uniform_arrival}
    Let the set $\cM_t$ (with size $M$) include clients that transmit their local updates to the server in global round $t$. We assume that the clients' update arrivals are uniformly distributed, i.e., from a theoretical perspective, the $M$ clients in $\cM_t$ are randomly sampled without replacement from all clients $[N]$ according to a uniform distribution\footnote{This assumption is only used for theoretical analysis. Our experiments that show the advantage of FADAS empirically do not rely on this assumption. }.
\end{assumption}
\vspace{-2pt}
Assumption~\ref{as:uniform_arrival} is also discussed in Anarchic FL~\cite{yang2022anarchic}, which has been utilized to analyze the \mbox{AFA-CD} algorithm proposed therein. 
\subsection{Convergence Rate of FADAS}\label{subsec:fadas}
For expository convenience, in the following, we provide the theoretical convergence analysis of FADAS under the case of $\beta_1 = 0$. The theoretical analysis and the proof for the general case of $0 \leq \beta_1 < 1$ are provided in Appendix~\ref{sec:async-thm}. We define the average of the maximum delay over time as $\tau_{\avg} = \frac{1}{T} \sum_{t=1}^T \tau_t^{\max} = \frac{1}{T} \sum_{t=1}^T \max_{i \in [N]} \{\tau_t^i\}$ which is useful in our analysis.
\begin{theorem}\label{thm:fadas}
    Under Assumptions~\ref{as:smooth}--\ref{as:uniform_arrival}, let $T$ represent the total number of global rounds, $K$ be the number of local SGD training steps and $M$ be the number of the accumulated updates (buffer size) in each round. If the learning rate $\eta$ and $\eta_l$ satisfies $\eta \eta_l \leq \min \Big\{ \frac{\epsilon^2 M(N-1)}{180 C_G N(N-M) \tau_{\max} KL }, \frac{\sqrt{\epsilon^3 M(N-1)} }{12\sqrt{C_G N(M-1) \tau_{\max}^3 } KL} \Big\}, \eta_l \leq \frac{\sqrt{\epsilon}}{\sqrt{360 C_G \tau_{\max}}KL}$, then the global iterates $\{\bx_t\}_{t=1}^T$ of Algorithm~\ref{alg:fadas} satisfy
    \begin{align}\label{eq:fadas_full}
    & \frac{1}{T}\sum_{t=1}^T  \EE[\|\nabla f(\bx_t)\|^2] \notag\\
    & \leq \frac{4C_G}{\eta \eta_l K T}\cF + \frac{20 C_G \eta_l^2 K L^2 (\sigma^2 + 6K \sigma_g^2)}{\epsilon} \notag\\
    & \quad+ \bigg[\frac{8C_G \eta^2 \eta_l^2 K L^2 \tau_{\avg}\tau_{\max}}{M \epsilon^3} + \frac{12 C_G \eta \eta_l L}{M\epsilon^2}\bigg] \notag\\
    & \quad\cdot \bigg\{ \sigma^2 + \frac{N-M}{N-1}[15\eta_l^2 K^2 L^2 (\sigma^2 + 6K\sigma_g^2) + 3K\sigma_g^2] \bigg\},
    \end{align}
    where $\cF = f(\bx_1) - f_*$, $f_* =\min_{\bx} f(\bx) > -\infty $ and $C_G = \eta_l K G + \epsilon$.
\end{theorem}
\begin{corollary}\label{cor:fadas}
    If we choose the global learning rate  $\eta = \Theta (\sqrt{M})$ and $\eta_l = \Theta\Big( \frac{\sqrt{\cF}}{\sqrt{TK(\sigma^2 + K\sigma_g^2)}} \Big)$ in Theorem \ref{thm:fadas}, then for sufficiently large $T$, the global iterates $\{\bx_t\}_{t=1}^T$ of Algorithm~\ref{alg:fadas} satisfy 
    \begin{align}\label{eq:fadas}
    & \frac{1}{T} \sum_{t=1}^T  \EE[\|\nabla f(\bx_t)\|^2] \leq \cO\bigg(\frac{\sqrt{\cF } \sigma}{\sqrt{TKM}} + \frac{\sqrt{\cF} \sigma_g}{\sqrt{TM}}  \notag\\
    & + \frac{\cF}{T} + \frac{\cF G}{T \sqrt{M}} + \frac{\cF \tau_{\max}\tau_{\avg}}{T}\bigg),
    \end{align}%
\end{corollary}

\begin{remark}\label{rm:fadas2}
    Corollary \ref{cor:fadas} suggests that given sufficiently large $T$ and relatively small worst-case delay $\tau_{\max}$,
    the proposed FADAS (without delay-adaptive learning rate) achieves a convergence rate of $\cO\big(\frac{1}{\sqrt{TM}}\big)$ w.r.t. $T$ and $M$.
\end{remark}    
\textbf{Comparison to asynchronous FL methods.}
Compared with the analysis for FedBuff in \citet{nguyen2022federated} and \citet{toghani2022unbounded}, our analysis for FADAS obtains a relaxed dependency on the worst-case gradient delay $\tau_{\max}$, and FADAS achieves a slightly better rate on non-dominant term than $\cO\big(\frac{1}{\sqrt{T}} + \frac{\tau_{\max}^2}{T} \big)$ obtained in \citet{toghani2022unbounded}. Moreover, \citet{wang2023tackling} also studied the convergence for FedBuff with relaxed requirements for $\tau_{\max}$, and our FADAS achieves a similar convergence of $\cO\big(\frac{1}{\sqrt{TM}} + \frac{\tau_{\max}\tau_{\avg}}{T} \big)$ as in \citet{wang2023tackling}. It is worthwhile to mention that recently CA$^2$FL \cite{wang2023tackling} improves the convergence of asynchronous FL under heterogeneous data distributions, while the improvement is obtained by using the cached variable on the server for global update calibration. 

Note that when $\tau_{\max}$ in Eq. \eqref{eq:fadas} is large, particularly in cases where $\tau_{\max} \geq \frac{\sqrt{T}}{\sqrt{M}}$, then $\frac{\tau_{\max} \tau_{\avg} }{T}$ becomes the dominant term in the convergence rate. This implies that a large worst-case delay $\tau_{\max}$ may lead to a worse convergence rate. 
In the next subsection, we demonstrate that the delay-adaptive learning rate strategy can relieve this problem and enhance FADAS with better resilience to large worst-case delays.

\subsection{Convergence Rate of Delay-adaptive FADAS} \label{subsec:da_fadas}

In the following, we provide the convergence analysis for delay-adaptive FADAS with $\beta_1 = 0$. To get started, we first define the median of the maximum delay over all communication rounds $[T]$:
\begin{align} \label{eq:tau_med}
    \tau_{\med} = \med\{\tau_1^{\max}, \tau_2^{\max},..., \tau_T^{\max}\}.
\end{align}
The definition of $\tau_{\med}$ implies that the number of global update rounds that have a maximum delay greater than $\tau_{\med}$ is less than half of the total number of global updates $T$. With this definition, we present the following theorem characterizing the convergence rate of delay-adaptive FADAS.

\begin{theorem}\label{thm:fadas_da}
    Under Assumptions~\ref{as:smooth}--\ref{as:uniform_arrival}, let $T$ be the total number of global rounds, $K$ be the number of local SGD training steps and $M$ be the number of the buffer size in each round. If the learning rate $\eta$ and $\eta_l$ satisfies $\eta \eta_l \leq \min \Big\{ \frac{\epsilon^2 M(N-1)}{60 C_G N(N-M) \tau_{\max} KL }, \frac{\sqrt{\epsilon^3 M(N-1)} }{12\sqrt{C_G N(M-1) \tau_{\max}^3 } KL} \Big\}, \eta_l \leq \frac{\sqrt{\epsilon}}{\sqrt{360 C_G \tau_{\max}}KL}$ and $\eta \leq \frac{\sqrt{M}}{\tau_c}$, then the global iterates $\{\bx_t\}_{t=1}^T$ of Algorithm~\ref{alg:fadas} satisfy
    \begin{align}\label{eq:fadas_d_full}
    & \frac{1}{\sum_{t=1}^T \eta_t } \sum_{t=1}^T \eta_t \EE[\|\nabla f(\bx_t)\|^2]\notag\\
    & \leq \frac{4C_G}{\eta \eta_l K T}\cF + \frac{20 C_G \eta_l^2 K L^2 (\sigma^2 + 6K \sigma_g^2)}{\epsilon} \notag\\
    & + \frac{8C_G \eta^3 \eta_l^2 K L^2 T\tau_{\avg}}{M \epsilon^3 \sum_{t=1}^T \eta_t } \sigma^2 + \frac{8C_G \eta^2 \eta_l^2 K L^2 T\tau_{\avg}}{\sqrt{M}\epsilon^3 \sum_{t=1}^T \eta_t } \notag\\
    & \cdot \frac{N-M}{N-1}[15\eta_l^2 K^2 L^2 (\sigma^2 + 6K\sigma_g^2) + 3K\sigma_g^2] + \frac{4 C_G \eta \eta_l L}{M\epsilon^2} \notag\\
    & \cdot \bigg\{\sigma^2 + \frac{N-M}{N-1}[15\eta_l^2 K^2 L^2 (\sigma^2 + 6K\sigma_g^2) + 3K\sigma_g^2] \bigg\},
    \end{align}
    where $\cF = f(\bx_1) - f_*$, $f_* =\min_{\bx} f(\bx) > -\infty $ and $C_G = \eta_l K G + \epsilon$.
\end{theorem}
\begin{corollary}\label{cor:fadas_da}
    If we pick $\tau_c = \tau_{\med}$, the global learning rate  $\eta = \Theta (\sqrt{M}/\tau_c)$ and $\eta_l = \Theta\big(\frac{ \tau_c \sqrt{\cF}}{\sqrt{TK(\sigma^2 + K\sigma_g^2)}} \big)$, then for sufficiently large $T$, the global iterates $\{\bx_t\}_{t=1}^T$ of Algorithm \ref{alg:fadas} satisfy
    \begin{align}\label{eq:fadas_da}
    & \frac{1}{\sum_{t=1}^T \eta_t } \sum_{t=1}^T \eta_t \EE[\|\nabla f(\bx_t)\|^2] \leq \cO \bigg(\frac{\sqrt{\cF} \sigma}{\sqrt{TKM}} + \frac{\sqrt{\cF} \sigma_g}{\sqrt{TM}} \notag\\
    & + \frac{\cF G \tau_c }{T \sqrt{M}} + \frac{\cF \tau_{\avg}}{T} + \frac{\cF (\tau_c^2 + \tau_c \tau_{\avg})}{T}\bigg).
   \end{align}
\end{corollary}

\begin{remark}\label{rm:fadas2_da}
    Corollary \ref{cor:fadas_da} suggests that with sufficiently large $T$, delay-adaptive FADAS also achieves a convergence rate of $\cO\big( \frac{1}{\sqrt{TM}}\big)$ w.r.t. $T$ and $M$. 
\end{remark}  

\begin{remark} \label{rm:fadas3_da}
    Compared to the convergence rate in Corollary~\ref{cor:fadas}, the convergence rate in Corollary~\ref{cor:fadas_da} does not rely on the (possibly large) worst-case delay $\tau_{\max}$. In cases where $\tau_c = \tau_{\med} \approx \tau_{\avg} \ll \tau_{\max}$, Corollary \ref{cor:fadas_da} relaxes the requirement from $\tau_{\max}$ to $\tau_{\med}$ for achieving the desired convergence rate. Since $\tau_{\med}$ describes the median of $\tau_t^{\max} = \max_{ i \in [N]} \{\tau_t^i\}$ in each round $t$, the  convergence rate in Corollary \ref{cor:fadas_da} is less sensitive to stragglers who may cause a large worst-case delay in the system. 
\end{remark}

\section{Experiments}
We explore the performance of our proposed FADAS algorithm through experiments on vision and language tasks, using the CIFAR-10/100 \cite{krizhevsky2009learning} datasets with ResNet-18 model \cite{he2016deep} for vision tasks, and applying the pre-trained BERT base model \cite{devlin2018bert} for fine-tuning several datasets from the GLUE benchmark dataset \cite{wang2018glue} for language tasks. We compare our proposed FADAS algorithm against asynchronous FL baselines, such as FedBuff (without differential privacy) \cite{nguyen2022federated} and FedAsync \cite{xie2019asynchronous}, a synchronous SGD-based FL baseline FedAvg \cite{mcmahan2017communication}, and a synchronous adaptive FL baseline FedAMS \cite{wang2022communication}. We summarize some crucial implementation details in the following, and we leave some additional results and experiment details to Appendix \ref{sec:appendix_exp}. Our code can be found at \url{https://github.com/yujiaw98/FADAS}.

\textbf{Overview of vision tasks' implementation. }
We set up a total of 100 clients for the \textit{mild delay} scenario, in which the concurrency $M_c = 20$ and the buffer size $M=10$ by default. We also set up a total of 50 clients for the \textit{large worst-case delay} scenario, with $M_c = 25$ and $M=5$ correspondingly. 
For both settings, we partition the data on clients based on the Dirichlet distribution following \citet{Wang2020Federated,wang2020tackling}, and the parameter $\alpha$ used in Dirichlet sampling determines the degree of data heterogeneity. We apply two levels of data heterogeneity with $\alpha=0.1$ and $\alpha=0.3$. Each client conducts two local epochs of training, and the mini-batch size is 50 for each client. The local optimizer for all methods is SGD with weight decay $10^{-4}$, and we grid search the global and local learning rates individually for each method.  

\textbf{Overview of language tasks' implementation. }
Considering the total number of data samples in the language classification datasets, we set up a total of 10 clients, partition the data on clients based on the labels, and we apply a heterogeneity level of $\alpha=0.6$. Each client conducts one local epoch and the mini-batch size is 32 for each client. The local optimizer for all methods is SGD with weight decay $10^{-4}$, and we grid search the global and local learning rates individually for each method. We set the concurrency $M_c = 5$ and buffer size $M=3$ by default. We employ the widely-used low-rank adaptation method, LoRA \cite{hu2021lora}, as a parameter-efficient fine-tuning strategy for our language classification tasks. This involves freezing the original pre-trained weight matrix $\bW_0 \in \RR^{d\times k}$ and fine-tuning $\Delta \bW$ through low-rank decomposition, where $\bW = \bW_0 + \alpha_{\text{LoRA}} \Delta \bW = \bW_0 + \alpha_{\text{LoRA}}\bB \bA$, $\bB \in \RR^{d\times r}$, and $\bA \in \RR^{r\times k}$, and we adopt $r=1$ and $\alpha_{\text{LoRA}} = 8$ in our experiments. 

\textbf{Overview of delay simulation. } In our experiments, we simulate the asynchronous environment as follows. 
Initially, we partition clients into three categories, including \texttt{Small}, \texttt{Medium}, and \texttt{Large} delay, at the start of training and tag them with a label reflective of their delay magnitude.  This partitioning was executed via a Dirichlet sampling process controlled by the parameter $\gamma$. A smaller $\gamma$ value corresponds to a higher proportion of clients experiencing large delays. Unless otherwise specified in subsequent experiments, we set $\gamma = 1$. To mimic actual wall-clock running times within each delay category, we apply uniform sampling at each round for each client. We adopt the following uniform distributions to simulate wall-clock running time for both the \textit{large worst-case delay} and \textit{mild delay} settings as shown in Table \ref{tab:delay}. 

\subsection{Results on Vision Tasks}
\textbf{Large worst-case delay. } Under this setting, we simulate the wall-clock running time by letting a small proportion of clients have more significant delays than other clients.  Tables \ref{tab:large_delay_cifar10} and \ref{tab:large_delay_cifar100} show the overall performance of training the ResNet-18 model on CIFAR-10 and CIFAR-100, respectively. The results show that FADAS, especially with a delay-adaptive learning rate, offers significant advantages in terms of test accuracy. Compared to FedAsync and FedBuff, both FADAS methods achieve higher accuracy, and FADAS with delay-adaptive learning rates is shown to be more stable during the learning process with lower standard derivation. In these experiments, we conduct a total of $T=500$ global communication rounds, and the maximum delay $\tau_{\max} = 127$, which even more than a quarter of the total number of global communication rounds. Notably, as seen in Tables \ref{tab:large_delay_cifar10} and \ref{tab:large_delay_cifar100}, FedAsync shows severely fluctuating in test accuracy, suggesting that it may be less reliable in situations with large worst-case delays.
\vspace{-5pt}
\begin{table}[H]
    \centering
    \caption{Overview for wall-clock delay simulation (in units of 10 seconds). }
    \vskip 0.05in
    {\small
    \begin{tabular}{llll}
        \toprule
        Delay & \texttt{Small} & \texttt{Medium} & \texttt{Large} \\
        \midrule
        \textit{Large worst-case} & $U(1,2)$ & $U(3, 5)$ & $U(50, 80)$ \\
        \textit{Mild} & $U(1,2)$ & $U(3, 5)$ & $U(5, 8)$ \\
        \bottomrule
    \end{tabular}
    }
    \label{tab:delay}
\end{table}
\vspace{-5pt}

\textbf{Mild delay.}
Under this setting, we simulate the wall-clock running time for clients by assuming that all clients can finish their local training within a comparable duration (see Table~\ref{tab:delay}). Tables \ref{tab:cifar10} and \ref{tab:cifar100} show the overall performance of training the ResNet-18 model on CIFAR-10 and CIFAR-100 under mild delay. 
The results highlight that both FADAS and its delay-adaptive variant achieve superior test accuracy than FedAsync and FedBuff. 
\begin{table}[t]
    \centering
    \caption{The test accuracy on training ResNet-18 model on CIFAR-10 dataset with two data heterogeneity levels in a \textit{large worst-case} delay scenario for 500 communication rounds. We report the average accuracy and standard derivation over the last 5 rounds, and we abbreviate delay-adaptive FADAS to FADAS$_{\text{da}}$ in this and subsequent tables.}
    \vskip 0.05in
    {\small
    \begin{tabular}{l|cc}
    \toprule 
    & Dir(0.1) & Dir (0.3)\\
    Method & Acc. \& std. & Acc. \& std. \\
    \midrule
    FedAsync & 50.92 $\pm$ 5.03 & 75.3 $\pm$ 6.18 \\
    FedBuff & 38.68 $\pm$ 8.16 & 51.32 $\pm$ 4.43 \\
    FADAS & 72.0 $\pm$ 0.94 & 73.27 $\pm$ 1.37 \\
    FADAS$_{\text{da}}$ & \textbf{73.96} $\pm$ 3.54 & \textbf{79.68} $\pm$ 2.14 \\
    \bottomrule
    \end{tabular}
    }
    \label{tab:large_delay_cifar10}
\end{table}

\begin{table}[t]
    \centering
    \caption{The test accuracy on training ResNet-18 model on CIFAR-100 dataset with two data heterogeneity levels in a \textit{large worst-case} delay scenario for 500 communication rounds. }
    \vskip 0.05in
    {\small
    \begin{tabular}{l|cc}
    \toprule 
    & Dir(0.1) & Dir (0.3)\\
    Method & Acc. \& std. & Acc. \& std. \\
    \midrule
    FedAsync & 46.51 $\pm$ 4.76 &  38.55 $\pm$ 7.36\\
    FedBuff & 13.04 $\pm$ 5.5 & 18.63 $\pm$ 5.13 \\
    FADAS & 47.84 $\pm$ 0.59 & 53.64 $\pm$ 0.52 \\
    FADAS$_{\text{da}}$ & \textbf{50.31} $\pm$ 1.0 & \textbf{57.18} $\pm$ 0.31 \\
    \bottomrule
    \end{tabular}
    }
    \label{tab:large_delay_cifar100}
\end{table}

\subsection{Results on Language Tasks}
The performance for fine-tuning the BERT base model on three GLUE benchmark datasets, RTE, MRPC, and SST-2, under mild delay conditions are shown in Table \ref{tab:bert}, which illustrates that FADAS and its delay-adaptive counterpart consistently outperform the results of FedAsync and FedBuff across the three datasets. FedAsync achieves good performance in SST-2 but is less satisfactory in RTE and MRPC, and FedBuff presents an overall lower accuracy with larger standard derivation compared with FADAS. The delay-adaptive FADAS shows parity with the standard FADAS algorithm under mild delays. Moreover, FADAS achieves significant accuracy improvements on RTE and MRPC datasets against the SGD-based asynchronous FL baselines, further demonstrating the intuition of developing the FADAS method.

\textbf{Running time speedup. }
Table \ref{tab:running_time} demonstrates the efficiency of FADAS and its delay-adaptive variant by comparing their performance with two synchronous FL methods in reaching the target validation accuracy across different dataset. Notably, FADAS consistently outperforms FedAvg and FedAMS in terms of wall-clock running time, requiring significantly fewer time units to reach the desired accuracy levels. In vision classification tasks such as CIFAR-10 and CIFAR-100, the standard FADAS shows a significant reduction in training time, achieving 8 $\times$ speedup than FedAvg and more than 2.5 $\times$ speedup than FedAMS. The delay-adaptive FADAS shows similar results as the standard version. For language classification tasks, FADAS also improves the training time compared with FedAMS and FedAvg. These results highlight the scalability and efficiency of FADAS, especially when considering the computational constraints in practical FL environments.

\begin{table}[t]
    \centering
    \caption{The test accuracy on training ResNet-18 model on CIFAR-10 dataset with two data heterogeneity levels under \textit{mild delay} scenario. }
    \vskip 0.05in
    {\small
    \begin{tabular}{l|cc}
    \toprule 
    & Dir(0.1) & Dir (0.3)\\
    Method & Acc. \& std. & Acc. \& std. \\
    \midrule
    FedAsync & 42.48 $\pm$ 4.93& 71.76 $\pm$ 3.85\\
    FedBuff & 72.15 $\pm$ 2.71 & 79.82 $\pm$ 3.25 \\
    FADAS & 77.68 $\pm$ 2.32 & 82.93 $\pm$ 0.81\\
    FADAS$_{\text{da}}$& \textbf{78.93} $\pm$ 0.83 & \textbf{83.91} $\pm$ 0.54\\
    \bottomrule
    \end{tabular}
    }
    \label{tab:cifar10}
\end{table}
\vspace{-0.2in}
\begin{table}[t]
    \centering
    \caption{The test accuracy on training ResNet-18 model on CIFAR-100 dataset with two data heterogeneity levels under \textit{mild delay} scenario. }
    \vskip 0.05in
    {\small
    \begin{tabular}{l|cc}
    \toprule 
    & Dir(0.1) & Dir (0.3)\\
    Method & Acc. \& std. & Acc. \& std. \\
    \midrule
    FedAsync & 45.26 $\pm$ 7.04  & 53.41 $\pm$ 8.94 \\
    FedBuff & 53.70 $\pm$ 1.13 & 56.26 $\pm$ 1.64 \\
    FADAS & \textbf{57.37} $\pm$ 0.47 & \textbf{61.22} $\pm$ 0.31\\
    FADAS$_{\text{da}}$ & 57.21 $\pm$ 0.45 & 60.34 $\pm$ 0.42\\
    \bottomrule
    \end{tabular}
    }
    \label{tab:cifar100}
\end{table}

\begin{table}[b]
    \centering
    \caption{The test accuracy on parameter-efficient fine-tuning BERT base model on three datasets from GLUE benchmark with heterogeneous data partitioned and \textit{mild delay}.}
    \vskip 0.05in
    {\small
    \begin{tabular}{l|ccc}
    \toprule 
    & RTE & MRPC & SST-2 \\
    Method & Acc. \& std. & Acc. \& std. & Acc. \& std. \\
    \midrule
    FedAsync &  49.46 $\pm$ 2.66 & 69.71 $\pm$ 1.02 & 90.02 $\pm$ 0.79\\
    FedBuff & 61.61 $\pm$ 4.90 & 76.80 $\pm$ 6.05 & 78.37 $\pm$ 4.86 \\
    FADAS & 64.26 $\pm$ 2.30 & \textbf{83.33} $\pm$ 1.20 & \textbf{90.76} $\pm$ 0.26 \\
    FADAS$_{\text{da}}$ & \textbf{65.10} $\pm$ 2.40 & 83.09 $\pm$ 1.71 & 90.05 $\pm$ 1.80 \\
    \bottomrule
    \end{tabular}
    }
    \label{tab:bert}
\end{table}

\begin{table*}[ht!]
    \centering
    \caption{Training/fine-tuning time simulation (in units of 10 seconds) to reach target test accuracy on the server under \textit{mild delay} scenarios. For each dataset, the concurrency $M_c$ is fixed for fair comparison.}
    \vskip 0.05in
    {\small
    \begin{tabular}{l|c|ccccc}
    \toprule
     & Acc. & FedAvg & FedAMS & FADAS & FADAS$_{\text{da}}$\\
    \midrule
    CIFAR-10 & 75\% & 2257.7 & 648.7 & \textbf{228.0} & 237.5 \\
    CIFAR-100 & 50\% & 1806.3 & 546.9 & \textbf{209.8 }& \textbf{209.8}\\
    RTE & 63\% & 921.9 & 412.4 & \textbf{376.2} & 436.9 \\
    MRPC & 80\% & 1018.1 & 424.0 & \textbf{368.3} & 370.1\\
    SST-2 & 90\% & - &  495.2 & 73.8 & \textbf{57.2} \\
    \bottomrule
    \end{tabular}
    }
    \label{tab:running_time}
\end{table*}

\vspace{10pt}
\subsection{Ablation studies}

\textbf{Sensitivity of delay adaptive learning rates. }
Figure~\ref{fig:ablation}~(a) exhibits the ablation study for different delay threshold $\tau_c$ for the delay-adaptive FADAS under the scenario of \textit{large worst-case delays}. Following Eq.~(\ref{eq:delay_adapt}), $\tau_c$ provides a threshold so that we reduce the learning rate if there exists a client with extremely large delay. The experiment compares the accuracy of three thresholds $\tau_c=1,4,8,10$, and $\tau_c = 4$ shows very similar test accuracy as $\tau_c = 10$. The result in Figure~\ref{fig:ablation}~(a) shows that using $\tau_c=8$ obtains a slightly better result than using $\tau_c =1 $, $\tau_c =4$, and $\tau_c =10$.
It is interesting that in this \textit{large worst-case delays} setting, we observe the average of the maximum delay $\tau_{\avg} = 10.89$, the median of the maximum delay $\tau_{\med} = 6.0$, and maximum delay during training is $\tau_{\max} = 127$, which shows $\tau_{\med} \approx \tau_{\avg} \ll \tau_{\max}$, confirming the practicality of our analysis  as discussed in Remark \ref{rm:fadas3_da}. Together with the theoretical and experimental results, we find that the optimal choice of $\tau_c$ may depend on the actual delay during training.

\begin{figure}[b]
    \centering
    \subfigure[Ablation on $\tau_c$]{\includegraphics[width=0.23\textwidth]{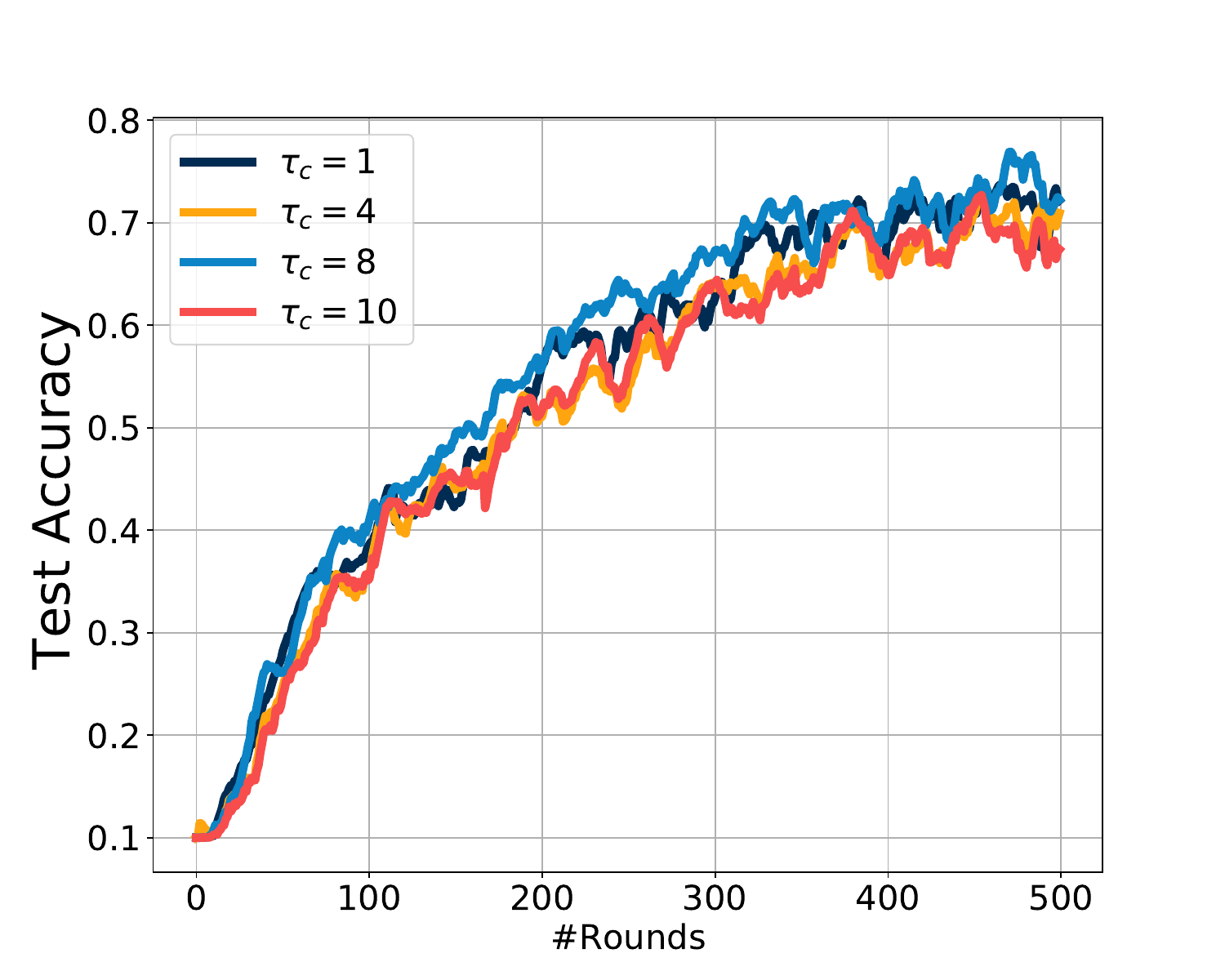}}
    \subfigure[Ablation on concurrency]{\includegraphics[width=0.23\textwidth]{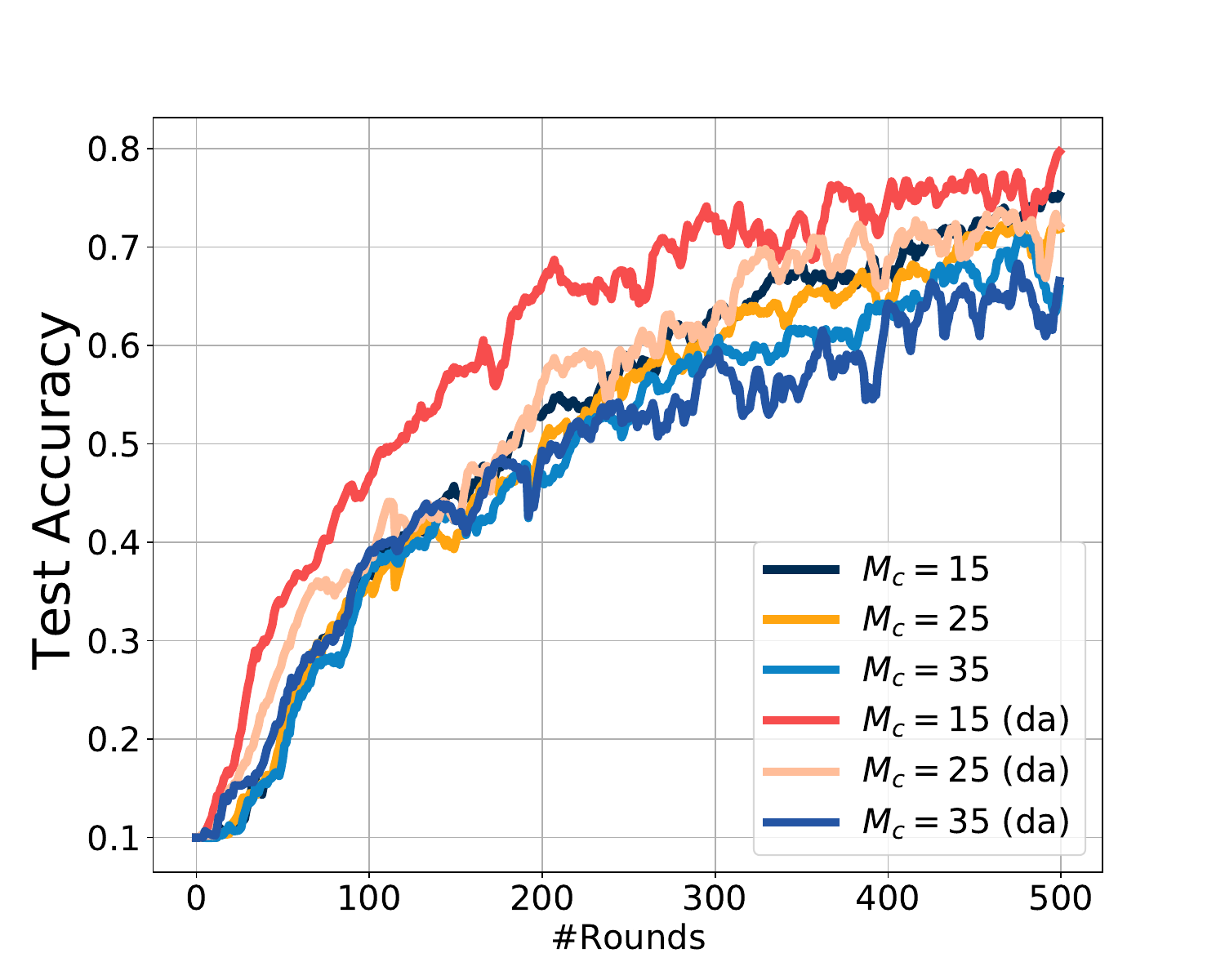}}
    \vspace{-1em}
    
    \subfigure[Ablation on buffer size]{\includegraphics[width=0.23\textwidth]{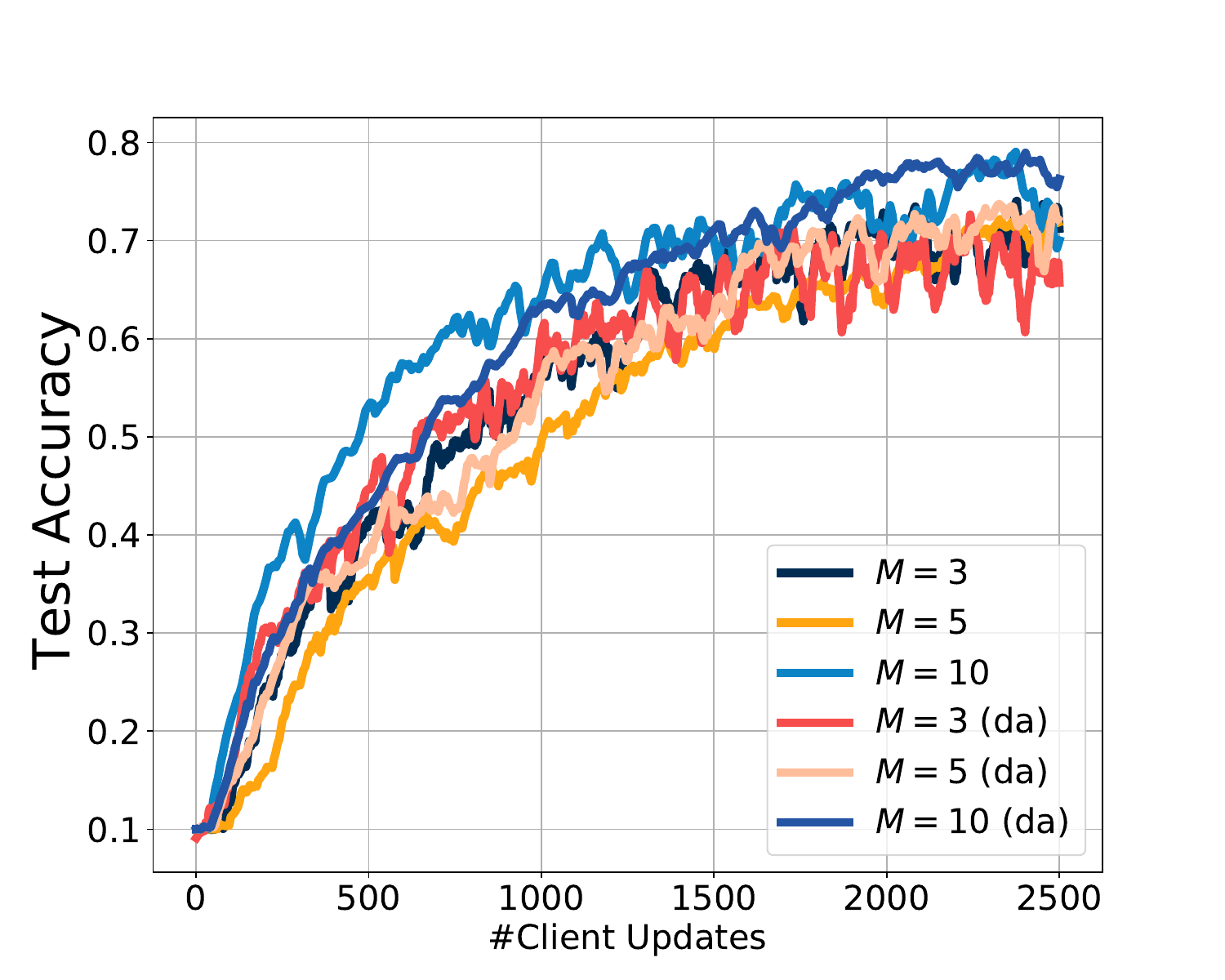}}
    \subfigure[Ablation on buffer size]{\includegraphics[width=0.23\textwidth]{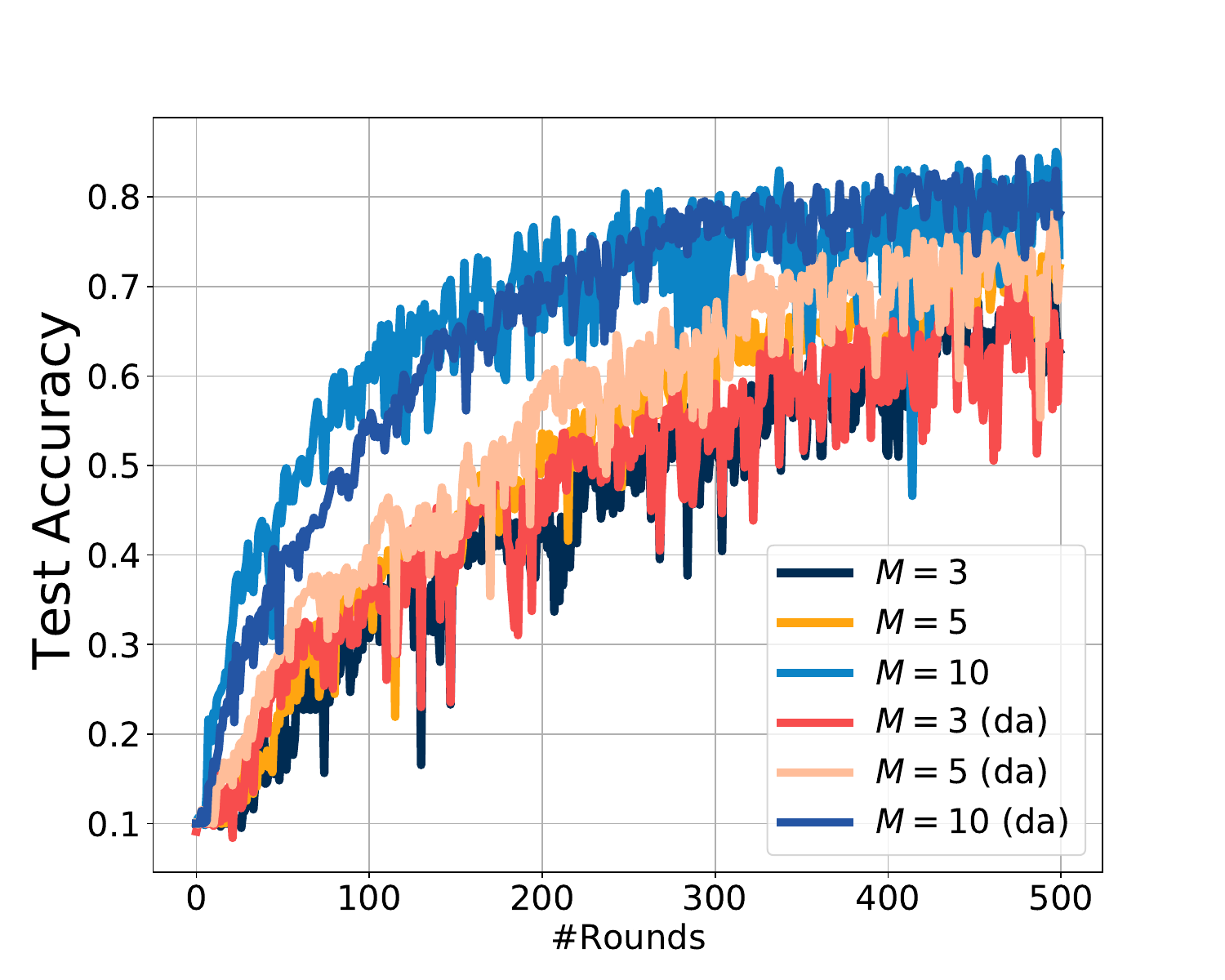}}
    \subfigure[Run time for FADAS]{\includegraphics[width=0.23\textwidth]{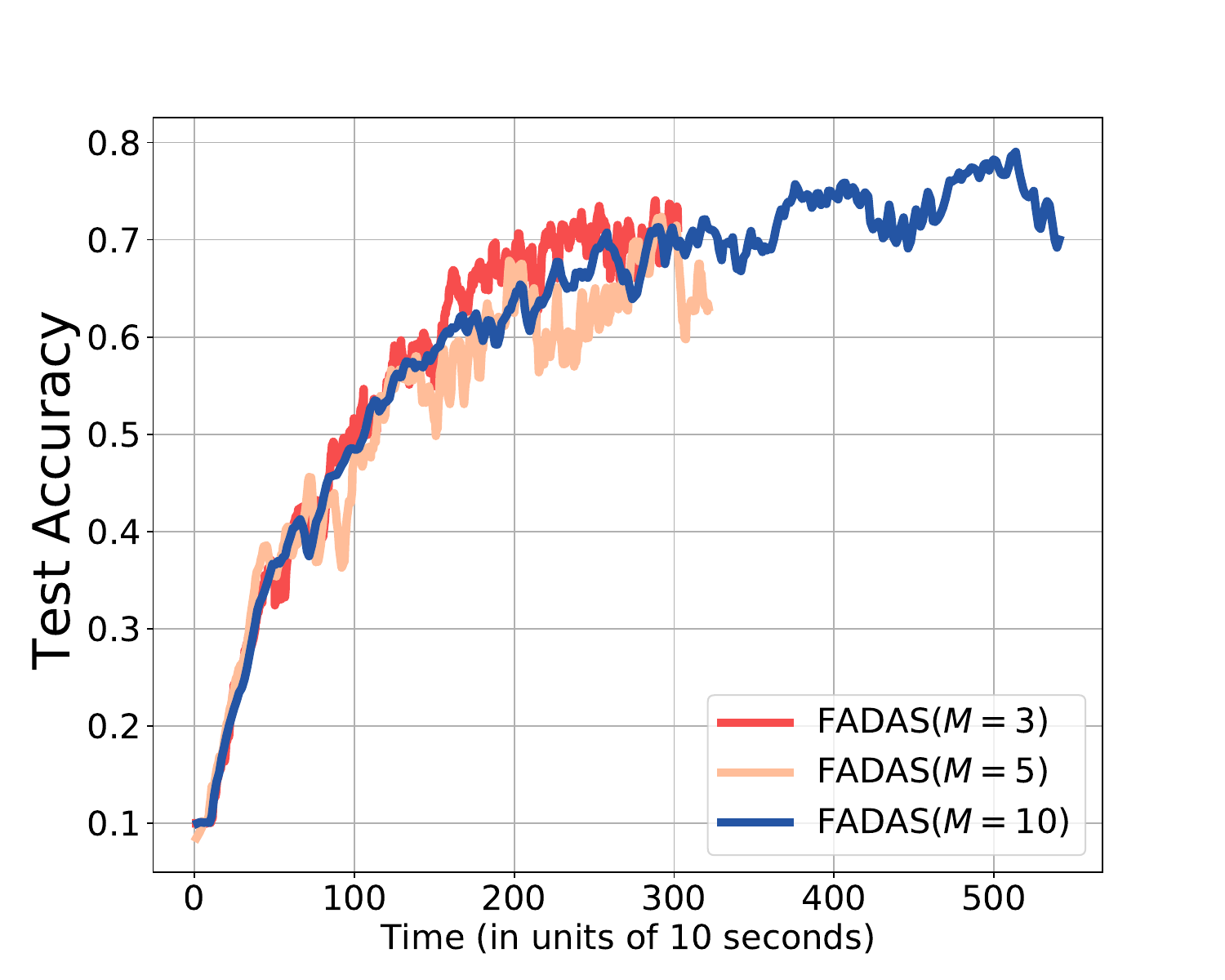}}
    \subfigure[Run time for FADAS$_{\text{da}}$]{\includegraphics[width=0.23\textwidth]{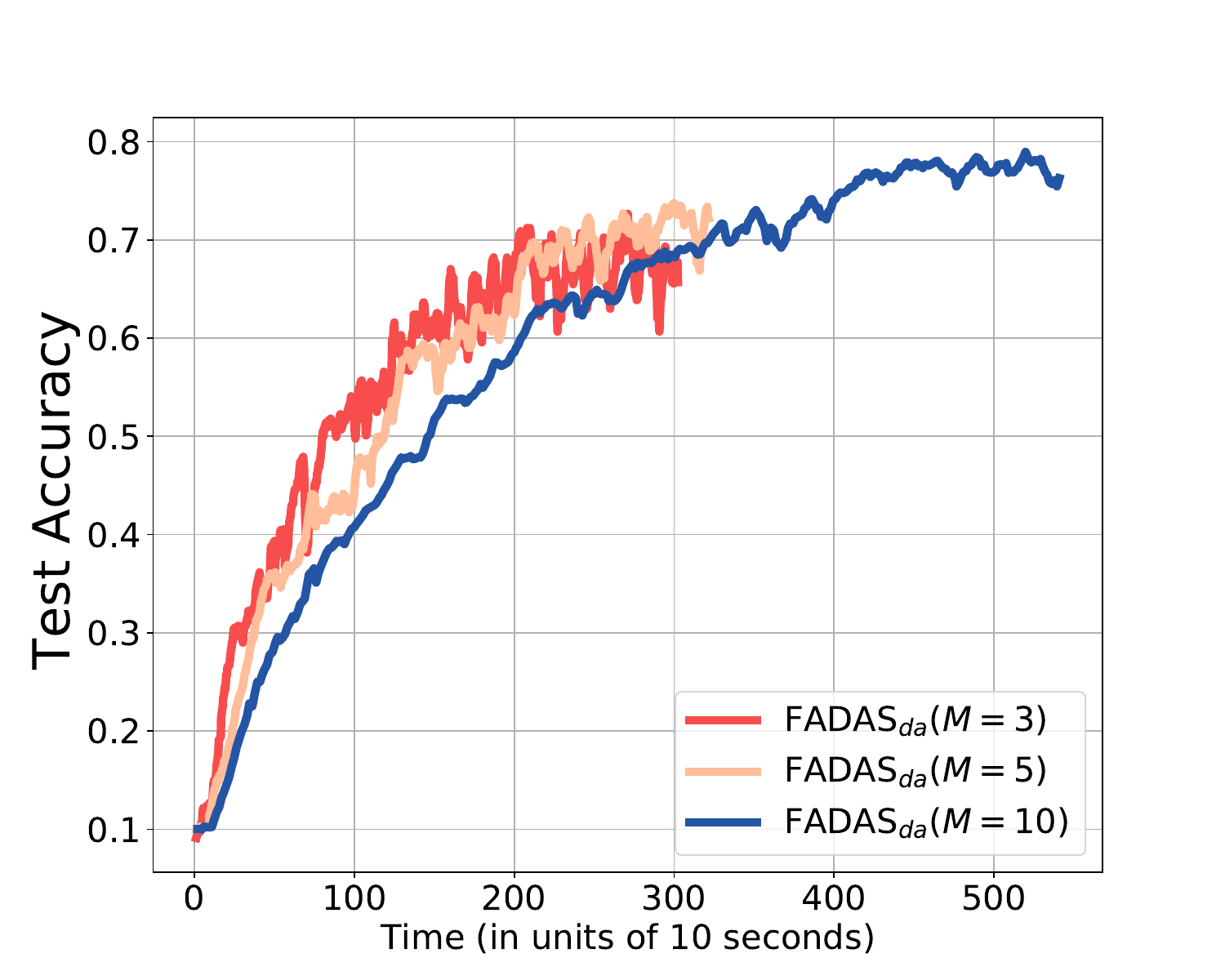}}
    \caption{Several ablation studies based on training ResNet-18 model on CIFAR-10 data under \textit{large worst-case} delay setting.}
    \label{fig:ablation}
\end{figure}

\textbf{Ablation for concurrency $M_c$ and buffer size $M$. }
Figure~\ref{fig:ablation}~(b) presents the test accuracy of both the standard and delay-adaptive FADAS for different concurrency levels $M_c$, given the same buffer size $M=5$. The delay-adaptive FADAS achieves higher accuracy than FADAS when concurrency $M_c = 15$ and $M_c = 25$, and the delay-adaptive version has worse accuracy at larger concurrency $M_c = 35$.

Figure \ref{fig:ablation} (c) presents an ablation study on buffer size for our proposed FADAS algorithm. It compares the performance of buffer sizes from $M = 3, 5, 10$ with their delay-adaptive counterparts over total client updates, i.e., the number of times the server receives updates from clients. It shows that with the same number of client trips, increasing the buffer size $M$ tends to achieve higher accuracy. This is also due to the design of the concurrency-buffer size framework, as increasing the buffer size moves closer to traditional synchronous FL algorithms, i.e., clients are more likely to get up-to-date with the server. We also provide the comparison w.r.t. global communication round in Figure~\ref{fig:ablation}~(d). Figure~\ref{fig:ablation}~(d) shows that as the buffer size $M$ increases, i.e., the number of clients contributing to one step of global update increases, the test accuracy also increases. 

Moreover, we simulate the running time (similar to the setting for Table \ref{tab:running_time}) for different buffer sizes $M$ to investigate the time efficiency for adopting different buffer sizes. Figure \ref{fig:ablation} (e) and (f) show the run time for FADAS and delay-adaptive FADAS. They reveal that a smaller buffer size ($M=3$) may have less training time to achieve a target accuracy, e.g., 70\%. These results demonstrate that using smaller buffer sizes may yield higher accuracy in the early stage of training. In conjunction with the results shown in Figure \ref{fig:ablation} (c) and (d), we think there is a trade-off between the time of reaching some initial target accuracy (that is slightly lower than the final accuracy) and the final accuracy with regard to the buffer size. A larger buffer size $M$ may yield improved final accuracy at convergence, but it also means that the server needs to wait for slower clients and there are less frequent updates of the global model, so the training speed at initial rounds can be slower.

\section{Conclusion}
In this paper, we propose FADAS, a novel asynchronous FL method that addresses the challenges of asynchronous updates in adaptive federated optimization. Based on the standard FADAS, we further integrate delay-adaptive learning rates to enhance the resiliency to stragglers with large delays. We theoretically establish the convergence rate for both standard and delay-adaptive FADAS under non-convex stochastic settings. Our theoretical analysis indicates that the delay-adaptive algorithm substantially reduces the impact of severe worst-case delays on the convergence rate. Empirical evaluations across multiple tasks affirm that FADAS outperforms existing asynchronous FL methods and offers improved training efficiency compared to synchronous adaptive FL methods.

\section*{Acknowledgments}
We thank the anonymous reviewers for their helpful comments. This work is partially supported by the National Science Foundation under Grant No. 2348541. The views and conclusions contained in this paper are those of the authors and should not be interpreted as representing any funding agencies.

\section*{Impact Statement}
This paper will make long-lasting contributions to the field of asynchronous federated learning and adaptive optimization. The focus of this work is on the technical advancement and optimization development of federated learning algorithms, and while there are numerous potential societal impacts of machine learning at large, this research does not necessitate a specific discussion on the societal consequences.



\bibliography{icml2024/0_main}
\bibliographystyle{icml2024}

\newpage
\appendix
\onecolumn

\section{Convergence analysis for adaptive asynchronous FL}\label{sec:async-thm}
\begin{proof}[Proof of Theorem \ref{thm:fadas}]
Here we directly start with general $\beta_1 \geq 0$ cases. Following several previous works studied centralized and federated adaptive methods \cite{chen2018convergence,wang2022communication}, we adopt an auxiliary Lyapunov sequence $\bz_t$, and assume $\bx_0 = \bx_1$, then for each $t\geq 1$, we have 
\begin{align}
    \bz_t = \bx_t + \frac{\beta_1}{1-\beta_1} (\bx_t - \bx_{t-1}) = \frac{1}{1-\beta_1}\bx_t - \frac{\beta_1}{1-\beta_1}\bx_{t-1}.
\end{align}

For the difference between $\bz_{t+1}$ and $\bz_t$, we have 
\begin{align}
    \bz_{t+1} - \bz_t & = \frac{1}{1-\beta_1} (\bx_{t+1} - \bx_t) - \frac{\beta_1}{1-\beta_1} (\bx_t - \bx_{t-1}) \notag\\
    & = \frac{1}{1-\beta_1} \cdot \eta \frac{\bbm_t}{\sqrt{\hat\bv_t} + \epsilon} - \frac{\beta_1}{1-\beta_1} \cdot \eta \frac{\bbm_{t-1}}{\sqrt{\hat\bv_{t-1}}+ \epsilon} \notag\\
    & = \frac{1}{1-\beta_1} \cdot \eta \frac{1}{\sqrt{\hat\bv_t} + \epsilon} [\beta_1 \bbm_{t-1} + (1-\beta_1)  \bDelta_t] - \frac{\beta_1}{1-\beta_1} \cdot \eta \frac{\bbm_{t-1}}{\sqrt{\hat\bv_{t-1}}+ \epsilon} \notag\\
    & = \eta \frac{ \bDelta_t}{\sqrt{\hat\bv_t} + \epsilon}  - \frac{\beta_1}{1-\beta_1} \cdot \bigg(\frac{\eta}{\sqrt{\hat\bv_t}+ \epsilon} - \frac{\eta}{\sqrt{\hat\bv_{t-1}}+ \epsilon} \bigg)\bbm_{t-1},
\end{align}
where 
$ \bDelta_t = -\frac{\eta_l }{M} \sum_{i \in \cM_t} \sum_{k=0}^{K-1} \bg_{t-\tau_t^i, k}^i = -\frac{\eta_l }{M} \sum_{i \in \cM_t} \sum_{k=0}^{K-1} \nabla F_i(\bx_{t-\tau_t^i, k}^i; \xi)$ and $\cM_t$ be the set that include client send the local updates to the server at global round $t$.

From Assumption \ref{as:smooth}, $f$ is $L$-smooth, taking the total expectation over all previous round, $0,1,...,t-1$  on the auxiliary sequence $\bz_t$,
\begin{align}\label{eq:f-Lsmooth-noncomp}
& \EE[f(\bz_{t+1})-f(\bz_{t})] \notag\\
& = \EE[f(\bz_{t+1})]-f(\bz_{t})] \notag\\
& \leq \EE[\langle \nabla f(\bz_t), \bz_{t+1}-\bz_t \rangle] + \frac{L}{2} \EE[\|\bz_{t+1}-\bz_t\|^2] \notag\\
& = \underbrace{\EE\bigg[\bigg\langle \nabla f(\bx_{t}), \eta \frac{ \bDelta_t}{\sqrt{\hat\bv_t} + \epsilon} \bigg\rangle\bigg]}_{I_1} \underbrace{ - \EE\bigg[\bigg\langle \nabla f(\bz_{t}), \frac{\beta_1}{1-\beta_1} \cdot \eta \bigg(\frac{1}{\sqrt{\hat\bv_t}+ \epsilon} - \frac{1}{\sqrt{\hat\bv_{t-1}}+ \epsilon} \bigg)\bbm_{t-1} \bigg\rangle\bigg]}_{I_2} \notag\\
& \quad + \underbrace{\frac{\eta^2 L}{2}\EE\bigg[\bigg\| \frac{ \bDelta_t}{\sqrt{\hat\bv_t} + \epsilon}  - \frac{\beta_1}{1-\beta_1} \bigg(\frac{1}{\sqrt{\hat\bv_t}+ \epsilon} - \frac{1}{\sqrt{\hat\bv_{t-1}}+ \epsilon} \bigg)\bbm_{t-1}\bigg\|^2\bigg]}_{I_3} \notag\\
& \quad + \underbrace{\EE\bigg[\bigg\langle (\nabla f(\bz_{t}) - \nabla f(\bx_{t})), \eta \frac{ \bDelta_t}{\sqrt{\hat\bv_t} + \epsilon} \bigg\rangle\bigg]}_{I_4}.
\end{align}

\textbf{Bounding $I_1$}
Denote a sequence $ \bar\bDelta_t = -\frac{\eta_l }{N} \sum_{i \in [N]} \sum_{k=0}^{K-1} \bg_{t-\tau_t^i, k}^i =  -\frac{\eta_l }{N} \sum_{i \in [N]} \sum_{k=0}^{K-1} \nabla F_i(\bx_{t-\tau_t^i, k}^i; \xi)$, where $\xi \sim \cD_i$. For $I_1$, there is 
\begin{align}\label{eq:I1-1}
    I_1 & = \eta \EE\bigg[\bigg\langle \nabla f(\bx_{t}), \frac{ \bDelta_t}{\sqrt{\hat\bv_t} + \epsilon} \bigg\rangle\bigg] \notag\\
    & = \eta \EE\bigg[\bigg\langle \nabla f(\bx_{t}), \frac{ \bar\bDelta_t}{\sqrt{\hat\bv_t} + \epsilon} \bigg\rangle\bigg] \notag\\
    & = \eta \EE\bigg[\bigg\langle \frac{\nabla f(\bx_{t})}{\sqrt{\hat\bv_t} + \epsilon} , \bar\bDelta_t + \eta_l K \nabla f(\bx_t) - \eta_l K \nabla f(\bx_t) \bigg\rangle\bigg] \notag\\
    & = -\eta  \eta_l K \EE \bigg[\bigg\|\frac{\nabla f(\bx_t)}{{(\sqrt{\hat\bv_t}+\epsilon)^{1/2}}} \bigg\|^2\bigg] + \eta  \EE\bigg[\bigg\langle \frac{\nabla f(\bx_{t})}{\sqrt{\hat\bv_t} + \epsilon} , \bar\bDelta_t + \eta_l K \nabla f(\bx_t) \bigg\rangle\bigg] \notag\\
    & = -\eta  \eta_l K \EE \bigg[\bigg\|\frac{\nabla f(\bx_t)}{{(\sqrt{\hat\bv_t}+\epsilon)^{1/2}}} \bigg\|^2\bigg] + \eta  \EE\bigg[\bigg\langle \frac{\nabla f(\bx_{t})}{\sqrt{\hat\bv_t} + \epsilon} , -\frac{\eta_l }{N} \sum_{i \in [N]} \sum_{k=0}^{K-1} \nabla F_i(\bx_{t-\tau_t^i, k}^i; \xi_i)+ \frac{\eta_l K}{N} \sum_{i \in [N]} \nabla F_i(\bx_t) \bigg\rangle\bigg], 
\end{align}
where the second equality holds due to the characteristic of uniform arrivals (see Assumption~\ref{as:bound-g-delay}), thus $\EE (\bDelta_t) = \bar\bDelta_t $. The last inequality holds by the definition of $\bar\bDelta_t$ and the fact of the objective function $f(\bx) = \frac{1}{N} \sum_{i=1}^N F_i(\bx)$. By the fact of $\langle \ba,\bb\rangle = \frac{1}{2}[\|\ba\|^2 + \|\bb\|^2 - \|\ba-\bb\|^2]$, for second term in \eqref{eq:I1-1}, we have 
\begin{align}\label{eq:I1-2}
    & \eta  \EE \bigg[\bigg\langle \frac{\nabla f(\bx_{t})}{\sqrt{\hat\bv_t} + \epsilon} , -\frac{\eta_l }{N} \sum_{i \in [N]} \sum_{k=0}^{K-1} \bg_{t-\tau_t^i, k}^i+ \frac{\eta_l K}{N} \sum_{i \in [N]} \nabla F_i(\bx_t) \bigg\rangle\bigg] \notag\\
    & = \eta \EE \bigg[ \bigg\langle \frac{\sqrt{\eta_l K}}{(\sqrt{\hat\bv_t}+\epsilon)^{1/2}} \nabla f(\bx_{t}), -\frac{\sqrt{\eta_l K}}{(\sqrt{\hat\bv_t}+\epsilon)^{1/2}} \frac{1}{NK} \sum_{i \in [N]} \sum_{k=0}^{K-1} (\bg_{t-\tau_t^i, k}^i - \nabla F_i(\bx_t))\bigg\rangle \bigg]\notag\\
    & = \eta \EE \bigg[ \bigg\langle \frac{\sqrt{\eta_l K}}{(\sqrt{\hat\bv_t}+\epsilon)^{1/2}} \nabla f(\bx_{t}), -\frac{\sqrt{\eta_l K}}{(\sqrt{\hat\bv_t}+\epsilon)^{1/2}} \frac{1}{NK} \sum_{i \in [N]} \sum_{k=0}^{K-1} (\nabla F_i (\bx_{t-\tau_t^i, k}^i) - \nabla F_i(\bx_t))\bigg\rangle \bigg]\notag\\
    & = \frac{\eta  \eta_l K}{2} \EE  \bigg[\bigg\|\frac{\nabla f(\bx_t)}{{(\sqrt{\hat\bv_t}+\epsilon)^{1/2}}} \bigg\|^2\bigg] + \frac{\eta  \eta_l }{2N^2 K} \EE \bigg[\bigg\|\frac{1}{{(\sqrt{\hat\bv_t}+\epsilon)^{1/2}}} \sum_{i \in [N]} \sum_{k=0}^{K-1} (\nabla F_i (\bx_{t-\tau_t^i, k}^i) - \nabla F_i(\bx_t))\bigg\|^2 \bigg] \notag\\
    & \quad - \frac{\eta  \eta_l }{2N^2 K} \EE \bigg[\bigg\|\frac{1}{{(\sqrt{\hat\bv_t}+\epsilon)^{1/2}}} \sum_{i \in [N]} \sum_{k=0}^{K-1} \nabla F_i (\bx_{t-\tau_t^i, k}^i) \bigg\|^2 \bigg],
\end{align}
where the second equality holds by $\EE[\bg_{t-\tau_t^i, k}^i] = \EE[\nabla F_i(\bx_{t-\tau_t^i, k}^i)]$. Then for the second term in Eq. \eqref{eq:I1-2} , we have 
\begin{align} \label{eq:I1_3}
    & \frac{\eta  \eta_l }{2N^2 K} \EE \bigg[\bigg\|\frac{1}{{(\sqrt{\hat\bv_t}+\epsilon)^{1/2}}} \sum_{i \in [N]} \sum_{k=0}^{K-1} (\nabla F_i (\bx_{t-\tau_t^i, k}^i) - \nabla F_i(\bx_t))\bigg\|^2 \bigg] \notag\\
    & \leq \frac{\eta  \eta_l }{2N^2 K \epsilon} \EE \bigg[\bigg\|\sum_{i \in [N]} \sum_{k=0}^{K-1} (\nabla F_i (\bx_{t-\tau_t^i, k}^i) - \nabla F_i(\bx_t))\bigg\|^2 \bigg] \notag\\
    & \leq  \frac{\eta  \eta_l }{2N \epsilon} \sum_{i \in [N]} \sum_{k=0}^{K-1} \EE [\|\nabla F_i(\bx_t) - \nabla F_i(\bx_{t-\tau_t^i,k}^i) \|^2 ] \notag\\
    & \leq \frac{\eta  \eta_l}{N \epsilon} \sum_{i \in [N]} \sum_{k=0}^{K-1} \bigg[ \EE [\|\nabla F_i(\bx_t) - \nabla F_i(\bx_{t-\tau_t^i})\|^2] + \EE [ \|\nabla F_i(\bx_{t-\tau_t^i}) - \nabla F_i(\bx_{t-\tau_t^i,k}^i) \|^2 ] \bigg]\notag\\
    & \leq \frac{\eta  \eta_l}{N \epsilon} \sum_{i \in [N]} \sum_{k=0}^{K-1} \bigg[L^2\EE [\|\bx_t - \bx_{t-\tau_t^i}\|^2] + L^2\EE [ \|\bx_{t-\tau_t^i} - \bx_{t-\tau_t^i,k}^i \|^2 ] \bigg],
\end{align}
where the second inequality holds by $\forall \ba_i, \|\sum_{i=1}^{n} \ba_i\|^2 \leq n \sum_{i=1}^{n}\|\ba_i\|^2 $, and the last inequality holds by Assumption \ref{as:smooth}. For the second term in Eq. \eqref{eq:I1_3}, following by Lemma \ref{lm:xikt-xt}, there is 
\begin{align}\label{eq:I1_4}
    \EE [ \|\bx_{t-\tau_t^i} - \bx_{t-\tau_t^i,k}^i \|^2 ] 
    & = \EE \bigg[\bigg\| \sum_{m=0}^{k-1} \eta_l \bg_{t-\tau_t^i,m}^i \bigg\|^2 \bigg]\notag\\
    & \leq 5K \eta_l^2 (\sigma^2 + 6K \sigma_g^2) + 30K^2\eta_l^2 \EE [\|\nabla f(\bx_{t-\tau_t^i}) \|^2].
\end{align}
For the first term in Eq. \eqref{eq:I1_3}, 
since by $\forall \ba_i, \|\sum_{i=1}^{n} \ba_i\|^2 \leq n \sum_{i=1}^{n}\|\ba_i\|^2 $, there is 
\begin{align}
    & \EE [\|\bx_t - \bx_{t-\tau_t^i}\|^2] = \EE  \bigg[\bigg\| \sum_{s = t - \tau_t^i}^{t-1} (\bx_{s+1} - \bx_s) \bigg\|^2\bigg] \leq \tau_t^i \sum_{s = t - \tau_t^i}^{t-1} \EE [\|\bx_{s+1} - \bx_s\|^2] \leq \tau_t^i \sum_{s = t - \tau_t^i}^{t-1} \EE\bigg[\bigg\| \eta \frac{\bbm_s}{\sqrt{\hat\bv_s} + \epsilon}\bigg\|^2\bigg],
\end{align}
then by decomposing stochastic noise, 
\begin{align} \label{eq:x-x}
    & \EE[\|\bx_t - \bx_{t-\tau_t^i}\|^2] \notag\\
    & \leq \frac{\eta^2 \tau_t^i }{\epsilon^2} \sum_{s = t - \tau_t^i}^{t-1}  \EE[\EE_s\|\bbm_s\|^2 ]]\notag\\
    & = \frac{\eta^2 \tau_t^i }{\epsilon^2} \sum_{s = t - \tau_t^i}^{t-1}  \EE\bigg[\EE_s\bigg[\bigg\|(1-\beta_1) \sum_{u=1}^s \beta_1^{s-u} \frac{1}{M} \sum_{j \in \cM_u} \sum_{k=0}^{K-1} \eta_l [\bg_{u-\tau_u^j,k}^j - \nabla F_j(\bx_{u-\tau_u^j,k}^j) + \nabla F_j(\bx_{u-\tau_u^j,k}^j)]\bigg\|^2 \bigg]\notag\\
    & \leq \frac{2\eta^2 \tau_t^i }{\epsilon^2} \sum_{s = t - \tau_t^i}^{t-1} (1-\beta_1)\sum_{u=1}^{s} \beta_{1}^{s-u} \frac{K \eta_l^2 }{M}\sigma^2 + \frac{2\eta^2 \tau_t^i }{\epsilon^2} \sum_{s = t - \tau_t^i}^{t-1} (1-\beta_1) \sum_{u=1}^{s} \beta_{1}^{s-u} \frac{\eta_l^2}{M^2} \EE\bigg[ \bigg\| \sum_{j\in \cM_u} \sum_{k=0}^{K-1}\nabla F_j (\bx_{u-\tau_u^j,k}^j)\bigg\|^2\bigg]\notag\\
    & \leq \frac{2\eta^2 (\tau_t^i)^2 }{\epsilon^2} \frac{K \eta_l^2}{M}\sigma^2 + \frac{2\eta^2 \tau_t^i }{\epsilon^2} \sum_{s = t - \tau_t^i}^{t-1} (1-\beta_1) \sum_{u=1}^{s} \beta_{1}^{s-u} \frac{\eta_l^2}{M^2} \EE\bigg[\bigg\| \sum_{j\in \cM_u} \sum_{k=0}^{K-1}\nabla F_j (\bx_{u-\tau_u^j,k}^j)\bigg\|^2\bigg],
\end{align}
where the first inequality holds by decomposing the momentum $\bbm_s$, i.e., $\bbm_s = (1-\beta_1) \sum_{u=1}^{s} \beta_{1}^{s-u} \bDelta_u = (1-\beta_1) \sum_{u=1}^{s} \beta_{1}^{s-u} \frac{1}{M} \sum_{j \in \cM_u} \sum_{k=0}^{K-1} \eta_l \bg_{u-\tau_u^j,k}^j $. The second inequality holds by $\|\ba + \bb\|^2 \leq 2 \|\ba\|^2 + 2\|\bb\|^2$ and the fact of $\EE[\cdot]] = \EE[\cdot]$, and the third inequality holds by $ (1-\beta_1) \sum_{u=1}^{s} \beta_{1}^{s-u} \leq 1$.

Following Lemma \ref{lm:gmv-bound}, $\frac{1}{C_G} \|\bx\|\leq \big\|\frac{\bx}{\sqrt{\hat\bv_t}+\epsilon} \big\| \leq \frac{1}{\epsilon} \|\bx\|$ and $C_G = \eta_l KG + \epsilon$, plugging Eq. \eqref{eq:I1-2}, Eq. \eqref{eq:I1_3} and Eq. \eqref{eq:I1_4} to \eqref{eq:I1-1}, we have 
\begin{align}
    \EE[I_1] \leq & - \frac{\eta  \eta_l K}{2 C_G} \EE[\|\nabla f(\bx_t)\|^2] - \frac{\eta  \eta_l}{2K C_G} \EE\bigg[\bigg\|\frac{1}{N} \sum_{i=1}^N \sum_{k=0}^{K-1} \nabla F_i(\bx_{t-\tau_t^i,k}^i) \bigg\|^2\bigg] \notag\\
    & + \frac{\eta  \eta_l K L^2 }{\epsilon} \bigg[5K \eta_l^2 (\sigma^2 + 6K \sigma_g^2) + 30K^2\eta_l^2 \frac{1}{N} \sum_{i=1}^N \EE[\|\nabla f(\bx_{t-\tau_t^i}) \|^2]\bigg]
    + \frac{2\eta_l^3 \eta^3 K^2 L^2  }{M \epsilon^3}\sigma^2 \frac{1}{N} \sum_{i=1}^N (\tau_t^i)^2 \notag\\
    & + \frac{2\eta_l^3 \eta^3 K L^2 }{M^2 \epsilon^3}\frac{1}{N} \sum_{i=1}^N \tau_t^i\sum_{s = t - \tau_t^i}^{t-1} (1-\beta_1) \sum_{u=1}^{s} \beta_{1}^{s-u} \EE\bigg[ \bigg\| \sum_{j\in \cM_u} \sum_{k=0}^{K-1}\nabla F_j (\bx_{u-\tau_u^j,k}^j)\bigg\|^2\bigg].
\end{align}

\textbf{Bounding $I_2$}
\begin{align}
    I_2 & = - \EE\bigg[\bigg\langle \nabla f(\bz_{t}), \frac{\beta_1}{1-\beta_1} \cdot \eta \bigg(\frac{1}{\sqrt{\hat\bv_t}+ \epsilon} - \frac{1}{\sqrt{\hat\bv_{t-1}}+ \epsilon} \bigg)\bbm_{t-1} \bigg\rangle\bigg] \notag\\
    & = -\eta \EE\bigg[\bigg\langle \nabla f(\bz_{t}) - \nabla f(\bx_{t}) + \nabla f(\bx_{t}), \frac{\beta_1}{1-\beta_1} \bigg(\frac{1}{\sqrt{\hat\bv_t}+ \epsilon} - \frac{1}{\sqrt{\hat\bv_{t-1}}+ \epsilon} \bigg)\bbm_{t-1} \bigg\rangle\bigg] \notag\\
    & \leq \eta \EE \bigg[\|\nabla f(\bx_t)\| \bigg\|\frac{\beta_1}{1-\beta_1} \bigg(\frac{1}{\sqrt{\hat\bv_t}+ \epsilon} - \frac{1}{\sqrt{\hat\bv_{t-1}}+ \epsilon} \bigg)\bbm_{t-1}\bigg\|\bigg] \notag\\
    & \quad + \eta^2 L \EE\bigg[\bigg\|\frac{\beta_1}{1-\beta_1} \frac{\bbm_{t-1}}{\sqrt{\hat\bv_{t-1}}+ \epsilon} \bigg\| \cdot \bigg\|\frac{\beta_1}{1-\beta_1} \bigg(\frac{1}{\sqrt{\hat\bv_t}+ \epsilon} - \frac{1}{\sqrt{\hat\bv_{t-1}}+ \epsilon} \bigg)\bbm_{t-1}\bigg\|\bigg] \notag\\
    & \leq \frac{\beta_1}{1-\beta_1} \eta_l \eta K G^2 \EE\bigg[\bigg\|\frac{1}{\sqrt{\hat\bv_t}+ \epsilon} - \frac{1}{\sqrt{\hat\bv_{t-1}}+ \epsilon} \bigg\|_1\bigg] + \frac{\beta_1^2}{(1-\beta_1)^2 \epsilon}\eta_l^2 \eta^2 K^2 G^2 L\EE\bigg[\bigg\|\frac{1}{\sqrt{\hat\bv_t}+ \epsilon} - \frac{1}{\sqrt{\hat\bv_{t-1}}+ \epsilon} \bigg\|_1\bigg],
\end{align}
where the first inequality holds by $\langle \ba, \bb \rangle \leq \|\ba\| \|\bb\|$ and $L$-smoothness of $f$, i.e., $\|\nabla f(\bz_t) - \nabla f(\bx_t)\| \leq L \|\bz_t - \bx_t\| $, and by the definition of $\bz_t$, there is $\bz_t - \bx_t = \frac{\beta_1}{1-\beta_1} \frac{\bbm_{t-1}}{\sqrt{\hat\bv_{t-1}}+ \epsilon}$. The second inequality holds by Lemma~\ref{lm:gmv-bound}.

\textbf{Bounding $I_3$}
\begin{align}
    I_3 & = \frac{\eta^2 L}{2}\EE\bigg[\bigg\|\frac{ \bDelta_t}{\sqrt{\hat\bv_t} + \epsilon} - \frac{\beta_1}{1-\beta_1} \bigg(\frac{1}{\sqrt{\hat\bv_t}+ \epsilon} - \frac{1}{\sqrt{\hat\bv_{t-1}}+ \epsilon} \bigg)\bbm_{t-1}\bigg\|^2\bigg] \notag\\
    & \leq \eta^2 L \EE\bigg[\bigg\| \frac{ \bDelta_t}{\sqrt{\hat\bv_t} + \epsilon} \bigg\|^2 \bigg] + \eta^2 L \EE\bigg[\bigg\|\frac{\beta_1}{1-\beta_1}\bigg(\frac{1}{\sqrt{\hat\bv_t}+ \epsilon} - \frac{1}{\sqrt{\hat\bv_{t-1}}+ \epsilon} \bigg)\bbm_{t-1}\bigg\|^2\bigg] \notag\\
    & \leq \eta^2 L \EE\bigg[\bigg\| \frac{ \bDelta_t}{\sqrt{\hat\bv_t} + \epsilon} \bigg\|^2 \bigg] + \eta^2 L \frac{\beta_1^2}{(1-\beta_1)^2} \eta_l^2 K^2 G^2 \EE\bigg[\bigg\|\frac{1}{\sqrt{\hat\bv_t}+ \epsilon} - \frac{1}{\sqrt{\hat\bv_{t-1}}+ \epsilon} \bigg\|^2\bigg] \notag\\
    & \leq \frac{\eta^2 L}{\epsilon^2} \EE[\| \bDelta_t\|^2 ] + \eta^2 L \frac{\beta_1^2}{(1-\beta_1)^2} \eta_l^2 K^2 G^2 \EE\bigg[\bigg\|\frac{1}{\sqrt{\hat\bv_t}+ \epsilon} - \frac{1}{\sqrt{\hat\bv_{t-1}}+ \epsilon} \bigg\|^2\bigg],
\end{align}
where the first inequality follows by Cauchy-Schwarz inequality, i.e., $\forall \ba_i, \|\sum_{i=1}^{n} \ba_i\|^2 \leq n \sum_{i=1}^{n}\|\ba_i\|^2 $, and the second one holds by Lemma \ref{lm:gmv-bound}.

\textbf{Bounding $I_4$}
\begin{align}
    I_4 & = \EE\bigg[\bigg\langle (\nabla f(\bz_{t}) - \nabla f(\bx_{t})), \eta \frac{ \bDelta_t}{\sqrt{\hat\bv_t} + \epsilon} \bigg\rangle\bigg] \notag\\
    & \leq \EE\bigg[\|\nabla f(\bz_{t}) - \nabla f(\bx_{t})\| \bigg\|\eta \frac{ \bDelta_t}{\sqrt{\hat\bv_t} + \epsilon} \bigg\|\bigg] \notag\\
    & \leq L \EE\bigg[\|\bz_{t} - \bx_{t}\| \bigg\|\eta \frac{ \bDelta_t}{\sqrt{\hat\bv_t} + \epsilon} \bigg\|\bigg] \notag\\
    & \leq \frac{\eta^2 L}{2} \EE\bigg[\bigg\|\frac{\beta_1}{1-\beta_1}\frac{\bbm_t}{\sqrt{\hat\bv_t} + \epsilon} \bigg\|^2 \bigg] + \frac{\eta^2 L}{2} \EE\bigg[\bigg\|\frac{ \bDelta_t}{\sqrt{\hat\bv_t} + \epsilon} \bigg\|^2 \bigg] \notag\\
    & \leq \frac{\eta^2 L}{2\epsilon^2 } \frac{\beta_1^2}{(1-\beta_1)^2} \EE[\|\bbm_t\|^2] +  \frac{\eta^2 L}{2\epsilon^2 } \EE[\|\bDelta_t\|^2],
\end{align}
where the second inequality holds by Assumption \ref{as:smooth} (the $L$-smoothness of $f$), and the third inequality holds by the definition of $\bz_t$ and the inequality $\|\ba \| \|\bb \| \leq \frac{1}{2} \|\ba\|^2 + \frac{1}{2} \|\bb\|^2$.

\textbf{Merging pieces. }
Therefore, by merging pieces together, we have 
\begin{align}
    & \EE[f(\zb_{t+1}) - f(\zb_t)] = \EE[I_1 + I_2 + I_3 + I_4 ]\notag\\
    \leq & - \frac{\eta  \eta_l K}{2 C_G} \EE[\|\nabla f(\bx_t)\|^2] - \frac{\eta  \eta_l}{2K C_G} \EE\bigg[\bigg\|\frac{1}{N} \sum_{i=1}^N \sum_{k=0}^{K-1} \nabla F_i(\bx_{t-\tau_t^i,k}^i) \bigg\|^2\bigg] \notag\\
    & + \frac{\eta  \eta_l K L^2 }{\epsilon} \bigg[5K \eta_l^2 (\sigma^2 + 6K \sigma_g^2) + 30K^2\eta_l^2 \frac{1}{N} \sum_{i=1}^N \EE[\|\nabla f(\bx_{t-\tau_t^i}) \|^2]\bigg] + \frac{2\eta^3 \eta_l^3 K^2 L^2 }{M \epsilon^3}\sigma^2 \frac{1}{N} \sum_{i=1}^N (\tau_t^i)^2 \notag\\
    & + \frac{2\eta^3 \eta_l^3 K L^2 }{M^2 \epsilon^3}\frac{1}{N} \sum_{i=1}^N\tau_t^i \sum_{s = t - \tau_t^i}^{t-1} (1-\beta_1) \sum_{u=1}^{s} \beta_{1}^{s-u} \EE\bigg[ \bigg\| \sum_{j\in \cM_u} \sum_{k=0}^{K-1}\nabla F_j (\bx_{u-\tau_u^j,k}^j)\bigg\|^2\bigg] \notag\\
    & + \frac{\beta_1}{1-\beta_1} \eta \eta_l K G^2 \EE\bigg[\bigg\|\frac{1}{\sqrt{\hat\bv_t}+ \epsilon} - \frac{1}{\sqrt{\hat\bv_{t-1}}+ \epsilon} \bigg\|_1\bigg] + \frac{\beta_1^2}{(1-\beta_1)^2 \epsilon}\eta^2 \eta_l^2 K^2 G^2 L\EE\bigg[\bigg\|\frac{1}{\sqrt{\hat\bv_t}+ \epsilon} - \frac{1}{\sqrt{\hat\bv_{t-1}}+ \epsilon} \bigg\|_1\bigg] \notag\\
    & + \frac{\eta^2 L}{\epsilon^2} \EE[\| \bDelta_t\|^2 ] + \frac{\beta_1^2}{(1-\beta_1)^2} \eta^2 \eta_l^2 K^2 G^2 L \EE\bigg[\bigg\|\frac{1}{\sqrt{\hat\bv_t}+ \epsilon} - \frac{1}{\sqrt{\hat\bv_{t-1}}+ \epsilon} \bigg\|^2\bigg] \notag\\
    & + \frac{\eta^2 L}{2\epsilon^2 } \frac{\beta_1^2}{(1-\beta_1)^2} \EE[\|\bbm_t\|^2] +  \frac{\eta^2 L}{2\epsilon^2 } \EE[\|\bDelta_t\|^2].
\end{align}
Denote a few sequences: $\bG_t = \sum_{j\in \cM_t} \sum_{k=0}^{K-1}\nabla F_j (\bx_{t-\tau_t^j,k}^j)$ and $\bV_t = \frac{1}{\sqrt{\hat\bv_t}+ \epsilon} - \frac{1}{\sqrt{\hat\bv_{t-1}}+ \epsilon}$, then re-write and organize the above inequality, we have 
\begin{align}
    & \EE[f(\bz_{t+1}) - f(\bz_t)] \notag\\
    \leq & - \frac{\eta \eta_l K}{2 C_G} \EE[\|\nabla f(\bx_t)\|^2] - \frac{\eta \eta_l}{2K C_G} \EE\bigg[\bigg\|\frac{1}{N} \sum_{i=1}^N \sum_{k=0}^{K-1} \nabla F_i(\bx_{t-\tau_t^i,k}^i) \bigg\|^2\bigg] \notag\\
    & + \frac{\eta \eta_l K L^2 }{\epsilon} \bigg[5K \eta_l^2 (\sigma^2 + 6K \sigma_g^2) + 30K^2\eta_l^2 \frac{1}{N} \sum_{i=1}^N \EE[\|\nabla f(\bx_{t-\tau_t^i}) \|^2]\bigg] + \frac{2\eta^3 \eta_l^3 K^2 L^2 }{M \epsilon^3}\sigma^2 \frac{1}{N} \sum_{i=1}^N (\tau_t^i)^2 \notag\\
    & + \frac{2\eta^3 \eta_l^3 K L^2 }{M^2 \epsilon^3}\frac{1}{N} \sum_{i=1}^N \tau_t^i \sum_{s = t - \tau_t^i}^{t-1} (1-\beta_1) \sum_{u=1}^{s} \beta_{1}^{s-u} \EE[\|\bG_u\|^2] \notag\\
    & + \frac{\beta_1}{1-\beta_1} \eta \eta_l K G^2 \EE[\|\bV_t\|_1] +  \frac{\beta_1^2}{(1-\beta_1)^2 \epsilon}\eta^2 \eta_l^2 K^2 G^2 L \EE[\|\bV_t\|_1] + \frac{\beta_1^2}{(1-\beta_1)^2} \eta^2 \eta_l^2 K^2 G^2 L \EE[\|\bV_t\|^2] \notag\\
    & + \frac{3\eta^2 L}{2\epsilon^2} \EE[\| \bDelta_t\|^2 ] + \frac{\eta^2 L}{2\epsilon^2 } \frac{\beta_1^2}{(1-\beta_1)^2} \EE[\|\bbm_t\|^2].
\end{align}

Summing over $t=1$ to $T$, we have 
\begin{align} \label{eq:sum_1}
    & \EE[f(\bz_{T+1}) - f(\bz_1)] \notag\\
    \leq & - \frac{\eta \eta_l K}{2C_G} \sum_{t=1}^T   \EE[\|\nabla f(\bx_t)\|^2] + \frac{\eta \eta_l K  L^2}{\epsilon} \bigg[5K \eta_l^2 (\sigma^2 + 6K \sigma_g^2) + 30K^2 \eta_l^2 \frac{1}{N}\sum_{t=1}^T \sum_{i=1}^N  \EE[\|\nabla f(\bx_{t-\tau_t^i}) \|^2]\bigg]\notag\\
    & + \underbrace{\frac{2\eta^3 \eta_l^3 K^2 L^2}{M \epsilon^3}\sigma^2 \frac{1}{N} \sum_{i=1}^N \sum_{t=1}^T  (\tau_t^i)^2}_{A_0} + \underbrace{\frac{2\eta^3 \eta_l^3 K L^2 }{M^2 \epsilon^3}\frac{1}{N} \sum_{i=1}^N \sum_{t=1}^T \tau_t^i\sum_{s = t - \tau_t^i}^{t-1} (1-\beta_1) \sum_{u=1}^s \beta_1^{s-u}\EE[\|\bG_u\|^2]}_{A_1} \notag\\
    & + \bigg(\frac{\beta_1}{1-\beta_1} \eta \eta_l K G^2 + \frac{\beta_1^2}{(1-\beta_1)^2 \epsilon} \eta^2 \eta_l^2 K^2 G^2 L\bigg)  \sum_{t=1}^T \EE[\|\bV_t\|_1]+ \frac{\beta_1^2}{(1-\beta_1)^2}\eta^2 \eta_l^2 K^2 G^2 L\sum_{t=1}^T \EE[\|\bV_t\|^2] \notag\\
    & + \frac{3\eta^2 L}{2\epsilon^2} \sum_{t=1}^T \EE[\| \bDelta_t\|^2 ]+ \frac{\eta^2 L}{2\epsilon^2 } \frac{\beta_1^2}{(1-\beta_1)^2} \sum_{t=1}^T \EE[\|\bbm_t\|^2] \notag\\
    & - \frac{\eta_l}{2K C_G} \sum_{t=1}^T  \EE\bigg[\bigg\|\frac{1}{N} \sum_{i=1}^N \sum_{k=0}^{K-1} \nabla F_i(\bx_{t-\tau_t^i,k}^i) \bigg\|^2\bigg],
\end{align}
we have the following for term $A_0$,
\begin{align}
    A_0 = \frac{2\eta^3 \eta_l^3 K^2 L^2}{M \epsilon^3}\sigma^2 \frac{1}{N} \sum_{i=1}^N \sum_{t=1}^T  (\tau_t^i)^2 \leq \frac{2\eta^3 \eta_l^3 K^2 L^2}{M \epsilon^3}\sigma^2 \tau_{\avg}\tau_{\max} T.
\end{align}
By Lemma \ref{lm:Delta_2}, we have the following for term $A_1$,
\begin{align}
    A_1 = & \frac{2\eta^3 \eta_l^3 K L^2 }{M^2 \epsilon^3}\frac{1}{N} \sum_{i=1}^N \sum_{t=1}^T \tau_t^i \sum_{s = t - \tau_t^i}^{t-1}  (1-\beta_1) \sum_{u=1}^s \beta_1^{s-u}\EE[\|\bG_u\|^2] \notag\\
    & = \frac{2\eta^3 \eta_l^3 K L^2 }{M^2 \epsilon^3}\frac{1}{N} \sum_{i=1}^N \sum_{t=1}^T \tau_t^i \sum_{s = t - \tau_t^i}^{t-1}  (1-\beta_1) \sum_{u=1}^s \beta_1^{s-u} \bigg\{\frac{3M(N-M)}{N-1} \notag\\
    & \cdot \bigg[5K^3 L^2 \eta_l^2(\sigma^2 + 6K\sigma_g^2) + (30K^4L^2\eta_l^2 + K^2) \frac{1}{N} \sum_{j=1}^N\EE[\|\nabla f(\bx_{u-\tau_u^j})\|^2] + K^2\sigma_g^2 \bigg]  \notag\\
    & + \frac{M(M-1)}{N(N-1)} \EE\bigg[\bigg\| \sum_{j=1}^N \sum_{k=0}^{K-1} \nabla F_j(\bx_{u-\tau_u^j,k}^j) \bigg\|^2\bigg]\bigg\} \notag\\
    & = \underbrace{\frac{2 \eta^3\eta_l^3 K L^2 }{M^2 \epsilon^3}\frac{1}{N} \sum_{i=1}^N \sum_{t=1}^T \tau_t^i \sum_{s = t - \tau_t^i}^{t-1} (1-\beta_1) \sum_{u=1}^s \beta_1^{s-u} \bigg[\frac{3M(N-M)}{N-1} [5K^3 L^2 \eta_l^2(\sigma^2 + 6K\sigma_g^2) + K^2 \sigma_g^2 ] \bigg]}_{A_2}\notag\\ 
    & + \underbrace{\frac{2\eta^3 \eta_l^3 K L^2}{M^2 \epsilon^3}\frac{1}{N} \sum_{i=1}^N \sum_{t=1}^T \tau_t^i \sum_{s = t - \tau_t^i}^{t-1} (1-\beta_1) \sum_{u=1}^s \beta_1^{s-u} \bigg[\frac{3M(N-M)}{N-1} (30K^4L^2\eta_l^2 + K^2) \frac{1}{N} \sum_{j=1}^N\EE[\|\nabla f(\bx_{u-\tau_u^j})\|^2] \bigg]}_{A_3} \notag\\
    & + \underbrace{\frac{2\eta^3 \eta_l^3 K L^2}{M^2 \epsilon^3}\frac{1}{N} \sum_{i=1}^N \sum_{t=1}^T \tau_t^i \sum_{s = t - \tau_t^i}^{t-1} (1-\beta_1) \sum_{u=1}^s \beta_1^{s-u} \bigg\{ \frac{M(M-1)}{N(N-1)} \EE\bigg[\bigg\| \sum_{j=1}^N \sum_{k=0}^{K-1} \nabla F_j(\bx_{u-\tau_u^j,k}^j) \bigg\|^2\bigg] \bigg\}}_{A_4}\notag\\
    & = A_2 + A_3 + A_4.
\end{align}
For term $A_2$, then re-organizing it we have 
\begin{align}
    A_2 & \leq \frac{2\eta^3 \eta_l^3 K L^2 }{M^2 \epsilon^3}\frac{1}{N} \sum_{i=1}^N \sum_{t=1}^T (\tau_t^i)^2 \bigg[\frac{3M(N-M)}{N-1} [5K^3 L^2 \eta_l^2(\sigma^2 + 6K\sigma_g^2) + K^2\sigma_g^2 ] \bigg] \notag\\
    & \leq \frac{\eta^3 \eta_l^3 K L^2 }{M^2 \epsilon^3}\bigg[\frac{6M(N-M)}{N-1} [5K^3 L^2 \eta_l^2(\sigma^2 + 6K\sigma_g^2) + K^2\sigma_g^2]\bigg] T \tau_{\avg} \tau_{\max}.
\end{align}
For term $A_3$, we have 
\begin{align}\label{eq:A3}
    A_3 & \leq \frac{2\eta^3 \eta_l^3 K L^2 }{M^2 \epsilon^3}\frac{1}{N} \sum_{i=1}^N \sum_{t=1}^T \tau_t^i \sum_{s = t - \tau_t^i}^{t-1} (1-\beta_1) \sum_{u=1}^s \beta_1^{s-u} \bigg\{\frac{3M(N-M)}{N-1} \cdot (30K^4L^2\eta_l^2 + K^2) \frac{1}{N} \sum_{j=1}^N\EE[\|\nabla f(\bx_{u-\tau_u^j})\|^2] \bigg\} \notag\\
    & \leq \frac{\eta^3 \eta_l^3 K L^2}{M^2 \epsilon^3}\tau_{\max}^2 \sum_{t=1}^T \bigg\{\frac{6M(N-M)}{N-1} (30K^4L^2\eta_l^2 + K^2) \frac{1}{N} \sum_{j=1}^N\EE[\|\nabla f(\bx_{t-\tau_t^j})\|^2] \bigg\} \notag\\
    & \leq \frac{\eta^3 \eta_l^3 K L^2}{M^2 \epsilon^3}\tau_{\max}^3 \frac{6M(N-M)}{N-1}(30K^4L^2\eta_l^2 + K^2) \sum_{t=1}^T \EE[\|\nabla f(\bx_t)\|^2],
\end{align}
where the first inequality in Eq. \eqref{eq:A3} holds due to: 1) $\tau_t^i \leq \tau_{max}$ and 2) for a positive sequence $\ba_t, \sum_{t=1}^T \sum_{s = t - \tau_t^i}^{t-1}(1-\beta_1) \sum_{u=1}^{s} \beta_{1}^{s-u} \ba_u \leq \tau_{\max} (1-\beta_1) \sum_{t=1}^T \sum_{u=1}^{t} \beta_{1}^{t-u} \ba_u \leq \tau_{\max} \sum_{t= 1}^{T} \ba_t$. 
In details, 
\begin{align}
    & \sum_{t=1}^T \sum_{s = t - \tau_t^i}^{t-1}(1-\beta_1) \sum_{u=1}^{s} \beta_{1}^{s-u} \ba_u \notag\\
    & = \sum_{t=1}^T \sum_{s = t - \tau_t^i}^{t-1}(1-\beta_1) (\beta_1^{s-1} \ba_1 + \beta_1^{s-2} \ba_2 + \cdots + \beta_1^0 \ba_s ) \notag\\
    & = \sum_{t=1}^T (1-\beta_1) \bigg[ \sum_{s = t - \tau_t^i}^{t-1} \beta_1^{s-1} \ba_1 + \sum_{s = t - \tau_t^i}^{t-1} \beta_1^{s-2} \ba_2 + \cdots + \sum_{s = t - \tau_t^i}^{t-1} \beta_1^0 \ba_s \bigg] \notag\\
    & \leq \tau_{\max} \sum_{t=1}^T (1-\beta_1) \sum_{u=1}^{t} \beta_{1}^{t-u} \ba_u \notag\\
    & \leq \tau_{\max} \sum_{t=1}^T \ba_t. 
\end{align}
The second inequality in Eq. \eqref{eq:A3} hold by the fact of $\sum_{t=1}^T \frac{1}{N} \sum_{j=1}^N \EE[\|\nabla f(\bx_{t-\tau_t^j})\|^2] \leq \tau_{\max} \sum_{t=1}^T \EE[\|\nabla f(\bx_t)\|^2]$.
Similar, for term $A_4$, we have 
\begin{align}
    A_4 & = \frac{2\eta^3 \eta_l^3 K L^2 }{M^2 \epsilon^3}\frac{1}{N} \sum_{i=1}^N \sum_{t=1}^T \tau_t^i \sum_{s = t - \tau_t^i}^{t-1} (1-\beta_1) \sum_{u=1}^s \beta_1^{s-u} \frac{M(M-1)}{N(N-1)} \EE\bigg[\bigg\| \sum_{j=1}^N \sum_{k=0}^{K-1} \nabla F_j(\bx_{t-\tau_t^j,k}^j) \bigg\|^2\bigg] \bigg\}\notag\\
    & \quad \leq \frac{\eta^3 \eta_l^3 K L^2 }{M^2 \epsilon^3} \tau_{\max}^2  \frac{2M(M-1)}{N(N-1)} \sum_{t=1}^T \EE\bigg[\bigg\|\sum_{j=1}^N \sum_{k=0}^{K-1} \nabla F_j(\bx_{t-\tau_t^j,k}^j) \bigg\|^2\bigg] \bigg\},
\end{align}
With the term of $A_0$ to $A_4$, by Lemma \ref{lm:Vt-difference} and Lemma \ref{lm:Delta}, we have the following for Eq. \eqref{eq:sum_1},
\begin{align} \label{eq:sumT}
    & \EE[f(\bz_{T+1}) - f(\bz_1)] \notag\\
    \leq & - \frac{\eta \eta_l K}{2C_G} \sum_{t=1}^T  \EE[\|\nabla f(\bx_t)\|^2] + \frac{\eta \eta_l K  L^2}{\epsilon} \bigg[5K \eta_l^2 T(\sigma^2 + 6K \sigma_g^2) + 30K^2\eta_l^2 \frac{1}{N}\sum_{t=1}^T \sum_{i=1}^N  \EE[\|\nabla f(\bx_{t-\tau_t^i}) \|^2]\bigg]\notag\\
    & + \frac{2\eta^3\eta_l^3 K^2L^2 \tau_{\avg} \tau_{\max} T }{M \epsilon^3}\sigma^2 + \frac{\eta^3 \eta_l^3 K L^2 }{M^2 \epsilon^3} \tau_{\max}^3 \frac{6M(N-M)}{N-1} (30K^4L^2\eta_l^2 + K^2) \sum_{t=1}^T \EE[\|\nabla f(\bx_t)\|^2] \notag\\
    & + \frac{\eta^3 \eta_l^3 K L^2 }{M^2 \epsilon^3}\cdot \frac{6M(N-M)}{N-1} [5K^3 L^2 \eta_l^2(\sigma^2 + 6K\sigma_g^2) + K^2\sigma_g^2 ] 
    \cdot T \tau_{\avg}\tau_{\max} \notag\\
    & + \bigg(\frac{\beta_1}{1-\beta_1} \eta \eta_l K G^2 + \frac{\beta_1^2}{(1-\beta_1)^2 \epsilon} \eta^2 \eta_l^2 K^2 G^2 L\bigg) \frac{d}{\epsilon} +  \frac{\beta_1^2}{(1-\beta_1)^2} \eta^2 \eta_l^2 K^2 G^2 L\frac{d}{\epsilon^2} \notag\\
    & + \bigg(\frac{3\eta^2 L}{2\epsilon^2}+ \frac{\eta^2 L}{2\epsilon^2 } \frac{\beta_1^2}{(1-\beta_1)^2} \bigg)\sum_{t=1}^T \bigg\{\frac{2\eta_l^2 K}{M} \sigma^2 + \frac{2\eta_l^2(N-M)}{NM(N-1)}\bigg[15NK^3 L^2 \eta_l^2(\sigma^2 + 6K\sigma_g^2) + (90K^4L^2\eta_l^2 +3K^2 )\notag\\
    & \cdot\sum_{i=1}^N \EE[\|\nabla f(\bx_{t-\tau_t^i})\|^2] + 3NK^2\sigma_g^2 \bigg]+ \frac{2\eta_l^2 (M-1)}{NM(N-1)} \EE\bigg[ \bigg\|\sum_{i=1}^N \sum_{k=0}^{K-1}\nabla F_i (\bx_{t-\tau_t^i,k}^i)\bigg\|^2\bigg]\bigg\} \notag\\
    & + \frac{2\eta^3 \eta_l^3 K L^2}{M^2 \epsilon^3}\tau_{\max}^2 \frac{M(M-1)}{N(N-1)} \sum_{t=1}^T \EE\bigg[\bigg\|\sum_{j=1}^N \sum_{k=0}^{K-1} \nabla F_j(\bx_{t-\tau_t^j,k}^j) \bigg\|^2\bigg] \bigg\}\notag\\
    & - \frac{\eta \eta_l}{2K N^2 C_G} \sum_{t=1}^T \EE\bigg[\bigg\| \sum_{i=1}^N \sum_{k=0}^{K-1} \nabla F_i(\bx_{t-\tau_t^i,k}^i) \bigg\|^2\bigg],
\end{align}
thus
\begin{align} \label{eq:sumT_2}
    & \EE[f(\bz_{T+1}) - f(\bz_1)] \notag\\
    \leq & - \frac{\eta \eta_l K}{2C_G} \sum_{t=1}^T  \EE[\|\nabla f(\bx_t)\|^2] + \frac{\eta \eta_l K  L^2}{\epsilon} \bigg[5K \eta_l^2 T(\sigma^2 + 6K \sigma_g^2) + 30K^2\eta_l^2 \tau_{\max}\sum_{t=1}^T  \EE[\|\nabla f(\bx_t) \|^2]\bigg]\notag\\
    & + \frac{2\eta^3\eta_l^3 K^2L^2 \tau_{\avg}\tau_{\max} T }{M \epsilon^3}\sigma^2 + \frac{\eta^3 \eta_l^3 K L^2 }{M^2 \epsilon^3} \tau_{\max}^3 \frac{6M(N-M)}{N-1} (30K^4L^2\eta_l^2 + K^2) \sum_{t=1}^T \EE[\|\nabla f(\bx_t)\|^2] \notag\\
    & + \frac{\eta^3 \eta_l^3 K L^2 }{M^2 \epsilon^3} \frac{6M(N-M)}{N-1}[5K^3 L^2 \eta_l^2(\sigma^2 + 6K\sigma_g^2) +  K^2\sigma_g^2]\cdot T \tau_{\avg}\tau_{\max} \notag\\
    & + \bigg(\frac{\beta_1}{1-\beta_1} \eta \eta_l K G^2 + \frac{\beta_1^2}{(1-\beta_1)^2 \epsilon} \eta^2 \eta_l^2 K^2 G^2 L\bigg) \frac{d}{\epsilon} + \frac{\beta_1^2}{(1-\beta_1)^2} \eta^2 \eta_l^2 K^2 G^2 L\frac{d}{\epsilon^2} \notag\\
    & + \bigg(\frac{3\eta^2 L}{\epsilon^2}+ \frac{\eta^2 L}{\epsilon^2 } \frac{\beta_1^2}{(1-\beta_1)^2} \bigg)\sum_{t=1}^T  \bigg\{\frac{K \eta_l^2}{M} \sigma^2 + \frac{\eta_l^2(N-M)}{NM(N-1)}\bigg[15NK^3 L^2 \eta_l^2(\sigma^2 + 6K\sigma_g^2) + 3NK^2\sigma_g^2 \bigg]\bigg\} \notag\\
    & + \bigg(\frac{3\eta^2 L}{\epsilon^2}+ \frac{\eta^2 L}{\epsilon^2 } \frac{\beta_1^2}{(1-\beta_1)^2} \bigg) \frac{\eta_l^2(N-M)}{M(N-1)} (90K^4L^2\eta_l^2 +3K^2 ) N\tau_{\max} \sum_{t=1}^T  \EE[\|\nabla f(\bx_{t})\|^2] \notag\\
    & + \bigg[\bigg(\frac{3\eta^2 L}{\epsilon^2}+ \frac{\eta^2 L}{\epsilon^2 } \frac{\beta_1^2}{(1-\beta_1)^2} \bigg) \frac{\eta_l^2 (M-1)}{NM(N-1)} + \frac{2\eta^3 \eta_l^3 KL^2}{M^2 \epsilon^3} \tau_{\max}^2 \frac{M(M-1)}{N(N-1)} - \frac{\eta \eta_l}{2K N^2 C_G}\bigg] \notag\\
    & \quad \cdot \sum_{t=1}^T \EE\bigg[\bigg\| \sum_{i=1}^N \sum_{k=0}^{K-1} \nabla F_i(\bx_{t-\tau_t^i,k}^i) \bigg\|^2\bigg].
\end{align}
If the learning rates satisfy $\eta_l \leq \frac{1}{8KL}$ and
\begin{align}
    & \eta \eta_l \leq \frac{\epsilon^2 M(N-1)}{4 C_G N(M-1)KL}\bigg( 3+ \frac{\beta_1^2}{(1-\beta_1)^2}\bigg)^{-1}, \eta \eta_l \leq \frac{\sqrt{\epsilon^3 M(N-1)}}{\sqrt{8 C_G N(M-1)}} \frac{1}{L\tau_{\max}}, \notag\\
    & \eta_l \leq \frac{\sqrt{\epsilon}}{\sqrt{360 C_G \tau_{\max}}KL}, \eta \eta_l \leq \frac{\epsilon^2 M(N-1)}{60 C_G N(N-M) KL \tau_{\max} }\bigg( 3+ \frac{\beta_1^2}{(1-\beta_1)^2}\bigg)^{-1}, \notag\\
    & \eta \eta_l \leq \frac{\sqrt{\epsilon^3 M(N-1)} }{12\sqrt{C_G N(M-1) \tau_{\max}^3 } KL },
\end{align}
then we have
\begin{align}
    & \bigg(\frac{3\eta^2 L}{\epsilon^2}+ \frac{\eta^2 L}{\epsilon^2 } \frac{\beta_1^2}{(1-\beta_1)^2} \bigg) \frac{\eta_l^2 (M-1)}{NM(N-1)} + \frac{2\eta^3 \eta_l^3 KL^2}{M^2 \epsilon^3} \tau_{\max}^2 \frac{M(M-1)}{N(N-1)}  - \frac{\eta \eta_l}{2K N^2 C_G} \leq 0 \notag\\
    & \frac{\eta \eta_l K L^2}{\epsilon} 30 K^2 \eta_l^2 \tau_{\max} +  \bigg(\frac{3\eta^2 L}{\epsilon^2}+ \frac{\eta^2 L}{\epsilon^2 } \frac{\beta_1^2}{(1-\beta_1)^2} \bigg) \frac{\eta_l^2(N-M)}{M(N-1)} (90K^4L^2\eta_l^2 +3K^2 ) N\tau_{\max} \notag\\
    & \quad + \frac{\eta^3 \eta_l^3 K L^2 }{M^2 \epsilon^3} \tau_{\max}^3 \frac{6M(N-M)}{N-1} (30K^4L^2\eta_l^2 + K^2) \leq \frac{\eta \eta_l K}{4C_G}.
\end{align}
Thus Eq. \eqref{eq:sumT_2}  becomes
\begin{align}\label{eq:sumT_3}
    \frac{\sum_{t=1}^T  \EE[\|\nabla f(\bx_t)\|^2]}{T} & \leq \frac{4C_G}{\eta \eta_l K T}[f(\bz_1) - \EE[f(\bz_{T+1})]] + 20 C_G \epsilon^{-1} L^2 K \eta_l^2 (\sigma^2 + 6K \sigma_g^2) + \frac{8C_G \eta^2 \eta_l^2 K L^2 \tau_{\avg}\tau_{\max} }{M \epsilon^3}  \sigma^2 \notag\\
    & \quad  + \frac{24C_G \eta^2 \eta_l^2 L^2 \tau_{\avg}\tau_{\max}}{M \epsilon^3} \frac{N-M}{N-1}\cdot[5K^3 L^2 \eta_l^2(\sigma^2 + 6K\sigma_g^2) + K^2 \sigma_g^2] \notag\\
    & \quad + \bigg(\frac{\beta_1}{1-\beta_1} G^2 + \frac{\beta_1^2}{(1-\beta_1)^2 \epsilon}\eta \eta_l K G^2 L \bigg) \frac{4C_Gd}{T \epsilon} +  \frac{\beta_1^2}{(1-\beta_1)^2} \eta \eta _l K G^2 L\frac{4C_Gd}{T \epsilon^2} \notag\\
    & \quad + 4C_G\bigg(\frac{3\eta L}{\epsilon^2}+ \frac{\eta L}{\epsilon^2 } \frac{\beta_1^2}{(1-\beta_1)^2} \bigg) \bigg\{ \frac{\eta_l}{M} \sigma^2 + \frac{\eta_l(N-M)}{M(N-1)}[15K^2 L^2 \eta_l^2(\sigma^2 + 6K\sigma_g^2) + 3K\sigma_g^2] \bigg\}.
\end{align}
With $C_G = \eta_l K G + \epsilon$, Eq. \eqref{eq:sumT_3} becomes 
\begin{align}\label{eq:sumT_5}
    & \frac{\sum_{t=1}^T  \EE[\|\nabla f(\bx_t)\|^2]}{T} \notag\\
    & \leq \frac{4(\eta_l K G + \epsilon)}{\eta \eta_l K T}[f(\bz_1) - \EE[f(\bz_{t+1})]] + 20 (\eta_l K G + \epsilon) \epsilon^{-1} L^2 K \eta_l^2 (\sigma^2 + 6K \sigma_g^2) \notag\\
    & \quad + \frac{8(\eta_l K G + \epsilon) \eta^2 \eta_l^2 K L^2 \tau_{\avg}\tau_{\max}}{M \epsilon^3}\sigma^2 + \frac{24(\eta_l K G + \epsilon) \eta^2 \eta_l^2 L^2 \tau_{\avg}\tau_{\max}}{M \epsilon^3} \frac{N-M}{N-1} \notag\\
    & \quad\cdot[5K^3 L^2 \eta_l^2(\sigma^2 + 6K\sigma_g^2) + K^2 \sigma_g^2] \notag\\
    & \quad + \bigg(\frac{\beta_1}{1-\beta_1} G^2 + \frac{\beta_1^2}{(1-\beta_1)^2 \epsilon}\eta \eta_l K G^2 L \bigg) \frac{4(\eta_l K G + \epsilon)d}{T \epsilon} +  \frac{\beta_1^2}{(1-\beta_1)^2} \eta \eta _l K G^2 L\frac{4(\eta_l K G + \epsilon)d}{T \epsilon^2} \notag\\
    & \quad + 4(\eta_l K G + \epsilon)\bigg(\frac{3\eta L}{\epsilon^2}+ \frac{\eta L}{\epsilon^2 } \frac{\beta_1^2}{(1-\beta_1)^2} \bigg) \bigg\{ \frac{\eta_l}{M} \sigma^2 + \frac{\eta_l(N-M)}{M(N-1)}[15K^2 L^2 \eta_l^2(\sigma^2 + 6K\sigma_g^2) + 3K\sigma_g^2] \bigg\}.
\end{align}
For $\beta_ 1 = 0$, with the definition of $\cF = f(\bx_1) - \min_{\bx} f(\bx)$, we have the following bound
\begin{align}\label{eq:sumT_6}
    & \frac{\sum_{t=1}^T  \EE[\|\nabla f(\bx_t)\|^2]}{T} \notag\\
    & \leq \frac{4(\eta_l K G + \epsilon)}{\eta \eta_l K T}[f(\bz_1) - \EE[f(\bz_{t+1})]] + 20 (\eta_l K G + \epsilon) \epsilon^{-1} L^2 K \eta_l^2 (\sigma^2 + 6K \sigma_g^2) \notag\\
    & \quad + \frac{8(\eta_l K G + \epsilon) \eta^2 \eta_l^2 K L^2 \tau_{\avg}\tau_{\max}}{M \epsilon^3}\sigma^2 + \frac{24(\eta_l K G + \epsilon) \eta^2 \eta_l^2 K L^2 \tau_{\avg}\tau_{\max}}{M \epsilon^3} \frac{N-M}{N-1} \notag\\
    & \quad\cdot[5\eta_l^2 K^2 L^2 (\sigma^2 + 6K\sigma_g^2) + K \sigma_g^2] \notag\\
    & \quad + 12(\eta_l K G + \epsilon)\frac{\eta L}{\epsilon^2} \bigg\{ \frac{\eta_l}{M} \sigma^2 + \frac{\eta_l(N-M)}{M(N-1)}[15\eta_l^2 K^2 L^2 (\sigma^2 + 6K\sigma_g^2) + 3K\sigma_g^2] \bigg\}.
\end{align}
This concludes the proof.
\end{proof}

\begin{proof}[Proof of Corollary \ref{cor:fadas}]
From Eq. \eqref{eq:sumT_6}, we have the following bound
\begin{align}\label{eq:sumT_5.1}
    & \frac{\sum_{t=1}^T  \EE[\|\nabla f(\bx_t)\|^2]}{T} \notag\\
    & = \cO\bigg( \frac{(\eta_l K G + \epsilon)}{\eta \eta_l K T} \cF + (\eta_l K G + \epsilon) \frac{\eta_l^2 K L^2 (\sigma^2 + K \sigma_g^2)}{\epsilon} \notag\\
    & \quad + \frac{(\eta_l K G + \epsilon) \eta^2 \eta_l^2 K L^2 \tau_{\avg}\tau_{\max}}{M \epsilon^3}\sigma^2 + \frac{(\eta_l K G + \epsilon) \eta^2 \eta_l^2 L^2 \tau_{\avg}\tau_{\max}}{M \epsilon^3} \frac{N-M}{N-1} [K^3 L^2 \eta_l^2(\sigma^2 + K\sigma_g^2) + K^2 \sigma_g^2] \notag\\
    & \quad + (\eta_l K G + \epsilon) \frac{\eta \eta_l L}{M \epsilon^2} \bigg[\sigma^2 + \frac{N-M}{N-1} K\sigma_g^2 \bigg] + (\eta_l K G + \epsilon) \frac{\eta \eta_l K L}{M \epsilon^2} \frac{N-M}{N-1}[\eta_l^2 K L^2 (\sigma^2 + K\sigma_g^2)] \bigg).
\end{align}
Reorganizing Eq. \eqref{eq:sumT_5.1}, particularly merging the stochastic variance and the global variance, we get 
\begin{align}\label{eq:sumT_5.2}
    & \frac{\sum_{t=1}^T  \EE[\|\nabla f(\bx_t)\|^2]}{T} \notag\\
    & = \cO\bigg(\frac{(\eta_l K G + \epsilon)}{\eta \eta_l K T} \cF + (\eta_l K G + \epsilon) \frac{\eta_l^2 K L^2 }{\epsilon} (\sigma^2 + K \sigma_g^2)\notag\\
    & \quad + (\eta_l K G + \epsilon) \frac{ \eta^2 \eta_l^2 K L^2}{M \epsilon^3} \bigg(\tau_{\avg}\tau_{\max}\sigma^2 + \frac{N-M}{N-1} \tau_{\avg}\tau_{\max} K\sigma_g^2 \bigg)\notag\\
    & \quad + \frac{(\eta_l K G + \epsilon) \eta^2 \eta_l^4 K^3 L^4 \tau_{\avg}\tau_{\max}}{M \epsilon^3} \frac{N-M}{N-1} (\sigma^2 + K\sigma_g^2) \notag\\
    & \quad + (\eta_l K G + \epsilon) \frac{\eta \eta_l L}{M \epsilon^2} \bigg[\sigma^2 + \frac{N-M}{N-1} K\sigma_g^2 \bigg] + (\eta_l K G + \epsilon) \frac{\eta \eta_l^3 K^2 L^3}{M \epsilon^2} \frac{N-M}{N-1}(\sigma^2 + K\sigma_g^2)\bigg).
\end{align}
By choosing $\eta = \Theta (\sqrt{M})$ and $\eta_l = \Theta\big( \frac{\sqrt{\cF}}{\sqrt{TK(\sigma^2 + K\sigma_g^2) L}} \big)$, which implies $\eta \eta_l = \Theta \big( \frac{\sqrt{\cF M}}{\sqrt{TK(\sigma^2 + K\sigma_g^2) L}} \big)$, and $\eta_l K G = \Theta \big( \frac{\sqrt{\cF K} G}{\sqrt{T(\sigma^2 + K\sigma_g^2) L}} \big)$,
\begin{align}\label{eq:sumT_5.3}
    & \frac{\sum_{t=1}^T  \EE[\|\nabla f(\bx_t)\|^2]}{T} \notag\\
    & = \cO\bigg\{ \frac{\cF G}{T \sqrt{M}} + \frac{\epsilon \sqrt{\cF (\sigma^2 + K \sigma_g^2) L }}{\sqrt{TKM}}+ \bigg( \frac{\sqrt{\cF K} G}{\sqrt{T(\sigma^2 + K\sigma_g^2) L}} + \epsilon \bigg)  \frac{\cF L}{T \epsilon}\notag\\
    & \quad + \bigg( \frac{\sqrt{\cF K} G}{\sqrt{T(\sigma^2 + K\sigma_g^2) L}} + \epsilon \bigg) \bigg(\frac{\cF L \tau_{\avg}\tau_{\max}}{T \epsilon^3} + \frac{N-M}{N-1}\frac{\cF L \tau_{\avg}\tau_{\max}}{T \epsilon^3} \bigg) \notag\\
    & \quad + \bigg( \frac{\sqrt{\cF K} G}{\sqrt{T(\sigma^2 + K\sigma_g^2) L}} + \epsilon \bigg) \bigg(\frac{\sqrt{\cF L} \sigma}{\sqrt{TKM}} + \frac{N-M}{N-1} \frac{\sqrt{\cF L} \sigma_g}{\sqrt{TM}} \bigg) + \frac{C_1}{T^{3/2}} + \frac{C_2}{T^2} \bigg\}.
\end{align}
We again generalize terms with smaller $T$ dependency orders, then we have 
\begin{align}\label{eq:sumT_5.4}
    & \frac{\sum_{t=1}^T  \EE[\|\nabla f(\bx_t)\|^2]}{T} \notag\\
    & = \cO\bigg( \frac{\cF G}{T \sqrt{M}} + \frac{\epsilon \sqrt{\cF L} \sigma}{\sqrt{TKM}} + \frac{\epsilon \sqrt{\cF L} \sigma_g}{\sqrt{TM}} + \frac{\cF L}{T} + \frac{\cF L \tau_{\avg}\tau_{\max}}{T \epsilon^2} + \frac{N-M}{N-1}\frac{\cF L \tau_{\avg}\tau_{\max}}{T \epsilon^2}\notag\\
    & \quad + \frac{\cF G}{T\sqrt{M}} + \frac{\epsilon \sqrt{\cF L} \sigma}{\sqrt{TKM}} + \frac{N-M}{N-1} \frac{\cF G}{T\sqrt{M}} + \frac{N-M}{N-1} \frac{\epsilon \sqrt{\cF L} \sigma_g}{\sqrt{TM}}+ \frac{C_1}{T^{3/2}} + \frac{C_2}{T^2} \bigg) \notag\\
    & = \cO\bigg( \frac{\sqrt{\cF} \sigma}{\sqrt{TKM}} + \frac{\sqrt{\cF} \sigma_g}{\sqrt{TM}} + \frac{\cF }{T} + \frac{\cF G}{T \sqrt{M}} + \frac{\cF \tau_{\max} \tau_{\avg}}{T } \bigg).
\end{align}
This concludes the proof for Corollary \ref{cor:fadas}.
\end{proof}

\section{Convergence analysis for delay adaptive asynchronous FL}\label{sec:async-thm_2}
\begin{proof}[Proof of Theorem \ref{thm:fadas_da}]

For the proof of delay adaptive, for proof convenience, we conduct analysis under the case that $\beta_1 = 0$.
From Assumption \ref{as:smooth}, $f$ is $L$-smooth, then taking conditional expectation at time $t$ on the auxiliary sequence $\bx_t$, we have

\begin{align}\label{eq:f-Lsmooth-noncomp_2}
& \EE[f(\bx_{t+1})-f(\bx_{t})] \notag\\
& = \EE[f(\bx_{t+1})-f(\bx_{t})] \notag\\
& \leq \EE[\langle \nabla f(\bx_t), \bx_{t+1}-\bx_t \rangle] + \frac{L}{2} \EE[\|\bx_{t+1}-\bx_t\|^2] \notag\\
& = \underbrace{\EE\bigg[\bigg\langle \nabla f(\bx_{t}), \eta_t \frac{\bDelta_t}{\sqrt{\hat\bv_t} + \epsilon} \bigg\rangle\bigg]}_{I_1} + \underbrace{\frac{\eta_t^2 L}{2}\EE\bigg[\bigg\| \frac{\bDelta_t}{\sqrt{\hat\bv_t} + \epsilon}\bigg\|^2\bigg]}_{I_2}.
\end{align}

\textbf{Bounding $I_1$}
We have 
\begin{align}\label{eq:I1-1_2}
    I_1 & = \eta_t \EE\bigg[\bigg\langle \nabla f(\bx_{t}), \frac{\bDelta_t}{\sqrt{\hat\bv_t} + \epsilon} \bigg\rangle\bigg] \notag\\
    & = \eta_t \EE\bigg[\bigg\langle \nabla f(\bx_{t}), \frac{\bar\bDelta_t}{\sqrt{\hat\bv_t} + \epsilon} \bigg\rangle\bigg] \notag\\
    & = \eta_t \EE\bigg[\bigg\langle \frac{\nabla f(\bx_{t})}{\sqrt{\hat\bv_t} + \epsilon} , \bar\bDelta_t + \eta_l K \nabla f(\bx_t) - \eta_l K \nabla f(\bx_t) \bigg\rangle\bigg] \notag\\
    & = -\eta_t \eta_l K \EE \bigg[\bigg\|\frac{\nabla f(\bx_t)}{{(\sqrt{\hat\bv_t}+\epsilon)^{1/2}}} \bigg\|^2\bigg] + \eta_t \EE\bigg[\bigg\langle \frac{\nabla f(\bx_{t})}{\sqrt{\hat\bv_t} + \epsilon} , \bar\bDelta_t + \eta_l K \nabla f(\bx_t) \bigg\rangle\bigg] \notag\\
    & = -\eta_t \eta_l K \EE \bigg[\bigg\|\frac{\nabla f(\bx_t)}{{(\sqrt{\hat\bv_t}+\epsilon)^{1/2}}} \bigg\|^2\bigg] + \eta_t \EE\bigg[\bigg\langle \frac{\nabla f(\bx_{t})}{\sqrt{\hat\bv_t} + \epsilon} , -\frac{1}{N} \sum_{i \in [N]} \sum_{k=0}^{K-1} \eta_l \bg_{t-\tau_t^i, k}^i+ \frac{\eta_l K}{N} \sum_{i \in [N]} \nabla F_i(\bx_t) \bigg\rangle\bigg],
\end{align}
where $ \bar\bDelta_t = -\frac{1}{N} \sum_{i \in [N]} \sum_{k=0}^{K-1} \eta_l \bg_{t-\tau_t^i, k}^i$. For the inner product term in \eqref{eq:I1-1_2}, by the fact of $\langle \ba,\bb\rangle = \frac{1}{2}[\|\ba\|^2 + \|\bb\|^2 - \|\ba-\bb\|^2]$,  we have 
\begin{align}\label{eq:I1-1_3}
    & \eta_t \EE\bigg[\bigg\langle \frac{\nabla f(\bx_{t})}{\sqrt{\hat\bv_t} + \epsilon} , -\frac{1}{N} \sum_{i \in [N]} \sum_{k=0}^{K-1} \eta_l \bg_{t-\tau_t^i, k}^i+ \frac{\eta_l K}{N} \sum_{i \in [N]} \nabla F_i(\bx_t) \bigg\rangle\bigg] \notag\\
    & = \eta_t \EE\bigg[\bigg\langle \frac{\sqrt{\eta_l K}}{(\sqrt{\hat\bv_t}+\epsilon)^{1/2}} \nabla f(\bx_{t}), - \frac{\sqrt{\eta_l K}}{(\sqrt{\hat\bv_t}+\epsilon)^{1/2}} \frac{1}{NK} \sum_{i \in [N]} \sum_{k=0}^{K-1} (\bg_{t-\tau_t^i, k}^i - \nabla F_i(\bx_t))\bigg\rangle \bigg] \notag\\
    & = \eta_t \EE\bigg[\bigg\langle \frac{\sqrt{\eta_l K}}{(\sqrt{\hat\bv_t}+\epsilon)^{1/2}} \nabla f(\bx_{t}), - \frac{\sqrt{\eta_l K}}{(\sqrt{\hat\bv_t}+\epsilon)^{1/2}} \frac{1}{NK} \sum_{i \in [N]} \sum_{k=0}^{K-1} (\nabla F_i (\bx_{t-\tau_t^i, k}^i) - \nabla F_i(\bx_t))\bigg\rangle \bigg]\notag\\
    & = \frac{\eta_t \eta_l K}{2} \EE \bigg[\bigg\|\frac{\nabla f(\bx_t)}{{(\sqrt{\hat\bv_t}+\epsilon)^{1/2}}} \bigg\|^2\bigg] + \frac{\eta_t \eta_l }{2N^2 K} \EE\bigg[\bigg\|\frac{1}{{(\sqrt{\hat\bv_t}+\epsilon)^{1/2}}} \sum_{i \in [N]} \sum_{k=0}^{K-1} (\nabla F_i (\bx_{t-\tau_t^i, k}^i) - \nabla F_i(\bx_t))\bigg\|^2 \bigg] \notag\\
    & \quad - \frac{\eta_t \eta_l }{2N^2 K} \EE\bigg[\bigg\|\frac{1}{{(\sqrt{\hat\bv_t}+\epsilon)^{1/2}}} \sum_{i \in [N]} \sum_{k=0}^{K-1} \nabla F_i (\bx_{t-\tau_t^i, k}^i) \bigg\|^2 \bigg],
\end{align}
where second equation holds by $\EE[\bg_{t-\tau_t^i, k}^i] = \EE[\nabla F(\bx_{t-\tau_t^i, k}^i)]$, for the second term in Eq. \eqref{eq:I1-1_3}, we have 
\begin{align}\label{eq:dI1_3}
    & \frac{\eta_t \eta_l }{2N^2 K} \EE\bigg[\bigg\|\frac{1}{{(\sqrt{\hat\bv_t}+\epsilon)^{1/2}}} \sum_{i \in [N]} \sum_{k=0}^{K-1} (\nabla F_i (\bx_{t-\tau_t^i, k}^i) - \nabla F_i(\bx_t))\bigg\|^2 \bigg] \notag\\
    & \leq \frac{\eta_t \eta_l }{2N^2 K \epsilon} \EE\bigg[\bigg\|\sum_{i \in [N]} \sum_{k=0}^{K-1} (\nabla F_i (\bx_{t-\tau_t^i, k}^i) - \nabla F_i(\bx_t))\bigg\|^2 \bigg] \notag\\
    & \leq  \frac{\eta_t \eta_l }{2N \epsilon} \sum_{i \in [N]} \sum_{k=0}^{K-1} \EE[\|\nabla F_i(\bx_t) - \nabla F_i(\bx_{t-\tau_t^i,k}^i) \|^2 ] \notag\\
    & \leq \frac{\eta_t \eta_l}{N \epsilon} \sum_{i \in [N]} \sum_{k=0}^{K-1} \bigg[ \EE[\|\nabla F_i(\bx_t) - \nabla F_i(\bx_{t-\tau_t^i})\|^2] + \EE[ \|\nabla F_i(\bx_{t-\tau_t^i}) - \nabla F_i(\bx_{t-\tau_t^i,k}^i) \|^2 ] \bigg]\notag\\
    & \leq \frac{\eta_t \eta_l}{N \epsilon} \sum_{i \in [N]} \sum_{k=0}^{K-1} \bigg[L^2\EE[\|\bx_t - \bx_{t-\tau_t^i}\|^2] + L^2\EE[ \|\bx_{t-\tau_t^i} - \bx_{t-\tau_t^i,k}^i \|^2 ] \bigg].
\end{align}
where the first second inequality holds by $\forall \ba_i, \|\sum_{i=1}^{n} \ba_i\|^2 \leq n \sum_{i=1}^{n}\|\ba_i\|^2 $, and the last inequality holds by Assumption \ref{as:smooth}. For the second term in Eq. \eqref{eq:dI1_3}, following by Lemma \ref{lm:xikt-xt}, there is 
\begin{align}
    \EE[ \|\bx_{t-\tau_t^i} - \bx_{t-\tau_t^i,k}^i \|^2 ]
    & = \EE \bigg[\bigg\| \sum_{m=0}^{k-1} \eta_l \bg_{t-\tau_t^i,m}^i \bigg\|^2 \bigg] \notag\\
    & \leq 5K \eta_l^2 (\sigma^2 + 6K \sigma_g^2) + 30K^2\eta_l^2 \EE[\|\nabla f(\bx_{t-\tau_t^i}) \|^2].
\end{align}
For the first term in Eq. \eqref{eq:dI1_3}, we have
\begin{align}
    & \EE[\|\bx_t - \bx_{t-\tau_t^i}\|^2] = \EE \bigg[\bigg\|\sum_{s = t - \tau_t^i}^{t-1} (\bx_{s+1} - \bx_s) \bigg\|^2\bigg] \notag\\
    & = \EE \bigg[\bigg\|\sum_{s = t - \tau_t^i}^{t-1} \eta_s \frac{\bDelta_s}{\sqrt{\hat\bv_s}+\epsilon}  \bigg\|^2\bigg] \notag\\
    & \leq \frac{1}{\epsilon^2 }\EE\bigg[\bigg\|\sum_{s = t - \tau_t^i}^{t-1}\eta_s \bDelta_s \bigg\|^2 \bigg] \notag\\
    & = \frac{1}{\epsilon^2 }\EE\bigg[\bigg\|\sum_{s = t - \tau_t^i}^{t-1}\eta_s \frac{1}{M} \sum_{j \in \cM_s} \bDelta_s^j \bigg\|^2 \bigg] \notag\\
    & = \frac{1}{\epsilon^2 }\EE\bigg[\bigg\|\sum_{s = t - \tau_t^i}^{t-1}\eta_s \frac{1}{M} \sum_{j \in \cM_s} \sum_{k=0}^{K-1} \eta_l \bg_{s-\tau_s^j, k}^j \bigg\|^2 \bigg],
\end{align}
then by decomposing stochastic noise, we have 
\begin{align}\label{eq:dI1_5}
    & \EE[\|\bx_t - \bx_{t-\tau_t^i}\|^2] \notag\\
    & \leq \frac{1}{\epsilon^2 }\EE\bigg[\bigg\|\sum_{s = t - \tau_t^i}^{t-1}\eta_s \frac{1}{M} \sum_{j \in \cM_s} \sum_{k=0}^{K-1} \eta_l [\bg_{s-\tau_s^j,k}^j - \nabla F_j(\bx_{s-\tau_s^j,k}^j) + \nabla F_j(\bx_{s-\tau_s^j,k}^j)] \bigg\|^2 \bigg] \notag\\
    & \leq \frac{2}{\epsilon^2 }\EE\bigg[\bigg\|\sum_{s = t - \tau_t^i}^{t-1}\eta_s \frac{1}{M} \sum_{j \in \cM_s} \sum_{k=0}^{K-1} \eta_l [\bg_{s-\tau_s^j,k}^j - \nabla F_j(\bx_{s-\tau_s^j,k}^j)] \bigg\|^2 \bigg] \notag\\
    & \quad + \frac{2}{\epsilon^2 }\EE\bigg[\bigg\|\sum_{s = t - \tau_t^i}^{t-1}\eta_s \frac{1}{M} \sum_{j \in \cM_s} \sum_{k=0}^{K-1} \eta_l \nabla F_j(\bx_{s-\tau_s^j,k}^j) \bigg\|^2 \bigg] \notag\\
    & \leq \frac{2}{\epsilon^2 }\sum_{s = t - \tau_t^i}^{t-1} \eta_s^2 \frac{K\eta_l^2}{M} \sigma^2 + \frac{2 \tau_t^i }{\epsilon^2} \sum_{s = t - \tau_t^i}^{t-1} \eta_s^2 \frac{\eta_l^2}{M^2} \EE\bigg[ \bigg\|\sum_{j\in \cM_s} \sum_{k=0}^{K-1}\nabla F_j (\bx_{s-\tau_s^j,k}^j)\bigg\|^2\bigg] \notag\\
    & \leq \frac{2}{\epsilon^2 }\tau_t^i \eta^2 \frac{K\eta_l^2}{M} \sigma^2 + \frac{2\tau_t^i }{\epsilon^2} \sum_{s = t - \tau_t^i}^{t-1} \eta_s^2 \frac{\eta_l^2}{M^2} \EE\bigg[ \bigg\|\sum_{j\in \cM_s} \sum_{k=0}^{K-1}\nabla F_j (\bx_{s-\tau_s^j,k}^j)\bigg\|^2\bigg],
\end{align}
where the  second inequality holds by $\|\ba+\bb\|^2 \leq 2\|\ba\|^2 + 2\|\bb\|^2$. 
The second inequality holds by Assumption \ref{as:bounded-v}, i.e., the zero-mean and the independency of stochastic noise.
The last inequality in Eq. \eqref{eq:dI1_5} holds due to the following: with adaptive learning rates 
\begin{align}
\eta_t =
\begin{cases}
\eta & \text{if } \tau_t^{\max} \leq \tau_c, \\
\min\{\eta, \frac{1}{\tau_t^{\max}}\} & \text{if } \tau_t^{\max} > \tau_c, 
\end{cases}
\end{align}
thus we have $\eta_s \leq \eta$ in the last inequality in Eq. \eqref{eq:dI1_5}. 
Then for $I_1$, following Lemma \ref{lm:gmv-bound} $\frac{1}{C_G} \|\bx\|\leq \big\|\frac{\bx}{\sqrt{\hat\bv_t}+\epsilon} \big\| \leq \frac{1}{\epsilon} \|\bx\|$ and $C_G = \eta_l KG + \epsilon$, we have 
\begin{align}
    I_1 & \leq - \frac{\eta_t \eta_l K}{2 C_G} \EE[\|\nabla f(\bx_t)\|^2] - \frac{\eta_t \eta_l}{2K C_G} \EE\bigg[\bigg\|\frac{1}{N} \sum_{i=1}^N \sum_{k=0}^{K-1} \nabla F_i(\bx_{t-\tau_t^i,k}^i) \bigg\|^2\bigg] \notag\\
    & \quad + \frac{\eta_t \eta_l K L^2 }{\epsilon} \bigg[5K \eta_l^2 (\sigma^2 + 6K \sigma_g^2) + 30K^2\eta_l^2 \frac{1}{N} \sum_{i=1}^N \EE[\|\nabla f(\bx_{t-\tau_t^i}) \|^2]\bigg]
    + \frac{2 \eta_t \eta^2 \eta_l^3 K^2 L^2 }{M \epsilon^3} \frac{1}{N} \sum_{i=1}^N \tau_t^i \sigma^2 \notag\\
    & \quad + \frac{2 \eta_t \eta_l^3 K L^2 }{M^2\epsilon^3}\frac{1}{N} \sum_{i=1}^N \tau_t^i \sum_{s = t - \tau_t^i}^{t-1} \eta_s^2 \EE\bigg[ \bigg\|\sum_{j\in \cM_s} \sum_{k=0}^{K-1}\nabla F_j (\bx_{s-\tau_s^j,k}^j)\bigg\|^2\bigg]. \notag\\
\end{align}

\textbf{Bounding $I_2$}
\begin{align}
    I_2 & = \frac{\eta_t^2 L}{2}\EE\bigg[\bigg\| \frac{\bDelta_t}{\sqrt{\hat\bv_t} + \epsilon}\bigg\|^2\bigg] \leq \frac{\eta_t^2 L}{2\epsilon^2} \EE[\| \bDelta_t\|^2 ],
\end{align}
where the first inequality follows by Cauchy-Schwarz inequality.

\textbf{Merging pieces. }
Therefore, by merging pieces together, we have 
\begin{align}
    & \EE[f(\bx_{t+1}) - f(\bx_t)] = I_1 + I_2  \notag\\
    \leq & - \frac{\eta_t \eta_l K}{2 C_G} \EE[\|\nabla f(\bx_t)\|^2] - \frac{\eta_t \eta_l}{2K C_G} \EE\bigg[\bigg\|\frac{1}{N} \sum_{i=1}^N \sum_{k=0}^{K-1} \nabla F_i(\bx_{t-\tau_t^i,k}^i) \bigg\|^2\bigg] \notag\\
    & + \frac{\eta_t \eta_l K L^2 }{\epsilon} \bigg[5K \eta_l^2 (\sigma^2 + 6K \sigma_g^2) + 30K^2\eta_l^2 \frac{1}{N} \sum_{i=1}^N \EE[\|\nabla f(\bx_{t-\tau_t^i}) \|^2]\bigg] + \frac{2 \eta_t \eta^2 \eta_l^3 K^2 L^2 }{M \epsilon^3} \frac{1}{N} \sum_{i=1}^N \tau_t^i \sigma^2 \notag\\
    & + \frac{2 \eta_t \eta_l^3 K L^2 }{M^2 \epsilon^3}\frac{1}{N} \sum_{i=1}^N \tau_t^i \sum_{s = t - \tau_t^i}^{t-1} \eta_s^2 \EE\bigg[ \bigg\|\sum_{j\in \cM_s} \sum_{k=0}^{K-1}\nabla F_j (\bx_{s-\tau_s^j,k}^j)\bigg\|^2\bigg] + \frac{\eta_t^2 L}{2\epsilon^2} \EE[\| \bDelta_t\|^2 ].
\end{align}
Denote a sequences $\bG_s = \sum_{j\in \cM_s} \sum_{k=0}^{K-1}\nabla F_j (\bx_{t-\tau_s^j,k}^j)$, then re-write and organize the above inequality, we have 
\begin{align}
    & \EE[f(\bx_{t+1}) - f(\bx_t)] \notag\\
    \leq & - \frac{\eta_t \eta_l K}{2 C_G} \EE[\|\nabla f(\bx_t)\|^2] - \frac{\eta_t \eta_l}{2K C_G} \EE\bigg[\bigg\|\frac{1}{N} \sum_{i=1}^N \sum_{k=0}^{K-1} \nabla F_i(\bx_{t-\tau_t^i,k}^i) \bigg\|^2\bigg] \notag\\
    & + \frac{\eta_t \eta_l K L^2 }{\epsilon} \bigg[5K \eta_l^2 (\sigma^2 + 6K \sigma_g^2) + 30K^2\eta_l^2 \frac{1}{N} \sum_{i=1}^N \EE[\|\nabla f(\bx_{t-\tau_t^i}) \|^2]\bigg] + \frac{2 \eta_t \eta^2 \eta_l^3 K^2 L^2 }{M \epsilon^3} \frac{1}{N} \sum_{i=1}^N \tau_t^i \sigma^2 \notag\\
    & + \frac{2\eta_t \eta_l^3 K L^2}{M^2 \epsilon^3}\frac{1}{N} \sum_{i=1}^N \tau_t^i \sum_{s = t - \tau_t^i}^{t-1} \eta_s^2 \EE[\|\bG_s\|^2] + \frac{\eta_t^2 L}{2\epsilon^2} \EE[\| \bDelta_t\|^2 ].
\end{align}

Summing over $t=1$ to $T$, we have 
\begin{align} \label{eq:sumT_d}
    & \EE[f(\bx_{T+1}) - f(\bx_1)] \notag\\
    \leq & - \frac{\eta_l K}{2C_G} \sum_{t=1}^T \eta_t \EE[\|\nabla f(\bx_t)\|^2] + \frac{\eta_l K  L^2}{\epsilon} \bigg[5K \eta_l^2 (\sigma^2 + 6K \sigma_g^2) \sum_{t=1}^T \eta_t + 30K^2\eta_l^2 \frac{1}{N}\sum_{t=1}^T \sum_{i=1}^N \eta_t \EE[\|\nabla f(\bx_{t-\tau_t^i}) \|^2]\bigg]\notag\\
    & + \frac{2\eta^2 \eta_l^3 K^2 L^2}{M \epsilon^3}\sigma^2 \sum_{t=1}^T \frac{1}{N} \sum_{i=1}^N \tau_t^i \eta_t + \frac{2\eta_l^3 K L^2 }{M^2 \epsilon^3}\frac{1}{N} \sum_{i=1}^N \sum_{t=1}^T \eta_t \tau_t^i \sum_{s = t - \tau_t^i}^{t-1} \eta_s^2 \EE[\|\bG_s\|^2] + \frac{L}{2\epsilon^2} \sum_{t=1}^T \eta_t^2 \EE[\| \bDelta_t\|^2 ] \notag\\
    & - \frac{\eta_l}{2K C_G} \sum_{t=1}^T \eta_t \EE\bigg[\bigg\|\frac{1}{N} \sum_{i=1}^N \sum_{k=0}^{K-1} \nabla F_i(\bx_{t-\tau_t^i,k}^i) \bigg\|^2\bigg] \notag\\
    \leq & - \frac{\eta_l K}{2C_G} \sum_{t=1}^T \eta_t \EE[\|\nabla f(\bx_t)\|^2] + \frac{\eta_l K  L^2}{\epsilon} \bigg[5K \eta_l^2 (\sigma^2 + 6K \sigma_g^2) \sum_{t=1}^T \eta_t + 30K^2\eta_l^2 \frac{1}{N}\sum_{t=1}^T \sum_{i=1}^N \eta_t \EE[\|\nabla f(\bx_{t-\tau_t^i}) \|^2]\bigg]\notag\\
    & + \frac{2\eta^3 \eta_l^3 K^2 L^2}{M \epsilon^3}\sigma^2 T \tau_{\avg} + \underbrace{\frac{2 \eta_l^3 K L^2 }{M^2\epsilon^3}\frac{1}{N} \sum_{i=1}^N \sum_{t=1}^T \eta_t \tau_t^i \sum_{s = t - \tau_t^i}^{t-1} \eta_s^2 \EE[\|\bG_s\|^2]}_{A_1} + \frac{L}{2\epsilon^2} \sum_{t=1}^T \eta_t^2 \EE[\| \bDelta_t\|^2 ] \notag\\
    & - \frac{\eta_l}{2K C_G} \sum_{t=1}^T \eta_t \EE\bigg[\bigg\|\frac{1}{N} \sum_{i=1}^N \sum_{k=0}^{K-1} \nabla F_i(\bx_{t-\tau_t^i,k}^i) \bigg\|^2\bigg],
\end{align}
where the second inequality holds by $\eta_t \leq \eta$. We have the following for term $A_1$,
\begin{align}
    & \frac{2\eta_l^3 K L^2 }{M^2 \epsilon^3}\frac{1}{N} \sum_{i=1}^N \sum_{t=1}^T \sum_{s = t - \tau_t^i}^{t-1} \eta_s^2 \EE[\|\bG_s\|^2] \notag\\
    & = \frac{2\eta_l^3 K L^2 }{M^2 \epsilon^3}\frac{1}{N} \sum_{i=1}^N \sum_{t=1}^T  \eta_t \tau_t^i \sum_{s = t - \tau_t^i}^{t-1} \eta_s^2 \cdot \frac{3M(N-M)}{N-1} \notag\\
    & \quad \cdot \bigg[5K^3 L^2 \eta_l^2(\sigma^2 + 6K\sigma_g^2) + (30K^4L^2\eta_l^2 + K^2) \EE[\|\nabla f(\bx_{s-\tau_s^j})\|^2] + K^2 \sigma_g^2 \bigg] \notag\\
    & \quad + \frac{2\eta_l^3 K L^2 }{M^2 \epsilon^3}\frac{1}{N} \sum_{i=1}^N \sum_{t=1}^T \eta_t \tau_t^i  \sum_{s = t - \tau_t^i}^{t-1} \eta_s^2 \cdot \frac{M(M-1)}{n(N-1)} \EE\bigg[\bigg\|\sum_{j=1}^N \sum_{k=0}^{K-1} \nabla F_j(\bx_{s-\tau_s^j,k}^j) \bigg\|^2\bigg] \notag\\
    & = \underbrace{\frac{\eta_l^3 K L ^2}{M^2 \epsilon^3}\frac{1}{N} \sum_{i=1}^N \sum_{t=1}^T \eta_t \tau_t^i \sum_{s = t - \tau_t^i}^{t-1} \eta_s^2\cdot \frac{6M(N-M)}{N-1} \bigg[5K^3 L^2 \eta_l^2(\sigma^2 + 6K\sigma_g^2) + K^2 \sigma_g^2 \bigg]}_{A_2} \notag\\ 
    & \quad + \underbrace{\frac{\eta_l^3 K L}{M^2 \epsilon^3}\frac{1}{N} \sum_{i=1}^N \sum_{t=1}^T \eta_t \tau_t^i \sum_{s = t - \tau_t^i}^{t-1} \eta_s^2\cdot \frac{6M(N-M)}{N-1} (30K^4L^2\eta_l^2 + K^2) \EE[\|\nabla f(\bx_{s-\tau_s^j})\|^2]}_{A_3} \notag\\
    & \quad + \underbrace{\frac{2\eta_l^3 K L^2 }{M^2 \epsilon^3}\frac{1}{N} \sum_{i=1}^N \sum_{t=1}^T \eta_t \tau_t^i  \sum_{s = t - \tau_t^i}^{t-1} \eta_s^2 \cdot \frac{M(M-1)}{N(N-1)} \EE\bigg[\bigg\|\sum_{j=1}^N \sum_{k=0}^{K-1} \nabla F_j(\bx_{s-\tau_s^j,k}^j) \bigg\|^2\bigg]}_{A_4} \notag\\
    & =A_2 + A_3 + A_4.
\end{align}
For term $A_2$, 
note that with $\eta \leq \frac{\sqrt{M}}{\tau_c}$, we have the adaptive learning rates
\begin{align}
\eta_t =
\begin{cases}
\eta & \text{if } \tau_t^{\max} \leq \tau_c, \\
\min\{\eta, \frac{1}{\tau_t^{\max}}\} & \text{if } \tau_t^{\max} > \tau_c,
\end{cases}
\end{align}
which implies that $\eta_t \leq \eta$ and $\eta_t \leq \min\{\frac{1}{\tau_t^{\max}}, \frac{\sqrt{M}}{\tau_c}\}$. Moreover, recall that $\tau_t^{\max} = \max_{i \in [N]} \{\tau_t^i\}$, for each $i$, we have $\eta_t \tau_t^i \leq \frac{\sqrt{M} \cdot \tau_t^i}{\max(\tau_t^{\max}, \tau_c)} \leq \sqrt{M}$.
by the fact of $\eta_t \tau_t^i \leq \sqrt{M}$, $\frac{1}{N}\sum_{i=1}^N \sum_{t=1}^T\tau_t^i \leq T \tau_{\avg}$ and $\eta_s \leq \eta$, we have 
\begin{align}
    A_2 & \leq \frac{\eta_l^3 K L ^2}{M^2 \epsilon^3}\frac{1}{N} \sum_{i=1}^N \sum_{t=1}^T \eta_t \tau_t^i \cdot \tau_t^i \eta^2\cdot \frac{6M(N-M)}{N-1} \bigg[5K^3 L^2 \eta_l^2(\sigma^2 + 6K\sigma_g^2) + K^2 \sigma_g^2 \bigg] \notag\\
    & \leq \frac{\eta_l^3 K L ^2}{M \epsilon^3}\frac{1}{N} \sum_{i=1}^N \sum_{t=1}^T \sqrt{M} \cdot \tau_t^i \eta^2\cdot \frac{N-M}{N-1} \bigg[5K^3 L^2 \eta_l^2(\sigma^2 + 6K\sigma_g^2) + K^2 \sigma_g^2 \bigg] \notag\\
    & \leq \frac{\eta^2 \eta_l^3 K L ^2}{\sqrt{M} \epsilon^3} \frac{N-M}{N-1} \bigg[5K^3 L^2 \eta_l^2(\sigma^2 + 6K\sigma_g^2) + K^2 \sigma_g^2 \bigg] T  \tau_{\avg}.
\end{align}
For term $A_3$, since $\eta_s \leq \eta$, and consider $\tau_t^i \leq \tau_{\max}$ and $\sum_{t=1}^T \sum_{s = t - \tau_t^i}^{t-1} \ba_s \leq \tau_{\max} \sum_{t=1}^T \ba_t$, we have 
\begin{align}
    A_3 & \leq \frac{\eta_l^3 K L ^2}{M^2 \epsilon^3}\frac{1}{N} \sum_{i=1}^N \sum_{t=1}^T \eta_t \tau_t^i \sum_{s = t - \tau_t^i}^{t-1} \eta_s^2 \cdot \frac{6M(N-M)}{N-1} (30K^4L^2\eta_l^2 + K^2) \EE[\|\nabla f(\bx_{s-\tau_s^j})\|^2] \notag\\
    & \leq \frac{\eta^2 \eta_l^3 K L}{M^2 \epsilon^3} \frac{6M(N-M)}{N-1}(30K^4L^2\eta_l^2 + K^2) \tau_{\max}^3 \sum_{t=1}^T \eta_t \EE[\|\nabla f(\bx_t)\|^2] .
\end{align}
For term $A_4$, similar to the proof of non-delay adaptive FADAS, by $\eta_s \leq \eta$, $\tau_t^i \leq \tau_{\max}$ and $\sum_{t=1}^T \sum_{s = t - \tau_t^i}^{t-1} \ba_s \leq \tau_{\max} \sum_{t=1}^T \ba_t$, there is 
\begin{align}
    A_4 & = \frac{2\eta_l^3 K L^2 }{M^2 \epsilon^3}\frac{1}{N} \sum_{i=1}^N \sum_{t=1}^T \eta_t \tau_t^i \sum_{s = t - \tau_t^i}^{t-1} \eta_s^2 \cdot \frac{M(M-1)}{N(N-1)} \EE\bigg[\bigg\|\sum_{j=1}^N \sum_{k=0}^{K-1} \nabla F_j(\bx_{s-\tau_s^j,k}^j) \bigg\|^2\bigg] \notag\\
    & \quad \leq \frac{2\eta^2 \eta_l^3 K L^2 }{M^2 \epsilon^3}\frac{M(M-1)}{N(N-1)}  \tau_{\max} \sum_{t=1}^T \eta_t \sum_{s = t - \tau_t^i}^{t-1} \EE\bigg[\bigg\|\sum_{j=1}^N \sum_{k=0}^{K-1} \nabla F_j(\bx_{s-\tau_s^j,k}^j) \bigg\|^2\bigg] \bigg\} \notag\\
    & \quad \leq \frac{2\eta^2 \eta_l^3 K L^2 }{M^2 \epsilon^3} \frac{M(M-1)}{N(N-1)} \tau_{\max}^2 \sum_{t=1}^T \eta_t \EE\bigg[\bigg\|\sum_{j=1}^N \sum_{k=0}^{K-1} \nabla F_j(\bx_{t-\tau_t^j,k}^j) \bigg\|^2\bigg] \bigg\}.
\end{align}

By Lemma \ref{lm:Vt-difference} and Lemma \ref{lm:Delta}, we have the following for Eq. \eqref{eq:sumT_d},  
\begin{align} \label{eq:sumT_d_2}
    & \EE[f(\bx_{T+1}) - f(\bx_1)] \notag\\
    \leq & - \frac{\eta_l K}{2C_G} \sum_{t=1}^T \eta_t \EE[\|\nabla f(\bx_t)\|^2] + \frac{\eta_l K  L^2}{\epsilon} \bigg[5K \eta_l^2 (\sigma^2 + 6K \sigma_g^2) \sum_{t=1}^T \eta_t + 30K^2\eta_l^2 \frac{1}{N}\sum_{t=1}^T \sum_{i=1}^N \eta_t \EE[\|\nabla f(\bx_{t-\tau_t^i}) \|^2]\bigg]\notag\\
    & + \frac{2\eta^3 \eta_l^3 K^2 L^2}{M \epsilon^3} T\tau_{\avg} \sigma^2 + \frac{\eta^2 \eta_l^3 K L ^2}{\sqrt{M} \epsilon^3} \frac{N-M}{N-1} [5K^3 L^2 \eta_l^2(\sigma^2 + 6K\sigma_g^2) + K^2\sigma_g^2 ]T \tau_{\avg} \notag\\
    & + \frac{\eta^2 \eta_l^3 K L^2}{M^2 \epsilon^3}\frac{6M(N-M)}{N-1} (30K^4L^2\eta_l^2 + K^2) \tau_{\max}^3 \sum_{t=1}^T \eta_t\EE[\|\nabla f(\bx_t)\|^2]  \notag\\
    & + \frac{L}{2\epsilon^2} \sum_{t=1}^T \eta_t^2 \bigg\{\frac{2K \eta_l^2}{M} \sigma^2 + \frac{2\eta_l^2(N-M)}{NM(N-1)}\bigg[15NK^3 L^2 \eta_l^2(\sigma^2 + 6K\sigma_g^2) + 3NK^2\sigma_g^2 \bigg]\notag\\
    & + \frac{2\eta_l^2(N-M)}{M(N-1)} (90NK^4L^2\eta_l^2 +3NK^2 )\frac{1}{N} \sum_{i=1}^N \EE[\|\nabla f(\bx_{t-\tau_t^i})\|^2] \bigg\}\notag\\
    & + \bigg[\frac{\eta L}{\epsilon^2} \frac{\eta_l^2(M-1)}{NM(N-1)} + \frac{2\eta^2 \eta_l^3 K L^2}{M^2 \epsilon^3}  \frac{M(M-1)}{N(N-1)} \tau_{\max}^2- \frac{\eta_l}{2K N^2 C_G} \bigg] \sum_{t=1}^T \eta_t \EE\bigg[\bigg\| \sum_{i=1}^N \sum_{k=0}^{K-1} \nabla F_i(\bx_{t-\tau_t^i,k}^i) \bigg\|^2\bigg],
\end{align}
by the relationship of $\sum_{t=1}^{T} \frac{1}{N} \sum_{j=1}^N \EE[\|\nabla f(\bx_{t-\tau_t^i})\|^2] \leq \tau_{\max} \sum_{t=1}^{T} \EE[\|\nabla f(\bx_t)\|^2]$.

If the learning rates satisfy $\eta_l \leq \frac{1}{8KL}$ and
\begin{align}
    & \eta \eta_l \leq \frac{\epsilon^2 M(N-1)}{4 C_G N(M-1)KL}, \eta \eta_l \leq \frac{\sqrt{\epsilon^3 M(N-1)}}{\sqrt{8 C_G N(M-1)}} \frac{1}{L\tau_{\max}}, \notag\\
    & \eta_l \leq \frac{\sqrt{\epsilon}}{\sqrt{360 C_G \tau_{\max}}KL}, \eta \eta_l \leq \frac{\epsilon^2 M(N-1)}{60 C_G N(N-M) KL \tau_{\max} }, \notag\\
    & \eta \eta_l \leq \frac{\sqrt{\epsilon^3 M(N-1)} }{12\sqrt{C_G N(M-1) \tau_{\max}^3 } KL }.
\end{align}
Then we have
\begin{align}
    & \frac{\eta L}{\epsilon^2} \frac{\eta_l^2 (M-1)}{NM(N-1)} + \frac{2\eta^2 \eta_l^3 KL^2}{M^2 \epsilon^3} \frac{M(M-1)}{N(N-1)} \tau_{\max}^2 - \frac{\eta_l}{2K N^2 C_G} \leq 0 \notag\\
    & \frac{ \eta_l K L^2}{\epsilon} 30 K^2 \eta_l^2 \tau_{\max} + \frac{\eta L}{\epsilon^2} \frac{\eta_l^2(N-M)}{M(N-1)} (90K^4L^2\eta_l^2 +3K^2 ) N\tau_{\max} \notag\\
    & \quad + \frac{\eta^2 \eta_l^3 K L^2 }{M^2 \epsilon^3}\frac{6M(N-M)}{N-1} (30K^4L^2\eta_l^2 + K^2)  \tau_{\max}^3 \leq \frac{\eta_l K}{4C_G}.
\end{align}
Thus Eq. \eqref{eq:sumT_d_2} becomes
\begin{align} \label{eq:sumT_d_3}
    & \sum_{t=1}^T \eta_t \EE[\|\nabla f(\bx_t)\|^2] \notag\\
    \leq & \frac{4C_G}{\eta_l K} [f(\bx_1)-\EE[f(\bx_{T+1})]]  + \frac{4C_G L^2}{\epsilon} 5K \eta_l^2 (\sigma^2 + 6K \sigma_g^2) \sum_{t=1}^T \eta_t \notag\\
    & + \frac{8 C_G \eta^3 \eta_l^2 K L^2}{M \epsilon^3} T\tau_{\avg} \sigma^2 + \frac{24 C_G \eta^2 \eta_l^2 L ^2}{\sqrt{M} \epsilon^3} \frac{N-M}{N-1} [5K^3 L^2 \eta_l^2(\sigma^2 + 6K\sigma_g^2) + K^2\sigma_g^2 ]T \tau_{\avg} \notag\\
    & + \frac{4C_G \eta L}{\epsilon^2} \sum_{t=1}^T \eta_t \bigg[\frac{\eta_l}{M} \sigma^2 + \frac{\eta_l(N-M)}{NM(N-1)}[15NK^2 L^2 \eta_l^2(\sigma^2 + 6K\sigma_g^2) + 3NK\sigma_g^2] \bigg],
\end{align}
divided by the learning rates,
\begin{align} \label{eq:sumT_d_4}
    \frac{\sum_{t=1}^T \eta_t \EE[\|\nabla f(\bx_t)\|^2]}{\sum_{t=1}^T \eta_t} & \leq \frac{4C_G}{\eta_l K \sum_{t=1}^T \eta_t}[f(\bx_1) - \EE[f(\bx_{T+1})]] \notag\\
    & \quad + 20 C_G \epsilon^{-1} L^2 K \eta_l^2 (\sigma^2 + 6K \sigma_g^2) + \frac{8C_G \eta^3 \eta_l^2 K L^2}{M \epsilon^3} \frac{T\tau_{\avg}}{\sum_{t=1}^T \eta_t } \sigma^2 \notag\\
    & \quad + \frac{24C_G \eta^2 \eta_l^2 L^2}{\sqrt{M} \epsilon^3} \frac{N-M}{N-1} [5K^3 L^2 \eta_l^2(\sigma^2 + 6K\sigma_g^2) + K^2\sigma_g^2 ] \frac{T\tau_{\avg}}{\sum_{t=1}^T \eta_t } \notag\\
    & \quad + \frac{4C_G \eta L}{\epsilon^2} \bigg\{ \frac{\eta_l}{M} \sigma^2 + \frac{\eta_l(N-M)}{M(N-1)}[15K^2 T L^2 \eta_l^2(\sigma^2 + 6K\sigma_g^2) + 3K\sigma_g^2] \bigg\}.
\end{align}
This concludes the proof.
\end{proof}

\begin{proof}[Proof of Corollary \ref{cor:fadas_da}]
Since the delay adaptive learning rate satisfy, $\eta_t \leq \eta $, and when $\tau_c = \tau_\med$, there is $\sum_{t=1}^T \eta_t \geq \sum_{t: \tau_t \leq \tau_c} \eta \geq \frac{T \eta }{2}$ (since there are at least half of the iterations with the delay smaller than $\tau_c$). Recalling that $C_G = \eta_l KG + \epsilon$, then
\begin{align}\label{eq:sumT_d_5.1}
    & \frac{\sum_{t=1}^T \eta_t \EE[\|\nabla f(\bx_t)\|^2]}{\sum_{t=1}^T \eta_t } \notag\\
    & = \cO\bigg\{ \frac{(\eta_l K G + \epsilon)}{\eta \eta_l K T} \cF + (\eta_l K G + \epsilon) \frac{\eta_l^2 K L^2 (\sigma^2 + K \sigma_g^2)}{\epsilon} \notag\\
    & \quad + \frac{(\eta_l K G + \epsilon) \eta^2 \eta_l^2 K L^2 \tau_{\avg}}{M \epsilon^3}\sigma^2 + \frac{(\eta_l K G + \epsilon) \eta \eta_l^2 K L^2 \tau_{\avg}}{\sqrt{M} \epsilon^3} \frac{N-M}{N-1} [K^2 L^2 \eta_l^2(\sigma^2 + K\sigma_g^2)+ K \sigma_g^2] \notag\\
    & \quad + (\eta_l K G + \epsilon) \frac{\eta \eta_l L}{M \epsilon^2} \bigg[\sigma^2 + \frac{N-M}{N-1} K\sigma_g^2 \bigg] + (\eta_l K G + \epsilon) \frac{\eta \eta_l K L}{M \epsilon^2} \frac{N-M}{N-1}[\eta_l^2 K L^2 (\sigma^2 + K\sigma_g^2)] \bigg\}.
\end{align}
Reorganizing Eq. \eqref{eq:sumT_d_5.1}, particularly merging the stochastic variance and the global variance, then we have 
\begin{align}\label{eq:sumT_d_5.2}
    & \frac{\sum_{t=1}^T \eta_t \EE[\|\nabla f(\bx_t)\|^2]}{\sum_{t=1}^T \eta_t } \notag\\
    & = \cO\bigg\{\frac{(\eta_l K G + \epsilon)}{\eta \eta_l K T} \cF + (\eta_l K G + \epsilon) \frac{\eta_l^2 K L^2 }{\epsilon} (\sigma^2 + K \sigma_g^2)\notag\\
    & \quad + (\eta_l K G + \epsilon) \frac{ \eta \eta_l^2 K L^2}{M \epsilon^3} \bigg(\eta \tau_{\avg}\sigma^2 + \frac{\sqrt{M} (N-M)}{N-1} \tau_{\avg} K\sigma_g^2 \bigg) + \frac{(\eta_l K G + \epsilon) \eta \eta_l^4 K^3 L^4 \tau_{\avg}}{\sqrt{M} \epsilon^3} \frac{N-M}{N-1} (\sigma^2 + K\sigma_g^2) \notag\\
    & \quad + (\eta_l K G + \epsilon) \frac{\eta \eta_l L}{M \epsilon^2} \bigg[\sigma^2 + \frac{N-M}{N-1} K\sigma_g^2 \bigg] + (\eta_l K G + \epsilon) \frac{\eta \eta_l^3 K^2 L^3}{M \epsilon^2} \frac{N-M}{N-1}(\sigma^2 + K\sigma_g^2) \bigg\} .
\end{align}
By choosing $\eta = \sqrt{M}/\tau_c$ and $\eta_l = \min\big\{\frac{1}{KL},  \frac{\tau_c \sqrt{\cF}}{\sqrt{TK(\sigma^2 + K\sigma_g^2) L}} \big\}$, which implies $\eta \eta_l = \min\big\{\frac{\sqrt{M}}{\tau_c KL}, \frac{\sqrt{\cF M}}{\sqrt{TK(\sigma^2 + K\sigma_g^2) L}} \big\}$,

\begin{align}\label{eq:sumT_d_5.32}
    & \frac{\sum_{t=1}^T \eta_t \EE[\|\nabla f(\bx_t)\|^2]}{\sum_{t=1}^T \eta_t } \notag\\
    & = \cO\bigg\{ \bigg(\frac{\tau_c \sqrt{\cF K} G }{\sqrt{T(\sigma^2 + K\sigma_g^2) L}} +\epsilon\bigg) \frac{\sqrt{\cF (\sigma^2 + K \sigma_g^2) L}}{\sqrt{TKM}} + \bigg(\frac{\tau_c \sqrt{\cF K} G }{\sqrt{T(\sigma^2 + K\sigma_g^2) L}} +\epsilon\bigg)  \frac{\cF L\tau_c^2 }{T \epsilon}\notag\\
    & \quad + \bigg(\frac{\tau_c \sqrt{\cF K} G }{\sqrt{T(\sigma^2 + K\sigma_g^2) L}} +\epsilon\bigg)\bigg(\frac{\cF L \tau_{\avg}}{T \epsilon^3} +  \frac{N-M}{N-1}\frac{\cF L \tau_c\tau_{\avg}}{T \epsilon^3} \bigg) \notag\\
    & \quad + \bigg( \frac{\sqrt{\cF K} G}{\sqrt{T(\sigma^2 + K\sigma_g^2) L}} + \epsilon \bigg) \bigg(\frac{\sqrt{\cF L} \sigma}{\sqrt{TKM}} + \frac{N-M}{N-1} \frac{\sqrt{\cF L} \sigma_g}{\sqrt{TM}} \bigg) + \frac{C_1}{T^{3/2}} + \frac{C_2}{T^2} \bigg\} \notag\\
    & = \cO\bigg\{ \bigg(\frac{\tau_c \sqrt{\cF K} G }{\sqrt{T(\sigma^2 + K\sigma_g^2) L}} +\epsilon\bigg) \frac{\sqrt{\cF (\sigma^2 + K \sigma_g^2) L}}{\sqrt{TKM}} + \bigg(\frac{\tau_c \sqrt{\cF K} G }{\sqrt{T(\sigma^2 + K\sigma_g^2) L}} +\epsilon\bigg)  \frac{ \cF L \tau_c^2}{T \epsilon}\notag\\
    & \quad + \bigg(\frac{\tau_c \sqrt{\cF K} G }{\sqrt{T(\sigma^2 + K\sigma_g^2) L}} +\epsilon\bigg)\bigg(\frac{\cF L \tau_{\avg}}{T \epsilon^3} + \frac{\cF L \tau_c\tau_{\avg}}{T \epsilon^3} \bigg) + \bigg( \frac{\sqrt{\cF K} G}{\sqrt{T(\sigma^2 + K\sigma_g^2) L}} + \epsilon \bigg) \bigg(\frac{\sqrt{\cF L} \sigma}{\sqrt{TKM}} + \frac{\sqrt{\cF L} \sigma_g}{\sqrt{TM}} \bigg) \notag\\
    & \quad + \frac{C_1}{T^{3/2}} + \frac{C_2}{T^2} \bigg\}.
\end{align}
We again generalize terms with smaller $T$ dependency orders, then we have 
\begin{align}\label{eq:sumT_d_5.4}
    & \frac{\sum_{t=1}^T \eta_t \EE[\|\nabla f(\bx_t)\|^2]}{\sum_{t=1}^T \eta_t}\notag\\
    & \leq \cO\bigg(\frac{\tau_c \cF G}{T \sqrt{M}} + \frac{ \sqrt{\cF } \sigma}{\sqrt{TKM}} + \frac{ \sqrt{\cF } \sigma_g}{\sqrt{TM}} + \frac{\cF \tau_c^2 }{T} + \frac{\cF \tau_{\avg}}{T} + \frac{\cF \tau_c \tau_{\avg}}{T} \bigg),
\end{align}
reorganizing and then obtain the rate of convergence in Eq. \eqref{eq:fadas_da}.

\end{proof}


\section{Supporting Lemmas}

\begin{lemma}[Lemma for momentum term in the update rule]\label{lm:mt}
    The first order momentum terms $\bbm_t$ in Algorithm \ref{alg:fadas} hold the following relationship w.r.t. model difference $\bDelta_t$: 
    \begin{align}
        \sum_{t=1}^T \EE [\|\bbm_t\|^2] \leq \sum_{t=1}^T \EE [\|\bDelta_t\|^2].
    \end{align}
\end{lemma}
\begin{proof}
By the updating rule, we have
\begin{align}
    \EE[\|\bbm_t\|^{2}] 
    & = \EE\bigg[\bigg\|(1-\beta_1) \sum_{u=1}^{t} \beta_{1}^{t-u} \bDelta_{u}\bigg\|^{2}\bigg] \notag\\
    & \leq (1-\beta_1)^2 \sum_{i=1}^{d} \EE\bigg[\bigg(\sum_{u=1}^{t} \beta_{1}^{t-u} \bDelta_u^i \bigg)^2 \bigg] \notag\\
    & \leq (1-\beta_1)^{2} \sum_{i=1}^{d} \EE\bigg[\bigg(\sum_{u=1}^{t} \beta_{1}^{t-u}\bigg)\bigg(\sum_{u=1}^{t} \beta_{1}^{t-u} (\bDelta_u^i)^2\bigg)\bigg] \notag \\
    & \leq (1-\beta_1) \sum_{u=1}^{t} \beta_{1}^{t-u} \EE[\|\bDelta_{u}\|^{2}]. 
\end{align}
Summing over $ t=1,...,T$ yields
\begin{align}
    \sum_{t=1}^T \EE[\|\bbm_t\|^{2}] & = (1-\beta_1)\sum_{t=1}^T  \sum_{u=1}^{t} \beta_{1}^{t-u} \EE[\|\bDelta_{u}\|^{2}] \notag\\
    & = (1-\beta_1) \sum_{u=1}^{T} \sum_{t=u}^T \beta_{1}^{t-u} \EE[\|\bDelta_{u}\|^{2}] \notag\\
    & \leq (1-\beta_1) \sum_{u=1}^{T} \frac{1}{1-\beta_1} \EE[\|\bDelta_{u}\|^{2}] \notag\\
    & = \sum_{u=1}^{T} \EE[\|\bDelta_{u}\|^{2}].
\end{align}
This concludes the proof.
\end{proof}

\begin{lemma}\label{lm:gmv-bound}
Under Assumptions \ref{as:bounded-g}, we have $\|\nabla f(\bx)\|\leq G$, $\|\bDelta_t\|\leq \eta_l K G$, $\|\bbm_t\| \leq \eta_l K G$, $\|\bv_t\| \leq \eta_l^2 K^2 G^2$ and $\|\hat\bv_t\| \leq \eta_l^2 K^2 G^2$.
\end{lemma}
\begin{proof}
Since $f$ has $G$-bounded stochastic gradients, for any $\bx$ and $\xi$, there is $\|\nabla f(\bx, \xi)\| \leq G$, thus it implies 
\begin{align*}
    \|\nabla f(\bx)\| = \|\EE_\xi \nabla f(\bx, \xi)\| \leq \EE_\xi \|\nabla f(\bx, \xi)\|\leq G.
\end{align*}
For each model difference $\bDelta_t^i$ on client $i$,  $\bDelta_t^i$ satisfies,  
\begin{align*}
    \bDelta_t^i = \bx_{t,K}^i - \bx_t = -\eta_l \sum_{k=0}^{K-1} \bg_{t,k}^i,
\end{align*}
therefore,
\begin{align*}
    \big\|\bDelta_t^i\big\| = \bigg\|-\eta_l \sum_{k=0}^{K-1} \bg_{t,k}^i\bigg\| \leq \eta_l K G,
\end{align*}
for the global model difference $\bDelta_t$, 
\begin{align*}
    \|\bDelta_t\| = \bigg\|\frac{1}{M} \sum_{i \in \cM_t} \bDelta_t^i \bigg\| \leq \eta_l K G.
\end{align*}
Thus we can obtain the bound for momentum $\bbm_t$ and variance $\bv_t$,
\begin{align*}
    \|\bbm_t\| & = \bigg\|(1-\beta_1) \sum_{s=1}^t \beta_1^{t-s} \bDelta_s\bigg\| \leq \eta_l K G, \quad 
    \|\bv_t\| = \bigg\|(1-\beta_2) \sum_{s=1}^t \beta_2^{t-s} \bDelta_s^2 \bigg\| \leq \eta_l^2 K^2 G^2.
\end{align*}
By the updating rule of $\hat\bv_t$, there exists a $j \in [t]$ such that $\hat\bv_t = \bv_j$ . Then
\begin{align}
    \|\hat\bv_t\| \leq \eta_l^2 K^2 G^2.
\end{align}
This concludes the proof.
\end{proof}

\begin{lemma}\label{lm:Vt-difference}
    For the variance difference sequence $\bV_t = \frac{1}{\sqrt{\hat\bv_t}+ \epsilon} - \frac{1}{\sqrt{\hat\bv_{t-1}}+ \epsilon}$, we have 
\begin{align}\label{eq:Dt}
    \sum_{t=1}^{T} \|\bV_t\|_1 \leq \frac{d}{\epsilon}, \quad \sum_{t=1}^{T} \|\bV_t\|_2^2 \leq \frac{d}{\epsilon^2}.
\end{align}
\end{lemma}
\begin{proof}
The proof of Lemma \ref{lm:Vt-difference} is exactly the same as the proof of Lemma C.2 in \citet{wang2022communication}. 
\end{proof}

\begin{lemma} \label{lm:Delta}
Recall the sequence $ \bDelta_t = \frac{1}{M} \sum_{i \in\cM_t} \bDelta_{t-\tau_t^i}^i = -\frac{\eta_l }{M} \sum_{i \in \cM_t} \sum_{k=0}^{K-1} \bg_{t-\tau_t^i, k}^i = -\frac{\eta_l }{M} \sum_{i \in \cM_t} \sum_{k=0}^{K-1} \nabla F_i(\bx_{t-\tau_t^i, k}^i; \xi)$ and $\cM_t$ be the set that include client send the local updates to the server at global round $t$. The global model difference $\bDelta_t $ satisfies 
\begin{align*}
    \EE[\|\bDelta_t\|^2] & = \EE\bigg[ \bigg\|\frac{1}{M} \sum_{i \in \cM_t} \bDelta_{t-\tau_t^i}^i \bigg\|^2\bigg] \notag\\
    & \leq \frac{2K \eta_l^2}{M} \sigma^2 + \frac{2\eta_l^2(N-M)}{NM(N-1)}\bigg[15NK^3 L^2 \eta_l^2(\sigma^2 + 6K\sigma_g^2) + (90NK^4L^2\eta_l^2 +3K^2 )\notag\\
    &\quad \cdot\sum_{i=1}^N \EE[\|\nabla f(\bx_{t-\tau_t^i})\|^2] + 3NK^2\sigma_g^2 \bigg]+ \frac{2\eta_l^2 (M-1)}{NM(N-1)} \EE\bigg[ \bigg\|\sum_{i=1}^N \sum_{k=0}^{K-1}\nabla F_i (\bx_{t-\tau_t^i,k}^i)\bigg\|^2\bigg].
\end{align*} 
\end{lemma}
\begin{proof}
The proof of Lemma \ref{lm:Vt-difference} is similar to the proof of Lemma C.6 in \cite{wang2022communication}. 
\end{proof}

\begin{lemma} \label{lm:xikt-xt}
(This lemma follows from Lemma 3 in FedAdam \citep{reddi2021adaptive}. For local learning rate which satisfying $\eta_l \leq \frac{1}{8KL}$, the local model difference after $k$ ($\forall k \in \{0,1,...,K-1\}$) steps local updates satisfies
\begin{align}
    \frac{1}{N}\sum_{i=1}^{N}\EE [\|\bx_{t,k}^i - \bx_t\|^2] \leq 5K\eta_l^2(\sigma_l^2+6K\sigma_g^2) + 30 K^2 \eta_l^2 \EE[\|\nabla f(\bx_t)\|^2].
\end{align}
\end{lemma}

\begin{proof}
The proof of Lemma \ref{lm:xikt-xt} is similar to the proof of Lemma 3 in \citet{reddi2021adaptive}.
\end{proof}

\begin{lemma} \label{lm:Delta_2}
If assuming that the clients' participation distributions are simulated as independently uniform distribution, then the sequence $\bG_s = \sum_{j\in \cM_s}\sum_{k=0}^{K-1} \nabla F_j (\bx_{s-\tau_s^j,k}^j)$ has the following upper bound, 
\begin{align*}
    & \EE \bigg[\bigg\|\sum_{j\in \cM_s}\sum_{k=0}^{K-1} \nabla F_j (\bx_{s-\tau_s^j,k}^j) \bigg\|^2\bigg] \notag\\
    & \leq \frac{M(N-M)}{N-1} \cdot \bigg[15K^3 L^2 \eta_l^2(\sigma^2 + 6K\sigma_g^2) + (90K^4L^2\eta_l^2 + 3K^2) \frac{1}{N} \sum_{j=1}^N\EE[\|\nabla f(\bx_{s-\tau_s^j})\|^2] \bigg] \notag\\
    & \quad + \frac{3M(N-M)}{N-1} K^2\sigma_g^2 + \frac{M(M-1)}{N(N-1)} \EE\bigg[\bigg\|\sum_{j=1}^N\sum_{k=0}^{K-1} \nabla F_j (\bx_{s-\tau_s^j,k}^j) \bigg\|^2 \bigg] .
\end{align*} 
\end{lemma}

\begin{proof}
We begin with the proof similar to the partial participation with sampling without replacement,
\begin{align}
    & \EE \bigg[\bigg\|\sum_{j\in \cM_s}\sum_{k=0}^{K-1} \nabla F_j (\bx_{s-\tau_s^j,k}^j) \bigg\|^2\bigg] \notag\\
    & = \frac{M(N-M)}{N(N-1)} \sum_{j=1}^N \EE\bigg[\bigg\|\sum_{k=0}^{K-1} \nabla F_j (\bx_{s-\tau_s^j,k}^j) \bigg\|^2 \bigg] + \frac{M(M-1)}{N(N-1)} \EE\bigg[\bigg\|\sum_{j=1}^N\sum_{k=0}^{K-1} \nabla F_j (\bx_{s-\tau_s^j,k}^j) \bigg\|^2 \bigg] \notag\\
    & \leq \frac{M(N-M)}{N(N-1)} \bigg[15NK^3  \eta_l^2(\sigma^2 + 6K\sigma_g^2) + (90K^4L^2\eta_l^2 + 3K^2) \sum_{j=1}^N\EE[\|\nabla f(\bx_{s-\tau_s^j})\|^2] + 3NK^2\sigma_g^2 \bigg] \notag\\
    &  \quad + \frac{M(M-1)}{N(N-1)} \EE\bigg[\bigg\|\sum_{j=1}^N\sum_{k=0}^{K-1} \nabla F_j (\bx_{s-\tau_s^j,k}^j) \bigg\|^2 \bigg] \notag\\
    & \leq \frac{M(N-M)}{N(N-1)} \bigg[15NK^3  \eta_l^2(\sigma^2 + 6K\sigma_g^2) + (90K^4L^2\eta_l^2 + 3K^2) \sum_{j=1}^N\EE[\|\nabla f(\bx_{s-\tau_s^j})\|^2] + 3NK^2\sigma_g^2 \bigg] \notag\\
    & \quad + \frac{M(M-1)}{N(N-1)} \EE\bigg[\bigg\|\sum_{j=1}^N\sum_{k=0}^{K-1} \nabla F_j (\bx_{s-\tau_s^j,k}^j) \bigg\|^2 \bigg] \notag\\
    & \leq \frac{M(N-M)}{N(N-1)} \bigg[15NK^3  \eta_l^2(\sigma^2 + 6K\sigma_g^2) + (90K^4L^2\eta_l^2 + 3K^2) \sum_{j=1}^N\EE[\|\nabla f(\bx_{s-\tau_s^j})\|^2] + 3NK^2\sigma_g^2 \bigg] \notag\\
    & \quad + \frac{M(M-1)}{N(N-1)} \EE\bigg[\bigg\|\sum_{j=1}^N\sum_{k=0}^{K-1} \nabla F_j (\bx_{s-\tau_s^j,k}^j) \bigg\|^2 \bigg] \notag\\
    & \leq \frac{M(N-M)}{N-1} \bigg[15K^3 L^2 \eta_l^2(\sigma^2 + 6K\sigma_g^2) + (90K^4L^2\eta_l^2 + 3K^2) \frac{1}{N}\sum_{j=1}^N\EE[\|\nabla f(\bx_{s-\tau_s^j})\|^2] + 3K^2\sigma_g^2 \bigg] \notag\\
    & \quad + \frac{M(M-1)}{N(N-1)} \EE\bigg[\bigg\|\sum_{j=1}^N\sum_{k=0}^{K-1} \nabla F_j (\bx_{s-\tau_s^j,k}^j) \bigg\|^2 \bigg] \notag\\
    & = \frac{M(N-M)}{N-1} \bigg[15K^3 L^2 \eta_l^2(\sigma^2 + 6K\sigma_g^2) + (90K^4L^2\eta_l^2 + 3K^2) \frac{1}{N} \sum_{j=1}^N\EE[\|\nabla f(\bx_{s-\tau_s^j})\|^2] \bigg] \notag\\
    & \quad + \frac{3M(N-M)}{N-1} K^2\sigma_g^2 + \frac{M(M-1)}{N(N-1)} \EE\bigg[\bigg\|\sum_{j=1}^N\sum_{k=0}^{K-1} \nabla F_j (\bx_{s-\tau_s^j,k}^j) \bigg\|^2 \bigg] .
\end{align}
\end{proof}

\clearpage

\section{Additional Experiments}\label{sec:appendix_exp}
  
Tables \ref{tab:large_delay_cifar10_ms} and \ref{tab:large_delay_cifar100_ms} present results computed over experiments with 3 different random seeds. Tables \ref{tab:large_delay_cifar10_ms} and \ref{tab:large_delay_cifar100_ms} compare the performance of various federated learning methods, on the test accuracy of the ResNet-18 model across CIFAR-10 and CIFAR-100 datasets with heterogeneous data distributions. It is observed that the delay-adaptive FADAS (abbreviated as delay-adaptive FADAS) consistently outperforms the other methods. The consistency of FADAS performance under \textit{large worst-case} delay settings indicates its reliability and potential for practical applications in federated learning environments with diverse and asynchronous model updates. 

\subsection{Additional Results}
\begin{table}[H]
    \centering
    \caption{The test accuracy on training ResNet-18 model on CIFAR-10 dataset with two data heterogeneity levels in a large worst-case delay scenario for 500 communication rounds. We abbreviate delay-adaptive FADAS to FADAS$_{\text{da}}$ in this and subsequent tables. We conduct experiments on three seeds, and we report the average accuracy and standard derivation.}
    \vskip 0.05in
    \begin{tabular}{l|cc}
    \toprule 
    & Dir(0.1) & Dir (0.3)\\
    Method & Acc. \& std. & Acc. \& std. \\
    \midrule
    FedAsync & 50.29 $\pm$ 6.86 & 59.75 $\pm$ 13.40 \\
    FedBuff & 44.92 $\pm$ 5.26 & 46.94 $\pm$ 0.99 \\
    FADAS & 70.57 $\pm$ 2.04 & 75.97 $\pm$ 2.64 \\
    FADAS$_{\text{da}}$ & \textbf{72.64} $\pm$ 1.00 & \textbf{80.26} $\pm$ 0.68 \\
    \bottomrule
    \end{tabular}
    \label{tab:large_delay_cifar10_ms}
\end{table}

\begin{table}[H]
    \centering
    \caption{The test accuracy on training ResNet-18 model on CIFAR-100 dataset with two data heterogeneity levels in a \textit{large worst-case} delay scenario for 500 communication rounds. We conduct experiments on three seeds, and we report the average accuracy and standard derivation. }
    \vskip 0.05in
    \begin{tabular}{l|cc}
    \toprule 
    & Dir(0.1) & Dir (0.3)\\
    Method & Acc. \& std. & Acc. \& std. \\
    \midrule
    FedAsync & 46.25 $\pm$ 4.33 &  43.22 $\pm$ 10.75\\
    FedBuff & 15.97 $\pm$ 2.44 & 28.58 $\pm$ 4.74 \\
    FADAS & 47.85 $\pm$ 0.69 & 52.80 $\pm$ 1.15 \\
    FADAS$_{\text{da}}$ & \textbf{51.55} $\pm$ 1.03 & \textbf{56.01} $\pm$ 0.95 \\
    \bottomrule
    \end{tabular}
    \label{tab:large_delay_cifar100_ms}
\end{table}

\subsection{Implementation Details} \label{subsec:hyp}
\textbf{Details of applying adaptive learning rate.} During our experiments, we found that choosing a relatively small global learning rate $\eta$ yields better results for adaptive FL methods (hyper-parameter details can be found in the following). To scale the learning rate down for the model update with larger delays, we directly scale down the learning rate for this step to $\eta/\tau_t^{\max}$, which is shown in \eqref{eq:delay_adapt_appendix},
\begin{align}\label{eq:delay_adapt_appendix}
\eta_t = 
\begin{cases}
\eta & \text{if } \tau_t^{\max} \leq \tau_c, \\
\min\left\{\eta, \frac{\eta}{\tau_t^{\max}}\right\} & \text{if } \tau_t^{\max} > \tau_c.
\end{cases}
\end{align}

\textbf{Hyper-parameter Settings.}
We conduct detailed hyper-parameter searches to find the best hyper-parameter for each baseline. We grid the local learning rate $\eta_l$ from $\{0.001, 0.003, 0.01, 0.03, 0.1\}$, and global learning rate $\eta = 1$ for SGD-based method. We grid the local learning rate $\eta_l$ from $\{0.003, 0.01, 0.03, 0.1\}$ and global learning rate $\eta$ from $\{0.0001, 0.0003, 0.001, 0.003\}$ for adaptive method. 
For the global adaptive optimizer, we set $\beta_1 = 0.9$, $\beta_2 = 0.99$, and we set $\epsilon =  10^{-8}$. Table \ref{tab:hyper} summarizes the hyper-parameter details in our experiments. 

\begin{table}[ht]
\small
    \centering
    \caption{Hyper-parameters details for vision tasks.}
    \vskip 0.05 in
    \begin{tabular}{l|cc cc cc cc}
    \toprule
    \multicolumn{9}{c}{CIFAR-10 (mild delay) }\\
    \toprule
    & \multicolumn{2}{c}{FedAsync} & \multicolumn{2}{c}{FedBuff} & \multicolumn{2}{c}{FADAS} & \multicolumn{2}{c}{FADAS$_{\text{da}}$} \\
    Models \& Dir($\alpha$) & $\eta_l$ & $\eta$ & $\eta_l$ & $\eta$ & $\eta_l$ & $\eta$ & $\eta_l$ & $\eta$ \\
    \midrule
    ResNet-18 \& Dir(0.1) & 0.003 & 1 & 0.03 & 1 & 0.1 & 0.0003 & 0.1 & 0.001 \\
    ResNet-18 \& Dir(0.3) & 0.01 & 1 & 0.03 & 1 & 0.1 & 0.0003 & 0.1 & 0.001 \\
    \toprule
    \multicolumn{9}{c}{CIFAR-100 (mild delay) }\\
    \toprule
    & \multicolumn{2}{c}{FedAsync} & \multicolumn{2}{c}{FedBuff} & \multicolumn{2}{c}{FADAS} & \multicolumn{2}{c}{FADAS$_{\text{da}}$} \\
    Models \& Dir($\alpha$) & $\eta_l$ & $\eta$ & $\eta_l$ & $\eta$ & $\eta_l$ & $\eta$ & $\eta_l$ & $\eta$ \\
    \midrule
    ResNet-18 \& Dir(0.1) & 0.01 & 1 & 0.03 & 1 & 0.1 & 0.0003 & 0.1 & 0.001 \\
    ResNet-18 \& Dir(0.3) & 0.01 & 1 & 0.03 & 1 & 0.1 & 0.0003 & 0.1 & 0.001 \\
    \bottomrule
    \toprule
    \multicolumn{9}{c}{CIFAR-10 (large worst-case delay) }\\
    \toprule
    & \multicolumn{2}{c}{FedAsync} & \multicolumn{2}{c}{FedBuff} & \multicolumn{2}{c}{FADAS} & \multicolumn{2}{c}{FADAS$_{\text{da}}$} \\
    Models \& Dir($\alpha$) & $\eta_l$ & $\eta$ & $\eta_l$ & $\eta$ & $\eta_l$ & $\eta$ & $\eta_l$ & $\eta$ \\
    \midrule
    ResNet-18 \& Dir(0.1) & 0.003 & 1 & 0.03 & 1 & 0.1 & 0.0001 & 0.1 & 0.001 \\
    ResNet-18 \& Dir(0.3) & 0.003 & 1 & 0.03 & 1 & 0.1 & 0.0001 & 0.1 & 0.001 \\
    \toprule
    \multicolumn{9}{c}{CIFAR-100 (large worst-case delay) }\\
    \toprule
    & \multicolumn{2}{c}{FedAsync} & \multicolumn{2}{c}{FedBuff} & \multicolumn{2}{c}{FADAS} & \multicolumn{2}{c}{FADAS$_{\text{da}}$} \\
    Models \& Dir($\alpha$) & $\eta_l$ & $\eta$ & $\eta_l$ & $\eta$ & $\eta_l$ & $\eta$ & $\eta_l$ & $\eta$ \\
    \midrule
    ResNet-18 \& Dir(0.1) & 0.003 & 1 & 0.03 & 1 & 0.1 & 0.0001 & 0.1 & 0.001 \\
    ResNet-18 \& Dir(0.3) & 0.001 & 1 & 0.03 & 1 & 0.1 & 0.0001 & 0.1 & 0.001 \\
    \bottomrule
    \end{tabular}
    \vspace{5pt}
    \label{tab:hyper}
\end{table}

\begin{table}[ht]
\small
    \centering
    \caption{Hyper-parameters details for language tasks.}
    \vskip 0.05 in
    \begin{tabular}{l|cc cc cc cc}
    \toprule
    & \multicolumn{2}{c}{FedAsync} & \multicolumn{2}{c}{FedBuff} & \multicolumn{2}{c}{FADAS} & \multicolumn{2}{c}{FADAS$_{\text{da}}$} \\
    Datasets  & $\eta_l$ & $\eta$ & $\eta_l$ & $\eta$ & $\eta_l$ & $\eta$ & $\eta_l$ & $\eta$ \\
    \midrule
    RTE & 0.01 & 1 & 0.01 & 1 & 0.01 & 0.005 & 0.01 & 0.01 \\
    MRPC & 0.001 & 1 & 0.01 & 1 & 0.01 & 0.001 & 0.01 & 0.002 \\
    SST-2 & 0.001 & 1 & 0.001 & 1 & 0.1 & 0.0005 & 0.1 & 0.001 \\
    
    \bottomrule
    \end{tabular}
    \vspace{5pt}
    \label{tab:hyper2}
\end{table}

\end{document}